%% file: main.tex
\documentclass{article}

 \usepackage[preprint]{neurips_2026}

\usepackage[utf8]{inputenc} 
\usepackage[T1]{fontenc}    
\usepackage{hyperref}       
\usepackage{url}            
\usepackage{booktabs}       
\usepackage{amsfonts}       
\usepackage{nicefrac}       
\usepackage{microtype}      
\usepackage{xcolor}         
\usepackage{graphicx}
\usepackage{amsmath}
\usepackage{amssymb}
\usepackage{mathtools}
\usepackage{amsthm}
\usepackage{multirow}
\usepackage{algorithm}
\usepackage{algpseudocode}
\usepackage{comment}
\usepackage{wrapfig}
\usepackage{graphicx}
\usepackage{makecell}
\usepackage[labelformat=empty]{subfig}
\usepackage[capitalize]{cleveref}
\Crefname{section}{Sec.}{Secs.}
\Crefname{equation}{Eq.}{Eqs.}
\Crefname{figure}{Fig.}{Figs.}
\Crefname{tabular}{Tab.}{Tabs.}

\usepackage{pifont}
\newcommand{\cmark}{\ding{51}}
\newcommand{\xmark}{\ding{55}}
\newcommand{\halfmark}{\ding{109}}

\newcommand{\myparagraph}[1]{\vspace{0pt}\noindent{\bf #1}}

\title{Can Circuit Alignment Predict OOD Generalization?}

\author{%
  Ayan Banerjee$^{1,2 \S}$,
    Abhra Chaudhuri$^{3,\S}$, 
  Josep Llados$^{1,2}$,
  Umapada Pal$^4$,
  Anjan Dutta$^5$, \\
  $^1$Computer Vision Center, $^2$Universitat Autònoma de Barcelona,
  $^3$Fujitsu Research of Europe, \\$^4$Indian Statistical Institute, Kolkata,
  $^5$ University of Surrey \\
  $^{1,2}$\texttt{\{abanerjee,josep\}@cvc.uab.cat}, $^3$\texttt{abhra.chaudhuri@fujitsu.com} \\
  $^4$\texttt{umapada@isical.ac.in}, $^5$\texttt{anjan.dutta@surrey.ac.uk}
}

\theoremstyle{plain}
\newtheorem{theorem}{Theorem}
\newtheorem{proposition}{Proposition}
\newtheorem{corollary}{Corollary}

\theoremstyle{definition}
\newtheorem{definition}{Definition}
\newtheorem{assumption}{Assumption}
\newtheorem{property}{Property}

\theoremstyle{remark}

\crefname{property}{Property}{Properties}
\Crefname{property}{Property}{Properties}

\newcommand{\ie}{\textit{i.e.}}

\begin{document}

\maketitle
\def\thefootnote{\S}\footnotetext{Equal Contribution.}

\input{neurips_tex/00_abs}
\input{neurips_tex/01_intro_org_bkup_6May}
\input{neurips_tex/02_01_related}
\input{neurips_tex/04_01_requirements}

\input{neurips_tex/04_02_cas_construction}
\input{neurips_tex/04_03_consistency}
\input{neurips_tex/05_exp}
\input{neurips_tex/06_conc}

\bibliographystyle{plain}
\bibliography{main}

\clearpage
\appendix
\input{neurips_tex/07_supply}

\clearpage
\newpage

\end{document}

%% file: neurips_tex/00_abs.tex
\begin{abstract}

    Can out-of-distribution (OOD) generalization be predicted from a trained model's weights alone, without any target-domain data? Existing representational similarity metrics (CKA, SVCCA, RSA) compare activations rather than predicting generalization capability of the model. We show they are provably insensitive to structural rerouting in the computational graph, the very change distribution shift induces. We close this gap with the Circuit Alignment Score (CAS), which compares class-specific circuits across domains via graph kernels, decomposed into same-class coherence and cross-class confusion. Casting CAS as a Lebesgue integral over the domain distribution, we prove its Monte Carlo estimate recovers the ground-truth ranking of learners by OOD accuracy, with pairwise inversion error vanishing at rate $O(1/M)$, where $M$ is the number of sampled domains. Across $48$ learners on PACS, CAS attains $0.88$ rank correlation with OOD accuracy, versus $0.58$ (CKA), $0.23$ (SVCCA), and $0.14$ (RSA), with similar trends on other benchmarks and even against data-dependent methods, making it the first provably consistent predictor of distributional robustness requiring neither target-domain data nor labels. The code is available at: \url{https://github.com/ayanban011/ACE}.

\end{abstract}

%% file: neurips_tex/01_intro_org_bkup_6May.tex
\section{Introduction}
\label{sec:introduction}

When can we predict whether a neural network will generalize under distribution shift, given only its trained weights? The question is not mere curiosity: in deployment settings ranging from clinical diagnostics to multilingual NLP, target-domain labels are unavailable by construction, and even unlabeled target samples may be scarce or arrive only after the model has been committed to. Yet, the field lacks a formally grounded answer. Methods that estimate target accuracy from unlabeled target data~\citep{baek2022agreement, garg2022leveraging, yu2022projection} presuppose access to the very distribution whose effect we wish to predict, and theoretical accounts of OOD generalization~\citep{ye2021towards, kaurmodeling} characterize \emph{when} generalization is possible but yield no computable diagnostics on weights. \textbf{No prior work, to our knowledge, derives the structural conditions any weight-only OOD predictor must satisfy, nor establishes a consistency guarantee for ranking learners by such a predictor.}

The natural candidates -- representational similarity metrics such as CKA~\citep{kornblith2019similarity}, SVCCA~\citep{raghu2017svcca}, and RSA~\citep{kriegeskorte2008representational}, operate on activation geometry and are provably blind to structural rerouting in the underlying computation. Two models can produce nearly identical penultimate-layer activations on a source domain while routing them through entirely different computational pathways, with sharply divergent OOD behavior. Rerouting is precisely what distribution shift induces; activation-level metrics measure the wrong object.

We propose to utilize a more expressive measure. Mechanistic interpretability has matured to the point where the computation a network performs for a given behavior can be localized to a sparse, causally-grounded subgraph of neurons and connections~\citep{ameisen2025circuit, cammarata2020thread, conmy2023towards, elhage2021mathematical, olah2020zoom, wanginterpretability}.
These \emph{circuits} are the natural locus at which to ask whether a model's computation is preserved under distribution shift~\citep{sharkey2025open}: we prove that, in the limit of perfect OOD accuracy, a learner is necessarily \emph{circuit-robust}, \ie, its class-specific circuits are preserved across domains and remain distinct between classes, grounding circuit invariance as a principled structural proxy for OOD robustness rather than a heuristic.

\begin{wrapfigure}[19]{r}{0.5\textwidth}
\vspace{-4mm}
\centering
\includegraphics[width=0.5\textwidth]{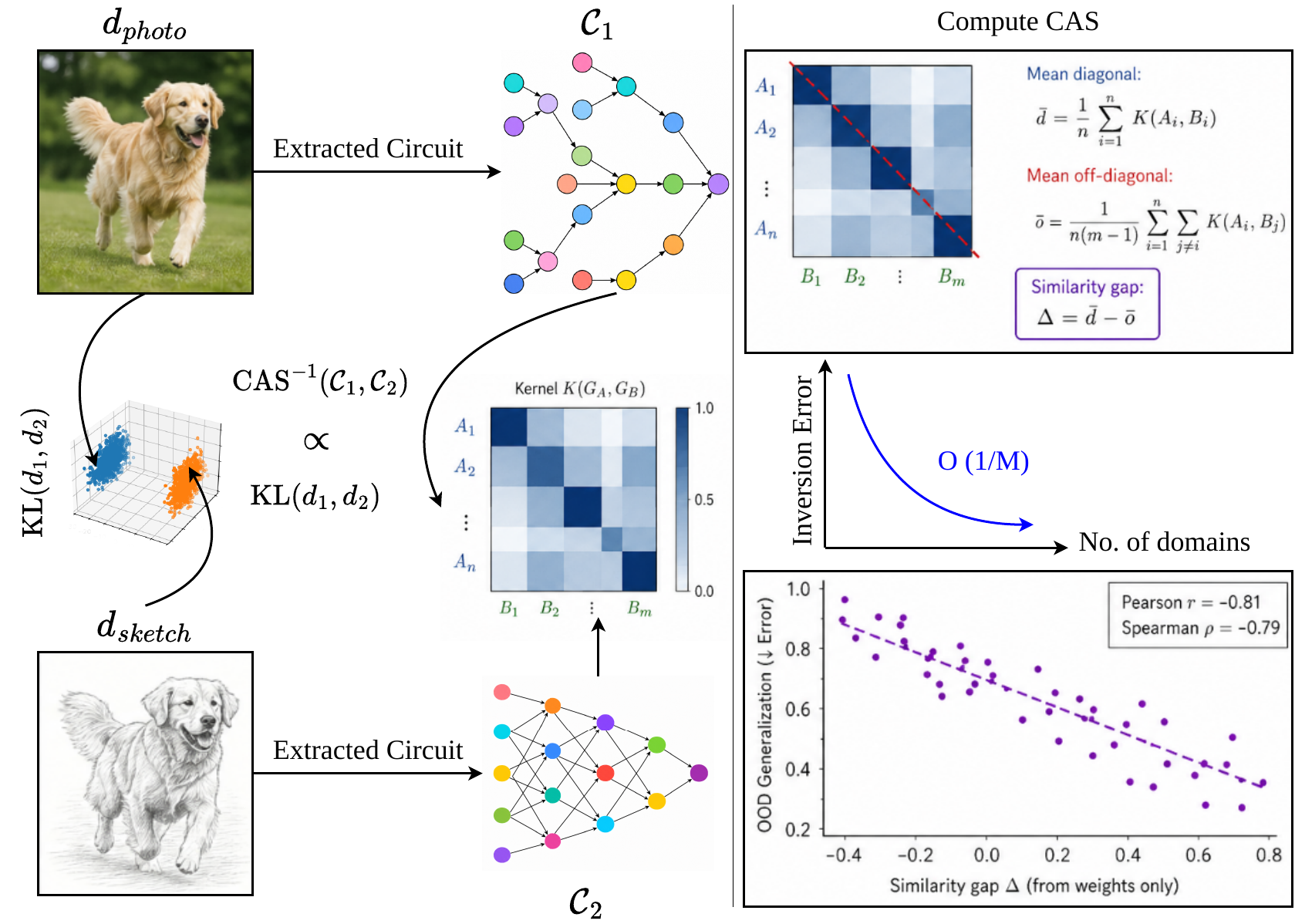}
\vspace{-2mm}
\caption{\textbf{Circuit Evolution prediction under distribution shift with CAS:} From weights alone, we extract class-specific circuits and compare them via graph kernels to score and rank learners based on OOD accuracy, without target data or labels.}
\label{fig:teaser}
\vspace{-4mm}
\end{wrapfigure}

Building on this, we ask: \emph{what must any weight-only metric satisfy to predict OOD generalization?} We answer with three necessary conditions, each derived via an impossibility theorem showing that any metric violating the condition provably conflates a circuit-robust learner with a non-robust one. The conditions require the metric to (i) be sensitive to structural rerouting, (ii) resolve same-class versus cross-class structure separately, and (iii) respond monotonically to graded perturbation. CKA, SVCCA, and RSA each fail at least two by construction. 

Illustrated in \cref{fig:teaser}, we then introduce the \textbf{Circuit Alignment Score (CAS)}, constructed in three steps, each \emph{necessitated} by the three conditions: comparing circuits via graph kernels, arranging the resulting pairwise similarities into a class-conditional matrix, and aggregating into the gap between mean-diagonal and mean-off-diagonal. They capture \emph{circuit drift} (same-class circuits change across domains) and \emph{circuit entanglement} (different-class circuits collapse onto each other), the two structurally distinct failure modes of distribution shift.

Casting the population CAS as a Lebesgue integral over the domain distribution, we prove that its $M$-sample Monte Carlo estimate ranks learners with pairwise inversion error vanishing at rate $O(1/M)$. Combined with an empirically verifiable monotonicity condition, supported by the limit-case correspondence above and the causal-invariance literature~\citep{chen2023understanding, kaurmodeling, wang2022out}, it yields convergence of the predicted ranking to the ground-truth OOD ranking. \textbf{To our knowledge, this is the first weights-only OOD predictor with both a structural justification and a consistency guarantee.}

Empirically, on a pool of $48$ learners spanning four architectures (MLP, ResNet50, MobileNetV2, ViT-B/16), four training objectives (ERM~\citep{vapnik2013nature}, IRM~\citep{arjovsky2019invariant}, CORAL~\citep{sun2016deep}, DANN~\citep{ganin2016domain}), and three regularizations on PACS~\citep{li2017deeper}, CAS attains Spearman $\rho_S = 0.88$ with leave-one-domain-out OOD accuracy, against $0.58$ (CKA), $0.23$ (SVCCA), and $0.14$ (RSA), and additionally even surpassing methods requiring target domain data and labels, directly verifying the monotonicity condition. The gap is consistent across target domains and extends to Office-Home~\citep{venkateswara2017deep} and DomainNet~\citep{peng2019moment}. A controlled LoRA-interpolation experiment in Stable Diffusion XL~\citep{podellsdxl} further shows CAS varies monotonically with continuous domain shift.

To summarize, (i) we prove that OOD robustness structurally implies circuit robustness in the limit; (ii) we derive three necessary conditions on weight-only OOD-predictive metrics via impossibility theorems and show that CKA, SVCCA, and RSA each fail at least two; (iii) we construct CAS to satisfy all three by design, with each construction step necessitated by one condition; (iv) we prove empirical CAS rankings recover the OOD ranking with pairwise inversion error $O(1/M)$; and (v) we empirically validate the full chain across 48 learners on three benchmarks.

\vspace{-2mm}

%% file: neurips_tex/02_01_related.tex
\section{Related Work}

\myparagraph{Predictors of OOD Generalization.}
Existing approaches include training-loss trends~\cite{brigatocan, kim2026spectral}, influence functions~\cite{ye2021out}, margin-based predictors~\cite{jiangpredicting, mouton2024input}, PAC-Bayes bounds~\cite{lotfi2022pac, picard2025good}, meta-learning~\cite{10825441}, and data-centric estimators~\cite{baek2022agreement, ding2021grounding, garg2022leveraging, yu2022projection}. All require either prediction-variance preservation across domains or unlabeled target data. CAS, in contrast, predicts ranking from source-domain circuits alone.

\myparagraph{Representation Similarity Metrics.}
\citep{ahujaempirical} shows that representational metrics correlate with OOD accuracy only inconsistently. Common measures such as CKA~\cite{kornblith2019similarity}, SVCCA~\cite{raghu2017svcca}, RSA~\cite{kriegeskorte2008representational} operate on activation geometry, and we prove (\cref{sec:properties}) that any such activation-factoring metric provably conflates circuit-robust and non-robust learners.

\myparagraph{Circuit Tracing.}
A growing body of work extracts the circuits a network uses for a target behavior, via attribution patching~\cite{syed2024attribution}, sparse autoencoders~\cite{thasarathan2025universal}, probing~\cite{salin2022vision}, causal tracing~\cite{palit2023towards}, and neuron-level analysis~\cite{schwettmann2023multimodal}, building on foundational frameworks for circuit-level interpretability~\cite{ameisen2025circuit, cammarata2020thread, conmy2023towards, elhage2021mathematical, olah2020zoom, wanginterpretability}. These methods implicitly assume a fixed data distribution, leaving open whether circuits are stable, transferable, or comparable across domains~\cite{sharkey2025open}. We address this gap by defining a metric directly on extracted circuits, agnostic to the choice of extractor.

%% file: neurips_tex/04_01_requirements.tex
\section{Circuit Alignment Score (CAS)}

\subsection{What Does Predicting OOD Generalization from Weights Require?}
\label{sec:properties}

We derive necessary conditions on any metric that predicts OOD generalization from a trained model's weights alone. For each condition, we prove an impossibility theorem (deferred to \Cref{app:property-proofs}): any metric violating the condition assigns identical values to a (distributionally) circuit-robust learner and a non-robust one, and is therefore unable to distinguish them. These conditions are not posited but necessitated by the structural premise established in the causal-invariance literature~\citep{chen2023understanding, kaurmodeling, wang2022out}: that OOD generalization is governed by the preservation of class-conditional computational structure across domains. \Cref{sec:cas-construction} constructs CAS to satisfy all three conditions, and \Cref{sec:consistency} proves that ranking learners by CAS recovers their OOD-accuracy ranking.

\myparagraph{Problem setup:} Let $f$ denote a trained model, $D_1, D_2 \in \Omega_D$ two domains drawn from a domain space with distribution $P_D$, and $\mathcal{C}^{(i)}(f, D)$ the sparse weighted directed subgraph of $f$'s computation mediating predictions for class $i$ on domain $D$. We write $\mathcal{C}(f, D) = \{\mathcal{C}^{(i)}(f, D)\}_{i=1}^c$. A metric $\mu$ takes two such circuit families and returns a scalar.


\myparagraph{Circuit robustness as a proxy for OOD robustness:}
Following standard formalizations~\citep{ gulrajanisearch,kaurmodeling,ye2021towards}, the \emph{OOD generalization score} of a learner $\ell$ is its expected accuracy under the domain distribution:
\begin{equation}
\label{eq:ood-score}
g(\ell) \;=\; \int_{\Omega_D} \mathrm{Acc}(\ell; d)\, dP_D(d),
\end{equation}
where $\mathrm{Acc}(\ell; d) \in [0, 1]$ is the classification accuracy of $\ell$ on domain $d$. We say $\ell_i$ is more OOD-robust than $\ell_j$ if $g(\ell_i) > g(\ell_j)$. We show that OOD robustness structurally implies a corresponding form of circuit-level invariance: a learner $\ell$ achieving perfect OOD accuracy must have class-specific circuits that are preserved across domains and remain distinct between classes (\Cref{cor:robustness-implication}). This justifies working with circuit robustness as a structural proxy for OOD robustness. Formally, a learner $\ell$ is called \emph{circuit-robust} between $D_1, D_2$ if:
\begin{equation}
\label{eq:robust}
\kappa(\mathcal{C}_1^{(i)}, \mathcal{C}_2^{(i)}) \geq \alpha \;\; \forall i, \qquad \kappa(\mathcal{C}_1^{(i)}, \mathcal{C}_2^{(j)}) \leq \beta \;\; \forall i \neq j,
\end{equation}
for some $0 \leq \beta < \alpha \leq 1$. The two failure modes of this structure -- \emph{circuit drift} (some same-class similarity falls below $\alpha$) and \emph{circuit entanglement} (some cross-class similarity exceeds $\beta$), are the two ways class-conditional invariance can break down, and a predictive metric must resolve both.

\myparagraph{\cref{prop:structural} -- Structural sensitivity:} A metric must distinguish models that compute the same function via different circuits. Activation-level metrics, those that factor through any embedding into activation space, including CKA~\citep{kornblith2019similarity}, SVCCA~\citep{raghu2017svcca}, and RSA~\citep{kriegeskorte2008representational}, are blind to internal routing: two circuits with identical activation footprints yield identical metric values regardless of how they implement the function. By the universal approximation property of overparameterized networks~\citep{hornik1989multilayer, kawaguchi2016deep}, such functionally-equivalent rerouted circuits provably exist, and one can be circuit-robust while the other is not. \Cref{thm:factoring} (\Cref{app:property-proofs}) shows that any activation-factoring metric assigns the same value to both, hence cannot distinguish robust from non-robust learners.

\myparagraph{\Cref{prop:class-cond} -- Class-conditional resolution:} A metric must report same-class preservation and cross-class entanglement separately. Drift and entanglement are structurally distinct failure modes: drift reduces diagonal similarity while leaving classes distinguishable, whereas entanglement collapses different-class circuits onto each other while diagonals can remain nominally high. A metric that aggregates over all class pairs into a single scalar -- equivalently, one whose value depends only on the multiset of pairwise similarities, not on which pairs are same-class versus cross-class -- cannot separate these regimes. \Cref{thm:aggregating} (\Cref{app:property-proofs}) shows that any such aggregating metric assigns identical values to a circuit-robust configuration and a maximally entangled one, by a permutation argument on the entries of the class-conditional similarity matrix.

\myparagraph{\Cref{prop:semantic} -- Semantic consistency:} A metric must respond monotonically to graded perturbation: if one perturbation simultaneously reduces same-class similarity and increases cross-class similarity relative to another, it must yield a smaller metric value. Without monotonicity, the metric's numerical value carries no ordinal information about shift magnitude, and any ranking of learners or domains derived from it is incoherent. \Cref{thm:non-monotone} (\Cref{app:property-proofs}) shows that non-monotone metrics provably invert perturbation orderings, assigning higher similarity to a strictly more-perturbed configuration, making them unsuitable for ranking-based OOD prediction.


\myparagraph{Summary:} \Cref{tab:properties-failure} reports condition satisfaction for each metric. CKA, SVCCA, and RSA fail P1 (none operate on circuit graphs) and P2 (none form a class-conditional similarity matrix), with P3 thereby vacuous; RSA partially satisfies P2 via its block structure. CAS satisfies all three by construction (\Cref{sec:cas-construction}).

\begin{wraptable}{r}{0.5\textwidth}
\vspace{-4mm}
\centering
\caption{Necessary conditions for OOD-predictive metrics derived in this section. \cmark{} = satisfied by construction; \xmark{} = fails by impossibility theorem (\cref{app:property-proofs}); \halfmark{} = partially satisfied.}
\label{tab:properties-failure}
\resizebox{0.5\textwidth}{!}{
\begin{tabular}{lcccc}
\toprule
Condition & CKA & SVCCA & RSA & CAS \\
\midrule
P1: Structural sensitivity        & \xmark & \xmark & \xmark & \cmark \\
P2: Class-conditional resolution  & \xmark & \xmark & \halfmark & \cmark \\
P3: Semantic consistency          & \xmark & \xmark & \xmark & \cmark \\
\bottomrule
\end{tabular}}
\vspace{-4mm}
\end{wraptable}

%% file: neurips_tex/04_02_cas_construction.tex
\subsection{Construction of Circuit Alignment Score (CAS)}
\label{sec:cas-construction}

We now construct the Circuit Alignment Score, a metric over circuit families satisfying the three necessary conditions derived in \Cref{sec:properties}. The construction proceeds in three steps, each necessitated by one of the conditions: graph-kernel comparison of circuits (P1), arrangement of pairwise similarities into a class-conditional matrix (P2), and decomposition into mean-diagonal minus mean-off-diagonal (P3).

\myparagraph{Step 1 -- Comparing circuits structurally:} P1 forbids any metric that factors through an embedding into activation space. The natural alternative is to compare circuits as the structured objects they are, \ie, sparse weighted directed graphs, using a kernel defined on graph topology rather than activation outputs. We equip circuit space with a graph kernel $\kappa : \mathcal{C} \times \mathcal{C} \to [0, 1]$, normalized so $\kappa(C, C) = 1$, that compares circuits via their nodes and edges. We use treelet \citep{gauzere2012two}, random walk \citep{nikolentzos2020random}, and optimal-transport \citep{petric2019got} graph kernels; \Cref{app:kernel-choice} ablates this choice.

\myparagraph{Step 2 -- Resolving class-conditional structure:} P2 forbids any metric that aggregates pairwise similarities without distinguishing same-class from cross-class pairs. Given two circuit families $\mathcal{C}_1 = \{C_1^{(i)}\}_{i=1}^c$ and $\mathcal{C}_2 = \{C_2^{(i)}\}_{i=1}^c$, we form the class-conditional similarity matrix:
\begin{equation}
\label{eq:S-matrix}
S \in \mathbb{R}^{c \times c}, \qquad S_{ij} = \kappa\!\left(C_1^{(i)}, C_2^{(j)}\right).
\end{equation}
Diagonal entries $S_{ii}$ measure preservation of each class's circuit across domains (low $S_{ii}$ indicates \emph{circuit drift}); off-diagonal entries $S_{ij}$ measure cross-class similarity (high $S_{ij}$ for $i \neq j$ indicates \emph{circuit entanglement}). The two failure modes from \cref{eq:robust} are now read off the matrix at distinct positions, exactly as P2 requires.

\myparagraph{Step 3 -- Aggregating into a scalar:} P3 forbids non-monotone aggregation of $S$. The simplest aggregation that is monotone non-decreasing in each $S_{ii}$ and non-increasing in each $S_{ij}$ ($i \neq j$), and class-permutation invariant, is the difference of uniformly-weighted means:
\begin{equation}
\label{eq:cas}
\mathrm{CAS}(\mathcal{C}_1, \mathcal{C}_2) \;=\; \underbrace{\frac{1}{c}\sum_{i=1}^{c} S_{ii}}_{\text{diagonal coherence}} \;-\; \underbrace{\frac{1}{c(c-1)}\sum_{i \neq j} S_{ij}}_{\text{off-diagonal confusion}}.
\end{equation}
CAS is high when same-class circuits are preserved and different-class circuits remain distinct, i.e., precisely when the learner is circuit-robust (\cref{eq:robust}).

\begin{theorem}[CAS Soundness]
\label{thm:cas-soundness} The Circuit Alignment Score (CAS) defined in \cref{eq:cas} satisfies \cref{prop:structural,prop:class-cond,prop:semantic}.
\end{theorem}


\myparagraph{Sound by construction:} $\kappa$ acts on graph topology (P1), CAS decomposes into independently-reportable same-class and cross-class terms (P2), and is linear in $S$ with the correct monotonicity signs (P3). \cref{app:cas-soundness-proof} provides the full proof along with boundedness, symmetry, and self-similarity.


\myparagraph{Decomposability as a diagnostic:} Beyond the scalar score, $S$ itself exposes a \emph{class vulnerability profile} (low diagonal entries) and a \emph{confusion topology} (entangled off-diagonal pairs), enabling per-class failure analysis that aggregate accuracy obscures (\cref{sec:experiments}).

%% file: neurips_tex/04_03_consistency.tex
\subsection{Consistency: From CAS to OOD Ranking}
\label{sec:consistency}

We lift CAS from a domain comparator to a learner-level diagnostic via a population integral over the domain distribution, prove its Monte Carlo estimate converges with $O(1/M)$ pairwise inversion error, and combine this with a monotonicity condition to recover the true OOD accuracy ranking.

\myparagraph{Population CAS:} Let $\mathcal{L} = \{\ell_1, \ldots, \ell_N\}$ be a pool of learners trained on a common source distribution, and write $\mathrm{CAS}(\ell; d)$ for the CAS between $\ell$'s source and domain-$d$ circuits. The \emph{population CAS} is the Lebesgue integral~\citep{gordon1994integrals}:
\begin{equation}
\label{eq:population-cas}
\overline{\mathrm{CAS}}(\ell) \;=\; \int_{\Omega_D} \mathrm{CAS}(\ell; d) \, dP_D(d),
\end{equation}
well-defined because $\mathrm{CAS} \in [-1,1]$ (\cref{prop:boundedness}). This parallels the OOD score $g(\ell)$ in \cref{eq:ood-score}; \cref{assumption:monotonicity} connects the two.

\myparagraph{Monte Carlo estimator:} The population integral is approximated by drawing $M$ domains $\{d_j\}_{j=1}^M \stackrel{\text{i.i.d.}}{\sim} P_D$ and averaging:
\vspace{-1mm}
\begin{equation}
\label{eq:mc-estimator}
\widehat{\mathrm{CAS}}_M(\ell) \;=\; \frac{1}{M}\sum_{j=1}^M \mathrm{CAS}(\ell; d_j).
\end{equation}
By the strong law of large numbers, $\widehat{\mathrm{CAS}}_M(\ell) \to \overline{\mathrm{CAS}}(\ell)$ almost surely as $M \to \infty$. The remaining question is the rate at which this convergence preserves \emph{pairwise rankings} between learners -- which is what determines whether finite-sample CAS estimates produce the right ordering.

\myparagraph{Pairwise inversion error:} For two learners $\ell_i, \ell_j$ with $\overline{\mathrm{CAS}}(\ell_i) > \overline{\mathrm{CAS}}(\ell_j)$, define the discrepancy $\delta_{ij}(d) = \mathrm{CAS}(\ell_i; d) - \mathrm{CAS}(\ell_j; d)$, with population mean $m_{ij} = \overline{\mathrm{CAS}}(\ell_i) - \overline{\mathrm{CAS}}(\ell_j) > 0$ and finite variance $\sigma_{ij}^2 = \int_{\Omega_D} (\delta_{ij}(d) - m_{ij})^2 \, dP_D(d)$. A pairwise \emph{inversion} occurs when the empirical estimate flips the population ordering. The inversion probability under $M$ samples is:
\[
p^{(M)}_{\mathrm{inv}}(\ell_i, \ell_j) \;=\; \Pr\!\left[\,\frac{1}{M}\sum_{j=1}^M \delta_{ij}(d_j) \leq 0 \,\right], \qquad d_j \stackrel{\text{i.i.d.}}{\sim} P_D.
\]

\begin{theorem}[Monte Carlo consistency of CAS rankings]
\label{thm:mc-consistency}
For any two learners $\ell_i, \ell_j$ with $\overline{\mathrm{CAS}}(\ell_i) \neq \overline{\mathrm{CAS}}(\ell_j)$, the inversion probability satisfies:
\vspace{-2mm}
\[
p^{(M)}_{\mathrm{inv}}(\ell_i, \ell_j) \;\leq\; \frac{\sigma_{ij}^2}{M \cdot m_{ij}^2} \;=\; O\!\left(\frac{1}{M}\right).
\]
Aggregated over all pairs in a learner pool $\mathcal{L}$ of size $N$, the total inversion probability $P^{(M)}_{\mathrm{inv}} = \binom{N}{2}^{-1}\sum_{i<j} p^{(M)}_{\mathrm{inv}}(\ell_i, \ell_j) \to 0$ as $M \to \infty$.
\end{theorem}

\noindent The proof (\cref{app:mc-consistency-proof}) is an application of Chebyshev's inequality to the i.i.d.\ Monte Carlo estimator of the integral $m_{ij}$. The bound depends on $\sigma_{ij}^2 / m_{ij}^2$, which is small whenever the gap between learners' population CAS values is large relative to their domain-wise variability.
\cref{app:inter-vs-intra} discusses why this gap is typically larger between learners trained by different domain-generalization algorithms than within a single algorithm family, yielding faster convergence on the cross-algorithmic pairs that matter most for model selection.

\myparagraph{From CAS rankings to OOD rankings:} \cref{thm:mc-consistency} establishes that empirical CAS rankings converge to population CAS rankings. To translate this into a guarantee about OOD-accuracy rankings, we need the population CAS to be ordinally aligned with the OOD generalization score:

\begin{assumption}[Monotonicity of CAS in OOD accuracy]
\label{assumption:monotonicity}
For all $\ell_i, \ell_j \in \mathcal{L}$,
\[
g(\ell_i) > g(\ell_j) \;\;\Longrightarrow\;\; \overline{\mathrm{CAS}}(\ell_i) > \overline{\mathrm{CAS}}(\ell_j).
\]
\end{assumption}

\noindent This assumption is the connection between the structural premise of \cref{sec:properties} (circuit invariance underlies OOD generalization) and the metric defined in \cref{sec:cas-construction}. We do not prove it from first principles; instead, we verify it empirically across all 48 learners in our pool on PACS, Office-Home, and DomainNet (\cref{sec:experiments}, and \cref{sec:refm} in the appendix), where the empirical rank correlation between $\overline{\mathrm{CAS}}$ and $g$ exceeds $\rho_S = 0.88$ on PACS and remains consistent across benchmarks. \cref{cor:robustness-implication} provides theoretical support: in the limit of perfect OOD accuracy, OOD robustness structurally implies circuit robustness, and CAS, by construction, is monotone in circuit robustness.
Combining \cref{thm:mc-consistency} with \cref{assumption:monotonicity} yields the predictive claim:

\begin{proposition}[Ranking recovery]
\label{prop:ranking-recovery}
Let $\mathcal{L}^* = (\ell_{\pi^*(1)}, \ldots, \ell_{\pi^*(N)})$ be the ground-truth ranking by $g$, and $\mathcal{L}'_M = (\ell_{\pi'_M(1)}, \ldots, \ell_{\pi'_M(N)})$ the predicted ranking by $\widehat{\mathrm{CAS}}_M$. Under \cref{assumption:monotonicity}, the Kendall rank correlation~\citep{abdi2007kendall} between the two satisfies:
\[
\lim_{M \to \infty} \tau(\mathcal{L}^*, \mathcal{L}'_M) = 1 \quad \text{in probability.}
\]
\end{proposition}

\noindent The proof (\cref{app:ranking-recovery-proof}) follows from \cref{thm:mc-consistency} via the union bound: the probability that any pair is inverted vanishes at $O(1/M)$, so the probability that all $\binom{N}{2}$ pairs are correctly ordered tends to 1, which is precisely $\tau \to 1$.

\myparagraph{Summary of the theoretical contribution:} \cref{sec:properties} derived three necessary conditions for OOD-predictive metrics; \cref{sec:cas-construction} constructed CAS to satisfy them (\cref{thm:cas-soundness}); and this section established that ranking learners by empirical CAS recovers their OOD-accuracy ranking, with $O(1/M)$ inversion error per pair (\cref{thm:mc-consistency}), under one empirically verifiable monotonicity assumption (\cref{assumption:monotonicity}). The remaining sections validate this chain experimentally.

%% file: neurips_tex/05_exp.tex
\section{Experimental Results}
\label{sec:experiments}

\myparagraph{Datasets:} We evaluate on three benchmarks of increasing complexity: (1) \textbf{PACS}~\cite{li2017deeper} covers naturalistic style shifts across 4 domains and 7 classes; (2) \textbf{Office-Home}~\cite{venkateswara2017deep} tests discriminative resolution under fine-grained distribution shifts across 4 domains and 65 classes; (3) \textbf{DomainNet}~\cite{peng2019moment}, the largest and most challenging, contains 6 domains, 345 classes, ${\sim}$600K images. Together, they span a wide range of domain shifts, ensuring our findings are not specific to a single data regime.

\myparagraph{Evaluation metrics:} We evaluate CAS as an OOD predictor using two complementary metrics. \textbf{Spearman rank correlation} ($\rho_S$) \cite{zar2005spearman} our primary metric, measures how well predicted model rankings match true OOD rankings, independent of exact accuracy values. \textbf{Mean absolute error} (MAE) measures the deviation between predicted and true OOD accuracy\footnote{True OOD accuracy is taken from the DomainBed repository: https://github.com/facebookresearch/DomainBed} in percentage points and applies only to calibrated predictors (ATC, ALine-D, and ProjNorm). Statistical significance is assessed via permutation tests ($10$ permutations), with $95\%$ confidence intervals from bootstrap resampling ($1,000$ iterations). We fix random seeds and enforce determinism to ensure circuit divergence reflects domain shift rather than training stochasticity.

\myparagraph{Circuit extraction:} CAS operates on class-specific circuit graphs and is agnostic to the extractor; any procedure producing class-conditional directed subgraphs suffices. We use Adaptive Circuit Extraction (ACE) for tractability and consistency across architectures (\cref{sec:ace}).

\subsection{CAS as OOD Predictors}
\label{sec:cas_vs_baselines}
This experiment jointly validates three claims: (1) CAS satisfies the monotonicity required by \cref{assumption:monotonicity}, (2) high CAS values correspond to the circuit-separation regime of \cref{cor:robustness-implication}, and (3) activation-level baselines exhibit the conflation failure guaranteed by \cref{thm:factoring}. We compare CAS against CKA~\cite{kornblith2019similarity}, SVCCA~\cite{raghu2017svcca}, and RSA~\cite{kriegeskorte2008representational} across 48 learners spanning 4 architectures (MLP, ResNet50, MobileNetV2, and ViT-B/16), 4 training objectives (ERM, IRM, CORAL, and DANN), and 3 regularization settings (None, Dropout (rate~$0.3$), and weight decay ($10^{-4}$)), all trained on PACS. For each learner, we compute: (i) cross-domain similarity score averaging pairwise similarities across all source domain pairs, and (ii) leave-one-domain-out OOD accuracy by training on three domains and evaluating on the held-out fourth domain. CKA, SVCCA, and RSA are computed on circuit representations obtained via graph embeddings~\cite{dutta2018stochastic}, whereas CAS operates directly on ACE-extracted circuit graphs using graph kernels~\cite{gauzere2012two,nikolentzos2020random,petric2019got}. We visualize these relationships in \cref{fig:casvsood}, where each point represents a learner colored by training objective, with a least-squares trend line overlaid. Additional comparisons are provided in \cref{fig:pacs,fig:of,fig:dn,fig:comp_fig}.

\begin{figure*}[!htbp]
    \centering
    \includegraphics[width=\linewidth]{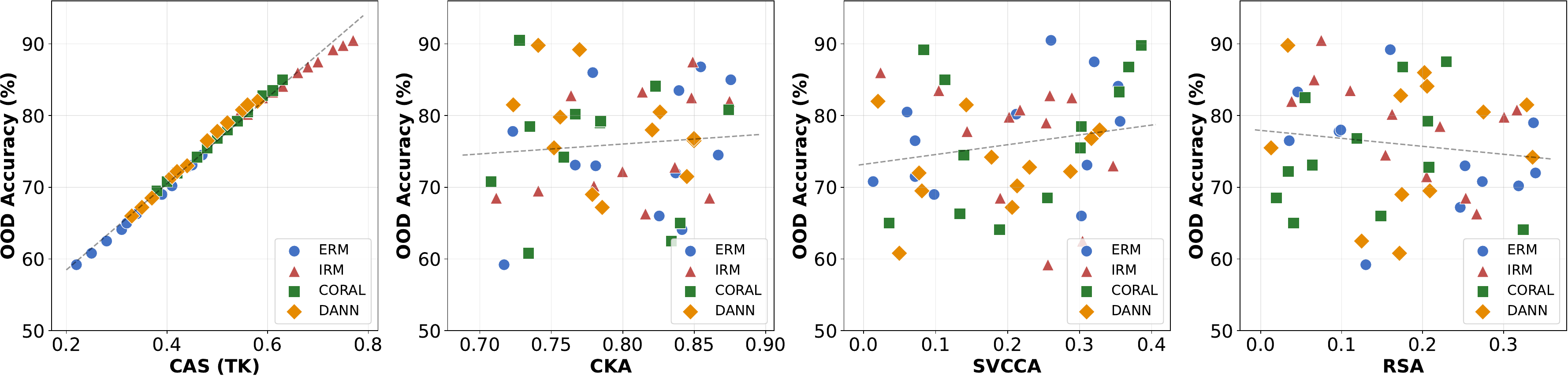}
    \vspace{-5mm}
    \caption{CAS vs OOD accuracy shows a monotonous trend ($\rho_S = 0.93$), whereas CKA ($\rho_S = 0.58$), SVCCA ($\rho_S = 0.23$), and RSA ($\rho_S = 0.14$) show a diffuse scatter, confirming that representational similarity fails to capture the circuit-level signal (Target class: cartoon of the PACS dataset).}
    \label{fig:casvsood}
    \vspace{-4mm}
\end{figure*}

CAS exhibits a strong near-monotone relationship with OOD accuracy, with higher circuit alignment corresponding to better OOD performance. Training objectives cluster coherently: ERM in the low-CAS, low-accuracy region, IRM in the high-CAS, high-accuracy region, and CORAL and DANN in between. In contrast, CKA, SVCCA, and RSA show weak positive trends with high variance, often assigning high similarity to poorly performing ERM models. The near-monotone ordering recovered by CAS supports \cref{assumption:monotonicity}, while high-CAS, high-accuracy IRM learners (red dots in \cref{fig:casvsood}) approach the kernel-separation regime of \cref{cor:robustness-implication}, indicating domain-invariant computational pathways rather than spuriously aligned representations. The failure of CKA, SVCCA, and RSA is consistent with the conflation predicted by \cref{thm:factoring}: they cannot distinguish circuit-robust learners from non-robust ones whose circuit drift produces similar activation statistics. Together, these results support CAS as a mechanistically grounded and empirically reliable signal of OOD robustness, while activation-level metrics remain insensitive to structural differences in computation.

\begin{wrapfigure}{r}{0.5\textwidth}
\vspace{-4mm}
\centering
\includegraphics[width=0.5\textwidth]{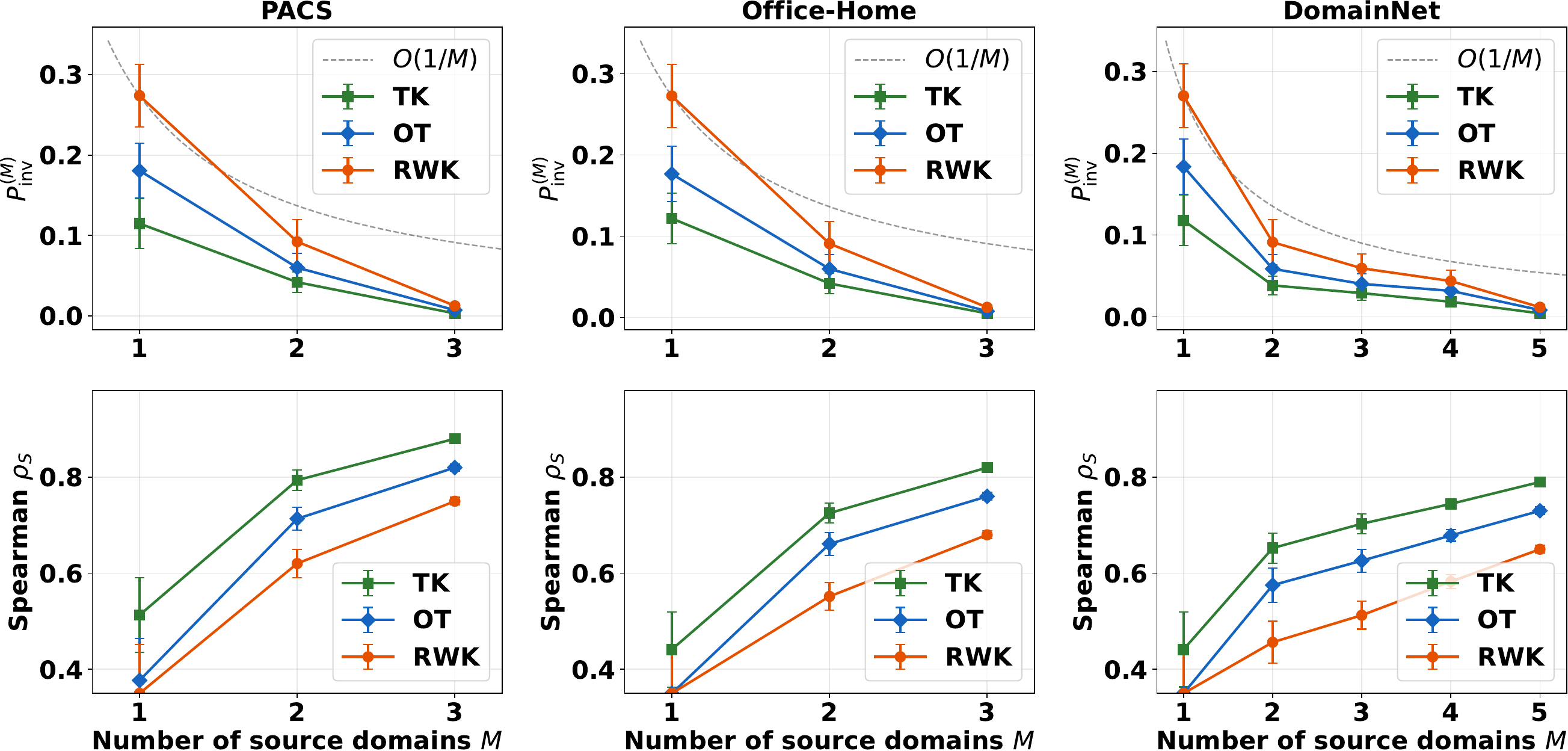}
\caption{\textbf{Consistency of CAS rankings.} Effect of the number of source domains \(M\) on ranking stability. \textbf{Top:} Pairwise inversion probability \(P_{\mathrm{inv}}^{(M)}\) vs.\ \(M\) (mean $\pm$ std over 10 trials) with $O(1/M)$ Chebyshev bound (dashed lines). \textbf{Bottom:} Spearman $\rho_S$ between \(M\)-domain CAS ranking and true OOD accuracy. TK consistently achieves the highest $\rho_S$ and all methods follow the $O(1/M)$ trend.}
\label{fig:mc_consistency}
\vspace{-4mm}
\end{wrapfigure}

\subsection{Consistency of CAS Rankings}
\label{sec:mc_validation}

We evaluate the prediction of \Cref{thm:mc-consistency}, which states that the pairwise rank inversion probability $P_{\mathrm{inv}}^{(M)}$ between two learners with distinct population CAS values decays as $O(1/M)$. In practice, $M$ is fixed by the benchmark: $M{=}3$ for PACS and Office-Home, and $M{=}5$ for DomainNet, leaving one domain out for OOD accuracy. To validate the theorem, we subsample $M \in \{1, 2, 3\}$ domains for PACS and Office-Home, and $M \in \{1, 2, 3, 4, 5\}$ for DomainNet, using real held-out source domains without any synthetic augmentation. For each $M$, we repeat CAS computation 10 times, rank the 48 learners by $\widehat{\mathrm{CAS}}_M$, and measure: (i) the fraction of learner-pair rank inversions relative to the maximum-$M$ ranking $\binom{48}{2}{=}1128$ pairs, and (ii) the Spearman correlation $\rho_S$ between the $M$-domain CAS ranking and true OOD accuracy ranking. We also overlay the theoretical $O(1/M)$ Chebyshev bound to verify the rate and constant.

The results in \cref{fig:mc_consistency} confirm the predicted $O(1/M)$ decay across all benchmarks and methods. At $M{=}1$, inversion probabilities are high, up to ${\sim}0.27$ for RWK, and Spearman correlations are low ($\rho_S \approx 0.39$--$0.52$), reflecting the large variance of a single-domain estimate. Here, CAS reduces to self-similarity, carrying no cross-domain alignment information, while OOD accuracy reflects single-source domain generalization. As $M$ increases, $P_{\mathrm{inv}}^{(M)}$ drops sharply following the $O(1/M)$ envelope, while $\rho_S$ rises steeply. At the practical operating points $M{=}3$ for PACS and Office-Home, $M{=}5$ for DomainNet, inversion probability is already low and $\rho_S$ approaches convergence. TK achieves the highest ranking quality ($\rho_S \approx 0.89$, $0.81$, $0.79$), followed by OT ($0.82$, $0.76$, $0.74$) and RWK ($0.75$, $0.69$, $0.64$) on PACS, Office-Home, and DomainNet, respectively. Overall, the empirical $P_{\mathrm{inv}}^{(M)}$ curves lie at or below the theoretical $O(1/M)$ Chebyshev bound ($M \geq 2$ for RWK), confirming that CAS-based OOD ranking errors decrease as the number of domains increases.

\subsection{Evolution of CAS with Domain Perturbations}

\begin{wrapfigure}{r}{0.5\textwidth}
\vspace{-4mm}
    \centering
    \includegraphics[width=0.5\textwidth]{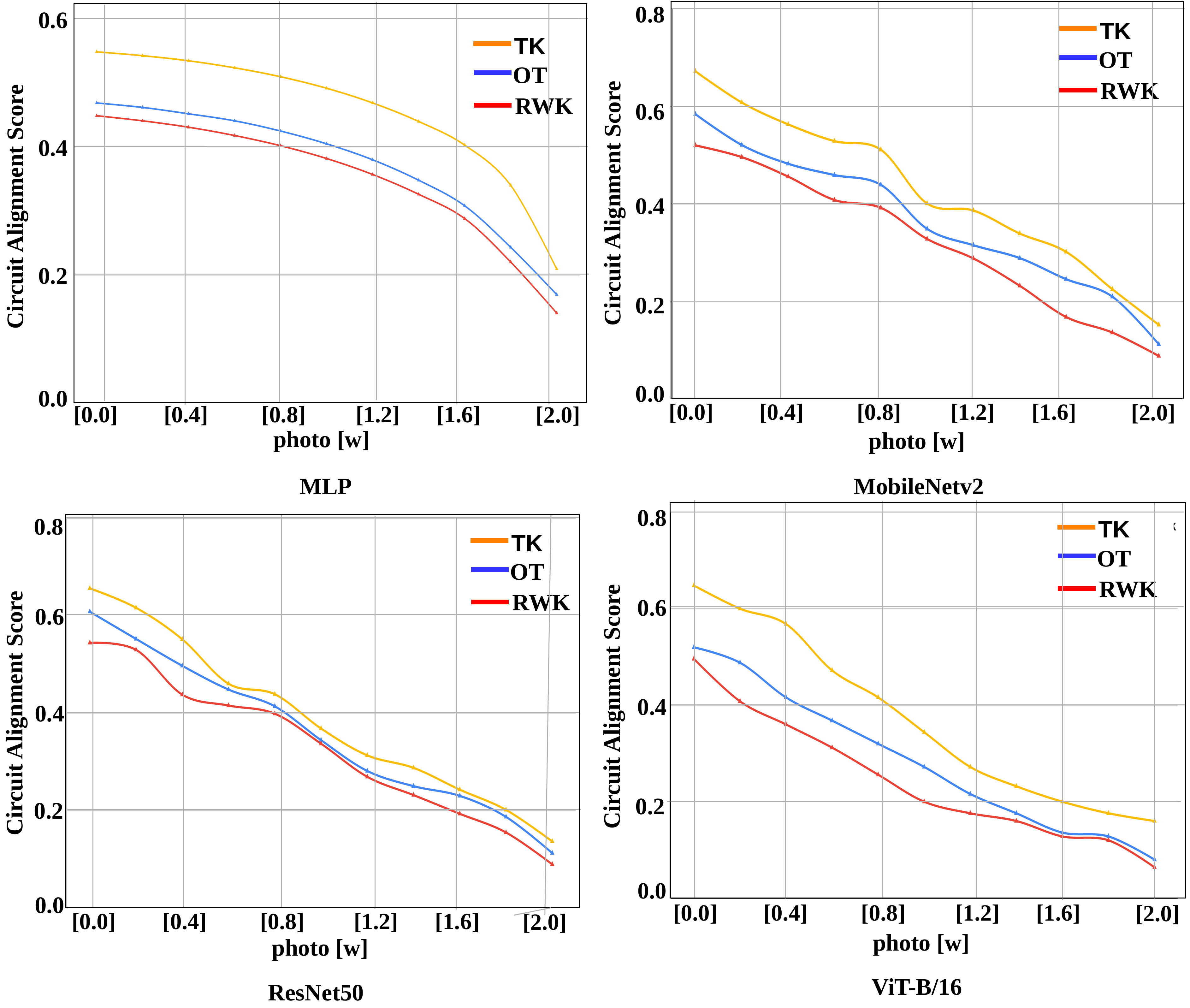}
    \caption{\textbf{CAS under controlled domain interpolation via LoRA.} CAS values are monotonously decreasing with increasing domain perturbations.}
    \label{fig:cas}
    \vspace{-4mm}
\end{wrapfigure}

This experiment validates whether CAS reflects domain proximity at the circuit level by recovering a ground-truth perturbation ordering under controlled, monotonically increasing domain shift. For $w < w'$, the intermediate domain $D_2[w']$ should induce a stronger perturbation of $D_1$ than $D_2[w]$, i.e., $D_2[w] \prec_{D_1} D_2[w']$ (\cref{def:perturbation-order}), and a faithful metric must assign a higher score to the weaker perturbation. \Cref{thm:non-monotone} formalizes the converse risk: any aggregation $\Psi(S)$ that is not coordinate-wise monotone can invert these rankings. CAS avoids this by construction, as it is monotone increasing on the diagonal and decreasing off the diagonal of $S$, consistent with \cref{def:perturbation-order}. We fix $D_1 =$ ArtPainting and construct intermediate domains $\{D_2[w]\}$ by interpolating between ArtPainting-style and Photo-style LoRAs at inference time with SDXL~\cite{podellsdxl}, with the Photo adapter weight $w \in [0, 2]$. This performs domain interpolation in \emph{parameter space}, preserving semantic content while smoothly varying domain-specific appearance (\cref{fig:domainmixup}), unlike input-space mixup~\cite{cao2024mixup,xu2020adversarial}. For each intermediate domain, we train MLP~\cite{popescu2009multilayer}, ResNet50~\cite{he2016deep}, MobileNetv2~\cite{sandler2018mobilenetv2}, and ViT-B/16~\cite{dosovitskiyimage} under identical initialization and data distribution, extract class-wise circuits via ACE, and compute the similarity matrix $S$ between $D_1$ and $D_2[w]$ using three graph kernels.

\begin{figure*}[!t]
    \centering
    \includegraphics[width=\linewidth]{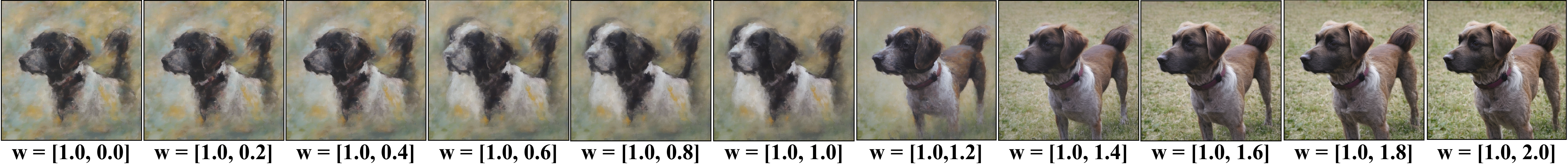}
    \caption{\textbf{Domain interpolation.} Images generated by SDXL with a fixed ArtPainting (weight 1.0) and a Photo LoRA with weight variations. At $w{=}[1.0, 0.0]$, the output is purely ArtPainting style; at $w{=}[1.0, 2.0]$, it is fully photorealistic. Semantic content (dog, pose, composition) is preserved throughout, while domain-specific appearance (brush strokes, texture, lighting) transitions smoothly, enabling controlled, continuous domain shift for circuit-level analysis.}
    \label{fig:domainmixup}
    \vspace{-7mm}
\end{figure*}

As shown in \cref{fig:cas}, CAS decreases monotonically with increasing $w$ across all four architectures and three graph kernels. Increasing domain shift induces consistent circuit divergence, with no reversals in ordering, plateau artefacts, or kernel- or architecture-specific anomalies across $w \in [0, 2]$. This monotonicity supports the perturbation ordering $D_2[w] \prec_{D_1} D_2[w']$ for all $w < w'$, consistent with \cref{sec:consistency}, and suggests that domain proximity is reflected in circuit space. Although architectures exhibit distinct slopes, reflecting their inductive biases and representational capacities, the \emph{relative ordering} of domain perturbations is preserved across all configurations. These results establish CAS as a stable measure of domain-induced circuit reorganization across the tested architectures, rather than an artefact of a particular architectural prior.
\subsection{Comparison with OOD Accuracy Prediction Methods}
We evaluate whether CAS can predict OOD accuracy without target-domain labels by comparing against four families of OOD predictors following standard protocols (\cref{app:baseline-protocols}). \Cref{tab:ood_pred_main} reports the head-to-head comparison. Confidence-based methods \cite{garg2022leveraging,guillory2021predicting,vedantam2021empirical} treat high softmax confidence as a proxy for correctness but can fail under large style shifts where models remain confidently wrong. Complexity-based measures \cite{Jiang*2020Fantastic} characterize the model's capacity to generalize but were designed for in-distribution gaps, explaining their weak transfer to OOD settings (e.g. $\rho_s < 0.5$ and goes negative for spectral complexity). Distribution-shift-based methods \cite{yu2022projection,vedantam2021empirical} measure the magnitude of distributional change but cannot distinguish benign shifts (background changes, circuits intact) from harmful ones (discriminative features altered).

\begin{table}[!htbp]
\vspace{-2mm}
\caption{\textbf{Comparison of OOD accuracy prediction methods on DomainBed Benchmarks:} CAS achieves state-of-the-art performance across all three benchmarks without any training. $^\dagger$ALine-D requires the full 48-model pool; all other methods are single-model. Whereas ProjNorm requires two extra retraining runs per fold. Similarly, DoC trains a simple regression, H$\Delta$H trains a domain classifier, and CSS performs post-deployment correction for accurate OOD prediction.}
\label{tab:ood_pred_main}
\centering
\resizebox{\textwidth}{!}{%
\begin{tabular}{llllcccccccccc}
\toprule
\multirow{2}{*}{Family} & \multirow{2}{*}{Method} & \multirow{2}{*}{Input Modality} & \multirow{2}{*}{Needs target?} & \multicolumn{3}{c}{PACS} & \multicolumn{3}{c}{Office-Home} & \multicolumn{3}{c}{DomainNet} \\
\cmidrule(lr){5-7} \cmidrule(lr){8-10} \cmidrule(lr){11-13}
 &  &  &  &  $P_s$ & $R^2$ & MAE & $P_s$ & $R^2$ & MAE & $P_s$ & $R^2$ & MAE\\
\midrule
\multirow{6}{*}{Confidence} & ATC \cite{garg2022leveraging} & Softmax outputs & Unlabeled & 0.58 & 0.64 & 5.31 & 0.53 & 0.47 & 5.89 & 0.47 & 0.42 & 6.31 \\
 & DoC$^\dagger$ \cite{guillory2021predicting} & Softmax outputs & Unlabeled & 0.63 & 0.49 & 5.26 & 0.58 & 0.51 & 5.54 & 0.50 & 0.44 & 6.10 \\
 & Entropy-Target \cite{vedantam2021empirical} & Softmax outputs & Unlabeled &  0.66 & 0.51 & 5.14 & 0.61 & 0.54 & 5.33 & 0.54 & 0.48 & 5.82 \\
 & Entropy-Source \cite{vedantam2021empirical} & Softmax outputs & Source only & 0.61 & 0.47 & 5.88 & 0.56 & 0.49 & 5.68 & 0.49 & 0.44 & 6.17 \\
 & Source Val.\ Acc.\ \cite{vedantam2021empirical} & Labels (source) & Source only & 0.73 & 0.59 & 4.96 & 0.68 & 0.60 & 4.85 & 0.62 & 0.55 & 5.27 \\
 & Mixup Criterion \cite{vedantam2021empirical} & Softmax outputs & Unlabeled &  0.27 & 0.34 & 8.05 & 0.25 & 0.24 & 7.80 & 0.21 & 0.20 & 8.09 \\
\midrule
\multirow{5}{*}{Complexity} & Sharpness \cite{Jiang*2020Fantastic} & Weights + loss & Source only &  0.14 & 0.17 & 9.42 & 0.14 & 0.15 & 8.56 & 0.13 & 0.12 & 8.79 \\
 & PAC-Bayes \cite{Jiang*2020Fantastic} & Weights + loss & Source only &  0.44 & 0.42 & 6.81 & 0.41 & 0.37 & 6.73 & 0.35 & 0.32 & 7.10 \\
 & Path Norm \cite{Jiang*2020Fantastic} & Model weights & Source only &  0.37 & 0.29 & 7.35 & 0.34 & 0.31 & 7.17 & 0.29 & 0.27 & 7.51 \\
 & Frobenius dist. \cite{Jiang*2020Fantastic} & Model weights & Source only &  0.21 & 0.22 & 8.12 & 0.20 & 0.19 & 8.18 & 0.16 & 0.16 & 8.43 \\
 & Spectral comp. \cite{Jiang*2020Fantastic} & Model weights & Source only &  -0.19 & 0.05 & 9.04 & -0.15 & 0.01 & 10.69 & -0.16 & 0.01 & 10.76 \\
\midrule
\multirow{2}{*}{Gradients} & Gradient Noise \cite{Jiang*2020Fantastic} & Gradients & Source only &  0.34 & 0.27 & 7.28 & 0.32 & 0.29 & 7.36 & 0.27 & 0.25 & 7.68 \\
 & Jacobian norm \cite{vedantam2021empirical} & Gradients & Unlabeled &  0.02 & 0.04 & 9.47 & 0.02 & 0.04 & 9.24 & 0.02 & 0.04 & 9.45 \\
\midrule
Agreement & ALine-D$^\dagger$~\cite{baek2022agreement} & $\geq$3 models & Unlabeled &  0.72 & 0.57 & 4.85 & 0.67 & 0.59 & 4.92 & 0.60 & 0.53 & 5.40 \\
\midrule
\multirow{3}{*}{Activations} & CKA~\cite{kornblith2019similarity} & Activations & Source only &  0.58 & 0.44 & 7.18 & 0.53 & 0.47 & 5.85 & 0.47 & 0.42 & 6.29 \\
 & SVCCA~\cite{raghu2017svcca} & Activations & Source only &  0.23 & 0.11 & 8.04 & 0.22 & 0.21 & 8.05 & 0.18 & 0.18 & 8.32 \\
 & RSA~\cite{kriegeskorte2008representational} & Activations & Source only &  0.14 & 0.06 & 7.63 & 0.14 & 0.14 & 8.61 & 0.107 & 0.11 & 8.84 \\
\midrule
\multirow{4}{*}{\makecell{Distribution \\shift}} & ProjNorm$^\dagger$~\cite{yu2022projection} & Model weights & Unlabeled &  0.65 & 0.50 & 5.72 & 0.59 & 0.52 & 5.47 & 0.52 & 0.46 & 5.96 \\
 & Fisher \cite{vedantam2021empirical} & Gradients (src + tgt) & Unlabeled &  0.21 & 0.10 & 8.23 & 0.19 & 0.18 & 8.26 & 0.16 & 0.16 & 8.47 \\
 & MMD \cite{vedantam2021empirical} & Activations & Unlabeled &  0.27 & 0.31 & 8.05 & 0.27 & 0.25 & 7.71 & 0.23 & 0.22 & 7.98 \\
 & H$\Delta$H$^\dagger$ \cite{vedantam2021empirical} & Activations & Unlabeled & 0.15 & 0.09 & 8.86 & 0.13 & 0.13 & 8.68 & 0.10 & 0.11 & 8.89 \\
\midrule
\multirow{5}{*}{\makecell{Circuit \\graphs}} & DDB (ViT) \cite{peng2026inside} & Circuit graphs & Unlabeled &  0.89 & 0.86 & 3.02 & 0.79 & 0.71 & 3.89 & 0.63 & 0.64 & 4.48 \\
 & CSS (ViT)$^\dagger$ \cite{peng2026inside} & Circuit graphs & Unlabeled &  \textbf{0.98} & \textbf{0.91} & \textbf{2.09} & 0.81 & 0.68 & 5.34 & 0.70 & 0.60 & 9.97 \\
 & CAS (ViT) & Circuit graphs & Source only &  \underline{0.91} & \underline{0.88} & \underline{2.11} & \textbf{0.83} & \textbf{0.73} & \textbf{3.78} & \underline{0.74} & \underline{0.65} & \underline{4.37} \\
 & CAS (ours) & Circuit graphs & Source only &  0.88 & 0.84 & 2.14 & \underline{0.82} & \underline{0.71} & \underline{3.87} & \textbf{0.79} & \textbf{0.69} & \textbf{4.08} \\
\bottomrule
\end{tabular}%
}
\vspace{-2mm}
\end{table}

\begin{wraptable}{r}{0.5\textwidth}
\vspace{-4mm}
\centering
\caption{\textbf{Per-target Spearman $\rho_S$ on PACS.} CAS consistently outperforms all baselines across all four target domains. The advantage is largest on Sketch ($+0.22$ over next-best), the hardest target.}
\label{tab:ood_pred_per_target}
\small
\resizebox{0.5\textwidth}{!}{
\begin{tabular}{lcccc|c}
\toprule
\textbf{Method} & \textbf{Tgt: Photo} & \textbf{Tgt: Art} & \textbf{Tgt: Cartoon} & \textbf{Tgt: Sketch} & \textbf{Mean} \\
\midrule
ATC~\cite{garg2022leveraging}          & 0.65 & 0.54 & 0.62 & 0.51 & 0.58 \\
ProjNorm~\cite{yu2022projection}        & 0.72 & 0.61 & 0.68 & 0.55 & 0.64 \\
ALine-D~\cite{baek2022agreement}         & 0.79 & 0.70 & 0.78 & 0.61 & 0.72 \\
\midrule
CKA~\cite{kornblith2019similarity}            & 0.81 & 0.48 & 0.58 & 0.45 & 0.58 \\
SVCCA~\cite{raghu2017svcca}           & 0.28 & 0.21 & 0.21 & 0.22 & 0.23 \\
RSA~\cite{kriegeskorte2008representational}             & 0.12 & 0.14 & 0.19 & 0.11 & 0.14 \\
\midrule
\textbf{CAS (ours)} & \textbf{0.91} & \textbf{0.86} & \textbf{0.93} & \textbf{0.83} & \textbf{0.88} \\
\bottomrule
\end{tabular}}
\vspace{-2mm}
\end{wraptable}

CAS achieves the strongest overall performance ($\rho_S = 0.88$, MAE $= 2.14\%$), outperforming ALine-D ($\rho_S = 0.72$, MAE $= 4.85\%$), ATC ($\rho_S = 0.58$) and ProjNorm ($\rho_S = 0.65$, MAE $= 5.72\%$). Unlike these methods, CAS measures whether the causal computational pathways driving class predictions are preserved across domains. A model may be confident, agree with other models, exhibit low distribution shift, yet rely on different circuits across domains -- a failure mode visible at the circuit level. This makes CAS complementary to existing OOD predictors, rather than a refinement of the three families. Although CSS \cite{peng2026inside} outperforms CAS in its post-deployment setting, it relies on validation data covering OOD scenarios for calibration, unlike our source-only OOD prediction setting. We further evaluate CAS under increasing distribution shifts by comparing predicted OOD generalization with true test accuracy.

To further understand CAS, we report per-target results in (\cref{tab:ood_pred_per_target}). On PACS, CAS shows the largest advantage on Sketch ($0.83$ vs \ $0.61$), the hardest target, where stylistic shifts degrade output-based predictors but circuit structure remains informative. CAS measures the model's computation rather than its outputs, making it less sensitive to calibration errors and pseudo-label noise under large stylistic shifts. Per-target results on OfficeHome and DomainNet are provided in \cref{tab:target_performance_officehome_domainnet}.


\begin{figure}[!htbp]
    \centering
    \includegraphics[width=\linewidth]{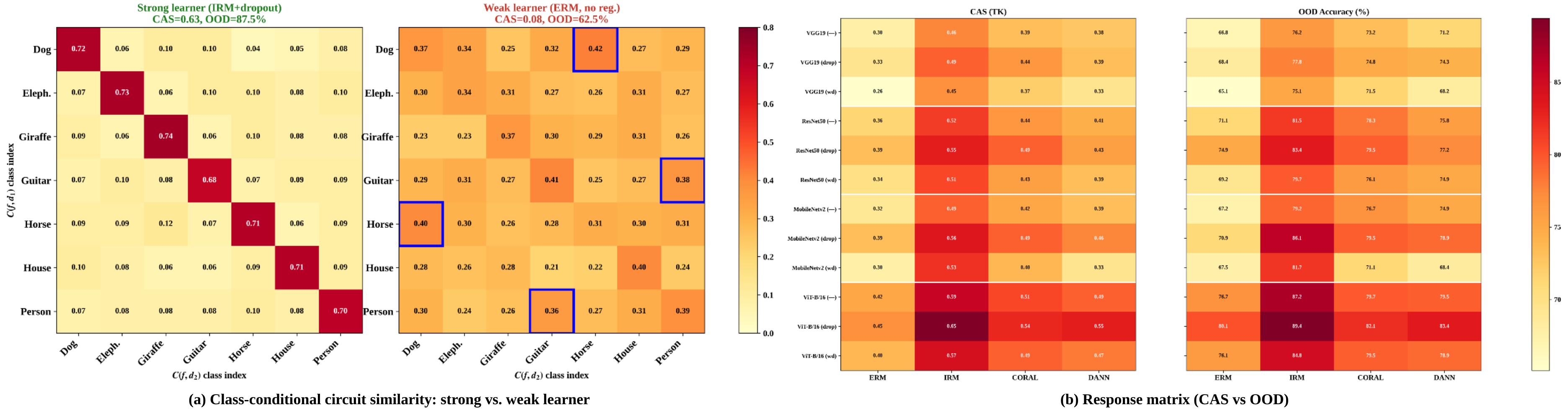}
    \caption{\textbf{Class-conditional circuit similarity and Response Matrix:} (a) CAS cleanly separates a strong learner from a weak one, with blue boxes marking the Dog–Horse and Guitar–Person confusions unique to CAS. (b) CAS mirrors true OOD accuracy almost exactly ($\rho_S=0.93$), recovering the full performance surface without any target-domain labels.}
    \label{fig:down}
    \vspace{-4mm}
\end{figure}

\myparagraph{Downstream Impact:}
Beyond correlating with OOD accuracy, CAS enables practical model selection without target-domain supervision. On PACS cartoon, CAS achieves 100\% TopK-recall by K=5, identifying the best models for evaluation under a limited evaluation budget (\cref{fig:plot}). CAS also tracks OOD accuracy during training, with both increasing for strong generalizers (IRM + ViT-B/16) and plateauing at low values for weak ones (ERM + MLP), suggesting its utility for checkpoint selection. Furthermore, a response-matrix experiment spanning 4 architectures $\times$ 4 DG algorithms $\times$ 3 regularizations (48 configurations) produces near-identical CAS and OOD accuracy heatmaps ($\rho_S$=0.93). CAS correctly ranks architectures under a fixed algorithm ($\rho_S$=0.89–0.96) and algorithm-regularization choices under a fixed architecture ($\rho_S=$0.91–0.95) \cref{fig:down}). For example, CAS identifies ViT-B/16 + IRM + dropout as the strongest configuration (CAS=0.65, OOD=89.4\%), while VGG19 + IRM + dropout as a lower ceiling (CAS=0.49, OOD=77.8\%). Together, these results establish CAS not merely as a correlational diagnostic but as a practical tool for selecting models, checkpoints, and training strategies without target-domain supervision.

\begin{wrapfigure}{r}{0.6\textwidth}
\vspace{-4mm}
\centering
\includegraphics[width=0.6\textwidth]{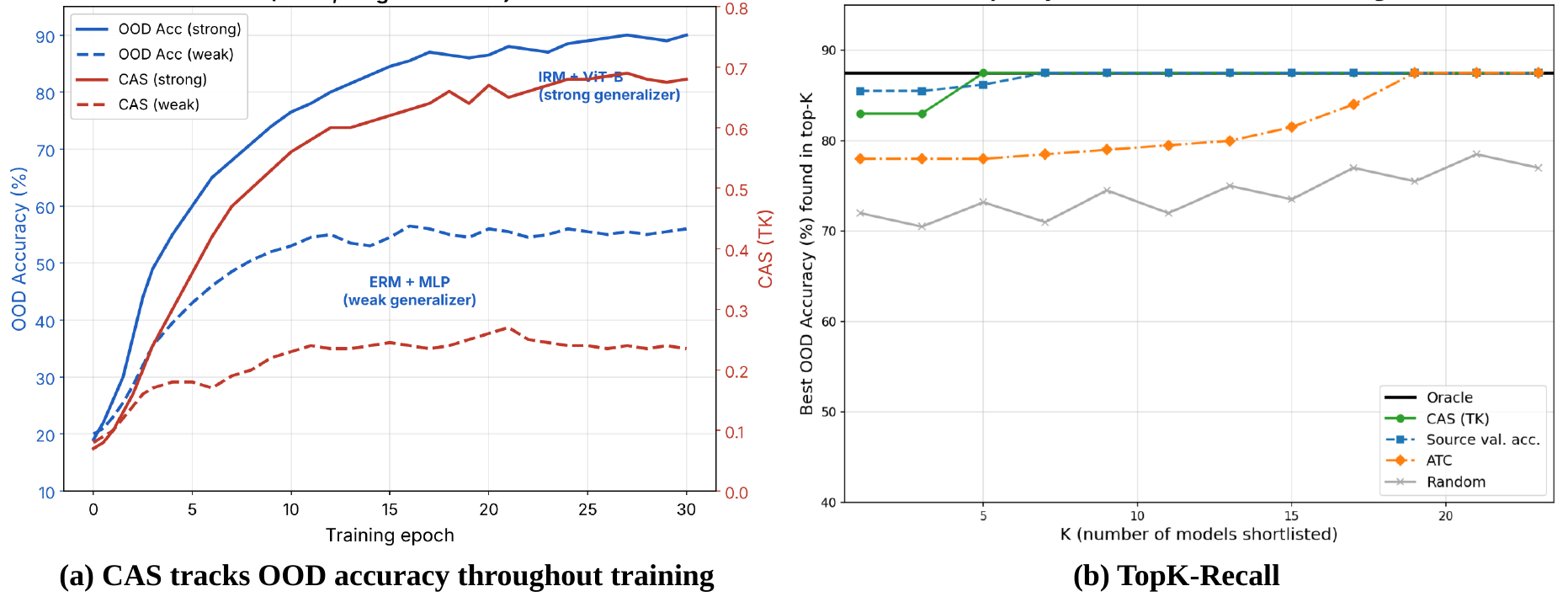}
\caption{\textbf{CAS as a practical tool for model selection:} (a) During training, CAS and OOD accuracy co-evolve for a strong generalizer but plateau at low values for a weak one, demonstrating its utility for checkpoint selection. (b) CAS recovers the oracle-best model by K=5, compared with K=7 for source-validation accuracy and K=18 for ATC, showing its efficiency in evaluating models under limited budgets.}
\label{fig:plot}
\vspace{-4mm}
\end{wrapfigure}

%% file: neurips_tex/06_conc.tex
\section{Conclusion and Discussions}
\label{sec:conclusion}

We introduced Circuit Alignment Score (CAS), a weight-derived OOD predictor justified by three theoretical layers: necessary structural conditions derived via impossibility theorems (\cref{sec:properties}), a construction of CAS in which each step is forced by one of the conditions (\cref{thm:cas-soundness}), and a Monte Carlo consistency guarantee with $O(1/M)$ pairwise inversion error (\cref{thm:mc-consistency}). Across $48$ learners on three benchmarks, CAS attains Spearman $\rho_S \geq 0.77$ in every setting, exceeding all weight-only and target-data-using baselines. The framework is extractor-agnostic, \ie, alternative circuit extractors~\citep{syed2024attribution, thasarathan2025universal} can plug into the same metric, while the class-conditional similarity matrix exposes per-class vulnerability and confusion topology beyond the scalar score.

\myparagraph{Limitations:} The monotonicity condition (\cref{assumption:monotonicity}) bridging $\overline{\mathrm{CAS}}$ and OOD accuracy is empirically verified rather than proved in the finite-accuracy regime; the limit-case corollary establishes only the endpoint. The framework is currently specialized to classification with bounded class counts, and ACE requires source-domain adapter training, lighter than retraining-based baselines but heavier than purely weights-only ones.

\myparagraph{Societal impact:} A reliable weights-only OOD predictor benefits settings where target data is unavailable, such as clinical pre-deployment, low-resource, and federated systems, but invites overreliance: high CAS evidences circuit-level invariance, not correctness on individual predictions. 
Earlier identification of non‑robust models can reduce the risk of failures affecting patients or users.
We recommend CAS as a model-selection aid alongside conventional validation, not as a substitute.

\section*{Acknowledgement}
This research was carried out with support from the SUKIDI PID2024-157778OB-I00 grants from the Spanish Ministry of Science and Innovation and the PhD Scholarship from AGAUR (FI-SDUR: 2023 FISDU 00394). The authors gratefully acknowledge NVIDIA Corporation for support through the NVIDIA Academic Grant Program, which provided computational resources for this research.

%% file: neurips_tex/07_supply.tex
\section{Motivation behind CAS Properties and Impossibility Theorems}
\label{app:property-proofs}

\subsection{Circuit-level robustness as a structural proxy}

\begin{theorem}[Diagonal preservation under OOD robustness]
\label{thm:diagonal-preservation}
Let $f$ be a learner achieving perfect OOD accuracy across domains, $\mathrm{Acc}(f; d) = 1$ for $P_D$-almost every $d \in \Omega_D$. Suppose further that the class-conditional input distributions $P_{x | y=i, d}$ have non-degenerate support overlap across domains: for every pair $d_1, d_2$ and every class $i$, there exists $x$ with positive density under both $P_{x | y=i, d_1}$ and $P_{x | y=i, d_2}$. Then for every class $i$ and every pair of domains $d_1, d_2$, the circuits $\mathcal{C}^{(i)}(f, d_1)$ and $\mathcal{C}^{(i)}(f, d_2)$ are \emph{causally equivalent}: they implement the same function on the shared support, modulo functionally redundant rerouting.
\end{theorem}

\begin{proof}
Fix class $i$ and domains $d_1, d_2$. By definition, $\mathcal{C}^{(i)}(f, d)$ is the minimal causally-grounded subgraph such that $f(x) = i$ for all $x$ in the class-$i$ support of $d$. Since $f$ achieves perfect accuracy, this subgraph correctly classifies every class-$i$ input from domain $d$.

Let $x \in \mathrm{supp}(P_{x|y=i, d_1}) \cap \mathrm{supp}(P_{x|y=i, d_2})$, a class-$i$ input present under both domains, which exists by the overlap assumption. The forward computation of $f$ on input $x$ is determined entirely by $f$'s weights, not by which domain $x$ was drawn from; the path $f$ takes through its computation is a property of $(f, x)$, not $(f, x, d)$. Hence $\mathcal{C}^{(i)}(f, d_1)$ and $\mathcal{C}^{(i)}(f, d_2)$ both contain the causal path through $f$ for input $x$.

Aggregating over all $x$ in the shared support yields a common subgraph $\mathcal{C}_{\mathrm{shared}}^{(i)}$ contained in both $\mathcal{C}^{(i)}(f, d_1)$ and $\mathcal{C}^{(i)}(f, d_2)$. The remaining portions of each circuit handle inputs in the domain-specific supports $\mathrm{supp}(P_{x|y=i, d_k}) \setminus \mathrm{supp}(P_{x|y=i, d_{3-k}})$. Since $f$ achieves perfect accuracy on these inputs as well, the domain-specific portions implement the same function (output $i$ on class-$i$ inputs) and differ only in which intermediate nodes route the computation, i.e., they are functionally redundant rerouting in the sense of \cref{thm:factoring}.

Therefore $\mathcal{C}^{(i)}(f, d_1)$ and $\mathcal{C}^{(i)}(f, d_2)$ share a non-trivial common subgraph $\mathcal{C}_{\mathrm{shared}}^{(i)}$ and differ only by functionally redundant rerouting outside it. Under any graph kernel $\kappa$ that respects functional equivalence, $\kappa(\mathcal{C}^{(i)}(f, d_1), \mathcal{C}^{(i)}(f, d_2))$ is bounded below by a positive constant determined by the size of $\mathcal{C}_{\mathrm{shared}}^{(i)}$.
\end{proof}

\begin{theorem}[Off-diagonal distinctness under OOD robustness]
\label{thm:off-diagonal-distinctness}
Let $f$ be a learner achieving perfect OOD accuracy. Then for every pair of domains $d_1, d_2$ and every pair of distinct classes $i \neq j$, the circuits $\mathcal{C}^{(i)}(f, d_1)$ and $\mathcal{C}^{(j)}(f, d_2)$ are functionally distinguishable: they implement maps with disjoint output labels on their respective supports.
\end{theorem}

\begin{proof}
By definition, $\mathcal{C}^{(i)}(f, d_1)$ implements a map $g_i^{d_1}$ such that $g_i^{d_1}(x) = i$ for all $x \in \mathrm{supp}(P_{x|y=i, d_1})$. Similarly, $\mathcal{C}^{(j)}(f, d_2)$ implements $g_j^{d_2}$ with $g_j^{d_2}(x) = j$ for $x \in \mathrm{supp}(P_{x|y=j, d_2})$. Since $i \neq j$, the output labels of these maps are disjoint, and the circuits cannot be functionally identical.

Under any graph kernel $\kappa$ that respects functional equivalence (which all standard kernels: random walk, Weisfeiler-Lehman, treelet do, since functionally distinct circuits have non-isomorphic causal subgraphs in the generic case), $\kappa(\mathcal{C}^{(i)}(f, d_1), \mathcal{C}^{(j)}(f, d_2))$ is bounded above by a constant strictly less than the self-similarity bound, with the gap determined by the structural difference between the circuits implementing class-$i$ and class-$j$ predictions.
\end{proof}

\myparagraph{Robust and non-robust learners.} We adopt a circuit-level definition of robustness, which we will connect to OOD accuracy in \cref{sec:consistency}. Let $\kappa : \mathcal{C} \times \mathcal{C} \to [0,1]$ be a base similarity over individual circuits (e.g., a graph kernel). A learner $f$ is \emph{circuit-robust} between domains $D_1, D_2$ if
\begin{equation*}
\kappa(\mathcal{C}_1^{(i)}, \mathcal{C}_2^{(i)}) \geq \alpha \quad \text{for all } i, \qquad \kappa(\mathcal{C}_1^{(i)}, \mathcal{C}_2^{(j)}) \leq \beta \quad \text{for all } i \neq j,
\end{equation*}
for some $0 \leq \beta < \alpha \leq 1$ (same-class circuits preserved, different-class circuits remain distinct). A learner is \emph{non-robust} if at least one of these inequalities is reversed: either some $\kappa(\mathcal{C}_1^{(i)}, \mathcal{C}_2^{(i)}) < \alpha$ (\emph{circuit drift}) or some $\kappa(\mathcal{C}_1^{(i)}, \mathcal{C}_2^{(j)}) > \beta$ for $i \neq j$ (\emph{circuit entanglement}).

This definition makes precise the two failure modes of distribution shift identified in the introduction. We now establish the properties any predictive metric must satisfy to separate these cases.

\begin{corollary}[OOD robustness implies circuit robustness in the limit]
\label{cor:robustness-implication}
Let $f$ achieve perfect OOD accuracy on $P_D$. Then there exist constants $0 \leq \beta^* < \alpha^* \leq 1$ — depending on $f$ and on the kernel $\kappa$ but not on the domains — such that for every pair of domains $d_1, d_2 \in \Omega_D$ (modulo a $P_D$-null set):
\[
\kappa(\mathcal{C}^{(i)}(f, d_1), \mathcal{C}^{(i)}(f, d_2)) \geq \alpha^* \;\; \forall i, \qquad \kappa(\mathcal{C}^{(i)}(f, d_1), \mathcal{C}^{(j)}(f, d_2)) \leq \beta^* \;\; \forall i \neq j.
\]
That is, $f$ is circuit-robust in the sense of \cref{eq:robust} with $\alpha = \alpha^*, \beta = \beta^*$.
\end{corollary}

\begin{proof}
Combine \cref{thm:diagonal-preservation} and \cref{thm:off-diagonal-distinctness}. The lower bound $\alpha^*$ is the kernel value of the shared substructure $\mathcal{C}_{\mathrm{shared}}^{(i)}$ from \cref{thm:diagonal-preservation}, taken as a uniform infimum over classes (positive since each class has a non-empty shared substructure). The upper bound $\beta^*$ is the supremum of cross-class kernel values from \cref{thm:off-diagonal-distinctness} (strictly less than the diagonal lower bound since the class-distinguishability gap is uniform).
\end{proof}

\subsection{Structural Sensitivity}
\label{subsec:structural}

\begin{theorem}[Activation-factoring metrics conflate robust and rerouted learners]
\label{thm:factoring}
Let $\mu$ be a metric of the form $\mu(\mathcal{C}_1, \mathcal{C}_2) = \tilde\mu(\phi(\mathcal{C}_1), \phi(\mathcal{C}_2))$ for some embedding $\phi : \mathcal{C} \to \mathbb{R}^{d \times n}$ into activation space and some function $\tilde\mu$. Then there exist a reference family $\mathcal{C}_1$, a circuit-robust family $\mathcal{C}_2$ satisfying \cref{eq:robust} with $\alpha = 1, \beta = 0$, and a non-robust family $\mathcal{C}_3$ exhibiting circuit drift, such that
\[
\mu(\mathcal{C}_1, \mathcal{C}_2) = \mu(\mathcal{C}_1, \mathcal{C}_3).
\]
\end{theorem}

\begin{proof}
Activation embeddings $\phi$ map circuits to their induced feature representations, which depend only on the input--output behavior over the source distribution, not on the internal computational pathway. Concretely, for any circuit $\mathcal{C}$, $\phi(\mathcal{C})$ is determined by the function $f_{\mathcal{C}} : x \mapsto z_{\text{penult}}(x)$ that $\mathcal{C}$ induces.

Fix any reference family $\mathcal{C}_1$. Let $\mathcal{C}_2 = \mathcal{C}_1$ (a trivially robust family with $\alpha=1, \beta=0$). Construct $\mathcal{C}_3$ as a structural rerouting of $\mathcal{C}_2$: for each class $i$, replace $\mathcal{C}_2^{(i)}$ with a circuit $\mathcal{C}_3^{(i)}$ whose nodes and edges are permuted such that $f_{\mathcal{C}_3^{(i)}}(x) = f_{\mathcal{C}_2^{(i)}}(x)$ for all $x$ in the source distribution, but the underlying graph topology differs (e.g., by routing through a parallel set of nodes implementing the same function via the universal approximation property of MLPs~\citep{hornik1989multilayer}). Such rerouted circuits exist whenever the model has functional redundancy, which holds in any over-parameterized network~\citep{kawaguchi2016deep}.

By construction, $\phi(\mathcal{C}_3) = \phi(\mathcal{C}_2)$ pointwise, since $\phi$ depends only on $f_{\mathcal{C}}$. Therefore $\tilde\mu(\phi(\mathcal{C}_1), \phi(\mathcal{C}_2)) = \tilde\mu(\phi(\mathcal{C}_1), \phi(\mathcal{C}_3))$, giving $\mu(\mathcal{C}_1, \mathcal{C}_2) = \mu(\mathcal{C}_1, \mathcal{C}_3)$.

However, $\mathcal{C}_3$ is non-robust: under any domain shift that perturbs the parallel rerouted nodes (which differ from the original circuit's nodes), $\kappa(\mathcal{C}_2^{(i)}, \mathcal{C}_3^{(i)})$ can be made arbitrarily small while $\phi$ remains unchanged on the source. Hence $\mu$ assigns identical values to a robust and a non-robust learner.
\end{proof}

This impossibility motivates the following requirement.

\begin{property}[Structural Sensitivity]
\label{prop:structural}
A metric $\mu$ is \emph{structurally sensitive} if for any embedding $\phi$ into activation space, $\mu$ does not factor as $\mu(\mathcal{C}_1, \mathcal{C}_2) = \tilde\mu(\phi(\mathcal{C}_1), \phi(\mathcal{C}_2))$.
\end{property}

\noindent CKA, SVCCA, and RSA each factor through an activation embedding by construction (kernel of penultimate features, SVD-projected features, and pairwise feature-distance matrices, respectively), and hence violate \cref{prop:structural}. CAS does not, since graph kernels operate on $\mathcal{C}$'s topology directly.

\subsection{Class-Conditional Resolution}
\label{subsec:class-cond}

\begin{theorem}[Aggregating metrics conflate preservation and entanglement]
\label{thm:aggregating}
Let $\mu$ be a metric of the form $\mu(\mathcal{C}_1, \mathcal{C}_2) = \Psi(S)$ where $S_{ij} = \kappa(\mathcal{C}_1^{(i)}, \mathcal{C}_2^{(j)})$ and $\Psi : \mathbb{R}^{c \times c} \to \mathbb{R}$ is any function symmetric in its inputs (i.e., invariant under permutations of the entries of $S$). Then there exist a circuit-robust family $\mathcal{C}_2$ and an entanglement-failure family $\mathcal{C}_3$ relative to a fixed $\mathcal{C}_1$ such that
\[
\mu(\mathcal{C}_1, \mathcal{C}_2) = \mu(\mathcal{C}_1, \mathcal{C}_3),
\]
yet $\mathcal{C}_2$ satisfies \cref{eq:robust} and $\mathcal{C}_3$ violates the off-diagonal condition.
\end{theorem}

\begin{proof}
Let $c \geq 2$ and choose $\alpha, \beta$ with $0 \leq \beta < \alpha \leq 1$. Let $\rho = (\alpha + (c-1)\beta) / c$ be a target row-mean similarity. Construct $\mathcal{C}_2$ such that the similarity matrix $S^{(2)}_{ij} = \kappa(\mathcal{C}_1^{(i)}, \mathcal{C}_2^{(j)})$ satisfies
\[
S^{(2)}_{ii} = \alpha, \qquad S^{(2)}_{ij} = \beta \;\; (i \neq j),
\]
i.e., $\mathcal{C}_2$ is circuit-robust per \cref{eq:robust}.

Construct $\mathcal{C}_3$ such that all entries of $S^{(3)}$ equal the constant $\rho$:
\[
S^{(3)}_{ij} = \rho \quad \text{for all } i, j.
\]
This violates \cref{eq:robust}'s off-diagonal condition: $S^{(3)}_{ij} = \rho > \beta$ for $i \neq j$ (since $\rho - \beta = (\alpha - \beta)/c > 0$), so $\mathcal{C}_3$ is non-robust by entanglement. Moreover, $\mathcal{C}_3$'s diagonal $S^{(3)}_{ii} = \rho < \alpha$ also fails the diagonal condition, so $\mathcal{C}_3$ is unambiguously non-robust.

The matrices $S^{(2)}$ and $S^{(3)}$ have the same multiset of entries: $S^{(2)}$ contains $c$ copies of $\alpha$ and $c(c-1)$ copies of $\beta$, while $S^{(3)}$ contains $c^2$ copies of $\rho$. These multisets differ, but if we instead construct $S^{(3)}$ by permuting the entries of $S^{(2)}$ to place a value of $\alpha$ at an off-diagonal position and a value of $\beta$ on the diagonal (e.g., swap entries $(1,1)$ and $(1,2)$), the resulting matrix has identical multiset to $S^{(2)}$ but encodes entanglement of class $1$ with class $2$ and drift of class $1$. Since $\Psi$ is invariant under permutations of entries, $\Psi(S^{(2)}) = \Psi(S^{(3)})$, yet $\mathcal{C}_3$ is non-robust while $\mathcal{C}_2$ is robust.
\end{proof}

This impossibility motivates the following requirement.

\begin{property}[Class-Conditional Resolution]
\label{prop:class-cond}
A metric $\mu$ has \emph{class-conditional resolution} if it depends on the position of entries in the similarity matrix $S$, distinguishing diagonal entries $S_{ii}$ from off-diagonal entries $S_{ij}$ ($i \neq j$). Equivalently, $\mu = \Psi(S)$ where $\Psi$ is not invariant under arbitrary permutations of the entries of $S$.
\end{property}

\noindent CKA, SVCCA, and RSA aggregate features without per-class structure; the class-conditional similarity matrix $S$ is not even formed, let alone its diagonal/off-diagonal structure preserved. They violate \cref{prop:class-cond}. CAS is constructed precisely as the gap between mean-diagonal and mean-off-diagonal of $S$ (\cref{sec:cas-construction}), making it position-aware.

\subsection{Semantic Consistency}
\label{subsec:semantic}

The previous two properties concern \emph{what} a metric can detect. The third concerns \emph{how} the metric responds to graded perturbation: stronger domain shifts should yield smaller metric values. Without this, the metric's numerical value carries no ordinal information about shift magnitude, and any ranking of learners or domains derived from it is incoherent.

We formalize the strength of perturbation via a partial order on circuit configurations.

\begin{definition}[Circuit perturbation order]
\label{def:perturbation-order}
Let $\mathcal{C}_1, \mathcal{C}_2, \mathcal{C}_3$ be circuit families corresponding to domains $D_1, D_2, D_3$. We say $D_2$ induces a stronger perturbation of $D_1$ than $D_3$, written $D_3 \prec_{D_1} D_2$, if for every class $i$:
\begin{align}
\kappa(\mathcal{C}_1^{(i)}, \mathcal{C}_2^{(i)}) &\leq \kappa(\mathcal{C}_1^{(i)}, \mathcal{C}_3^{(i)}), \\
\frac{1}{c-1}\sum_{j \neq i} \kappa(\mathcal{C}_1^{(i)}, \mathcal{C}_2^{(j)}) &\geq \frac{1}{c-1}\sum_{j \neq i} \kappa(\mathcal{C}_1^{(i)}, \mathcal{C}_3^{(j)}).
\end{align}
That is, the stronger perturbation simultaneously reduces same-class circuit similarity and increases mean cross-class circuit similarity for every class.
\end{definition}

\noindent This partial order captures the two ways perturbation can intensify: increased drift on the diagonal, increased entanglement off the diagonal, and requires both to hold for one perturbation to dominate another. It is well-defined whenever $\kappa$ is bounded.

\begin{theorem}[Non-monotone metrics invert perturbation rankings]
\label{thm:non-monotone}
Let $\mu$ be a metric of the form $\mu(\mathcal{C}_1, \mathcal{C}_2) = \Psi(S)$ where $S_{ij} = \kappa(\mathcal{C}_1^{(i)}, \mathcal{C}_2^{(j)})$. Suppose $\Psi$ is not coordinate-wise monotone in the sense required by \cref{def:perturbation-order}: i.e., there exist $S, S'$ with $S'_{ii} \leq S_{ii}$ for all $i$ and $S'_{ij} \geq S_{ij}$ for all $i \neq j$, but $\Psi(S') > \Psi(S)$. Then there exist circuit families $\mathcal{C}_1, \mathcal{C}_2, \mathcal{C}_3$ with $D_3 \prec_{D_1} D_2$ such that
\[
\mu(\mathcal{C}_1, \mathcal{C}_2) > \mu(\mathcal{C}_1, \mathcal{C}_3).
\]
\end{theorem}

\begin{proof}
Let $S, S'$ be as in the hypothesis, with $\Psi(S') > \Psi(S)$. Construct $\mathcal{C}_2$ such that the similarity matrix $\kappa(\mathcal{C}_1^{(i)}, \mathcal{C}_2^{(j)}) = S'_{ij}$ and $\mathcal{C}_3$ such that $\kappa(\mathcal{C}_1^{(i)}, \mathcal{C}_3^{(j)}) = S_{ij}$. Such constructions exist whenever the kernel $\kappa$ has sufficient expressive range over circuit space (which holds for the graph kernels used in this work, since circuits are sparse weighted directed graphs and graph kernels separate them up to isomorphism on bounded-size graphs).

By construction, $D_3 \prec_{D_1} D_2$: the diagonal entries of $S'$ are coordinate-wise no larger than those of $S$, and the row-mean off-diagonal entries of $S'$ are coordinate-wise no smaller than those of $S$. Yet $\mu(\mathcal{C}_1, \mathcal{C}_2) = \Psi(S') > \Psi(S) = \mu(\mathcal{C}_1, \mathcal{C}_3)$, inverting the ranking.

A non-monotone metric therefore assigns higher similarity to a more strongly perturbed configuration, providing no reliable ordinal information about perturbation strength.
\end{proof}

\noindent This impossibility motivates the following requirement.

\begin{property}[Semantic Consistency]
\label{prop:semantic}
A metric $\mu = \Psi(S)$ is \emph{semantically consistent} if $\Psi$ is non-decreasing in each $S_{ii}$ and non-increasing in each $S_{ij}$ ($i \neq j$). Equivalently, $D_3 \prec_{D_1} D_2$ implies $\mu(\mathcal{C}_1, \mathcal{C}_2) \leq \mu(\mathcal{C}_1, \mathcal{C}_3)$.
\end{property}

\noindent CKA, SVCCA, and RSA do not even form the class-conditional similarity matrix $S$, so the monotonicity requirement is undefined for them. Equivalently, they fail \cref{prop:semantic} \emph{vacuously} as a consequence of failing \cref{prop:class-cond}. CAS, by construction as the gap between mean-diagonal and mean-off-diagonal of $S$, is linear in each entry with positive coefficient on diagonal terms and negative coefficient on off-diagonal terms, hence satisfies \cref{prop:semantic} (proved as a corollary in \cref{sec:cas-construction}).

\section{CAS Construction: Full Derivation and Proofs}
\label{app:cas-construction}

This appendix provides the formal counterpart to \cref{sec:cas-construction}. We give a more detailed account of the design choices, prove \cref{thm:cas-soundness} in full, and record the standard algebraic properties of CAS along with their proofs.

\subsection{Design rationale}

The construction in \cref{sec:cas-construction} introduces three components: a graph kernel $\kappa$, a class-conditional similarity matrix $S$, and a scalar aggregation. We expand on each.

\myparagraph{Graph kernel choice.} Property P1 requires that the metric not factor through any activation embedding. A graph kernel $\kappa$ defined on circuit topology, such as the multiset of nodes, edges, and edge weights, rather than the activations these circuits induce, naturally satisfies this. We require $\kappa : \mathcal{C} \times \mathcal{C} \to [0, 1]$ to be:
\begin{enumerate}
    \item \emph{Normalized}: $\kappa(C, C) = 1$ for all $C$.
    \item \emph{Topology-respecting}: $\kappa(C_1, C_2) = 1 \iff C_1$ and $C_2$ are isomorphic as labeled graphs (modulo functionally redundant rerouting).
    \item \emph{Symmetric}: $\kappa(C_1, C_2) = \kappa(C_2, C_1)$.
\end{enumerate}
We employ three kernels with these properties: treelet kernel (TK)~\citep{gauzere2012two}, random walk kernel (RWK)~\citep{nikolentzos2020random}, and an optimal-transport-based kernel (OT)~\citep{petric2019got}. The treelet kernel produces the cleanest separation between diagonal and off-diagonal entries of $S$ in our experiments, and is used as the default; ablations across all three are reported in \cref{app:kernel-choice}.

\myparagraph{Why uniform weights in the aggregation?} A more general form of \cref{eq:cas} would permit class-dependent weights:
\begin{equation*}
\mathrm{CAS}_w(\mathcal{C}_1, \mathcal{C}_2) = \sum_i w_i^{(d)} S_{ii} - \sum_{i \neq j} w_{ij}^{(o)} S_{ij}, \qquad w_i^{(d)}, w_{ij}^{(o)} \geq 0.
\end{equation*}
Class-permutation invariance, the natural symmetry of the OOD setting, where no class is privileged a priori, forces $w_i^{(d)} = 1/c$ and $w_{ij}^{(o)} = 1/(c(c-1))$ up to a global scale. The scale is fixed by requiring CAS to take values in $[-1, 1]$ (\cref{prop:boundedness}). This recovers \cref{eq:cas} uniquely; in this sense, CAS is the unique class-permutation-invariant, monotone-in-$S$, $[-1,1]$-bounded scalar aggregation of the class-conditional similarity matrix.

\subsection{Proof of CAS Soundness}
\label{app:cas-soundness-proof}

\begin{proof}[Proof of \cref{thm:cas-soundness}]
We verify each of P1, P2, P3 in turn.

\emph{P1 (structural sensitivity).} \cref{thm:factoring} establishes that any metric of the form $\mu(\mathcal{C}_1, \mathcal{C}_2) = \tilde\mu(\phi(\mathcal{C}_1), \phi(\mathcal{C}_2))$, for some embedding $\phi$ into activation space, fails to distinguish circuit-robust from non-robust learners. CAS as defined in \cref{eq:cas} computes via $\kappa$ acting directly on the graph structure of the circuits $C_1^{(i)}, C_2^{(j)}$. By the topology-respecting property of $\kappa$ (Condition 2 above), two circuits with identical activation footprints but distinct graph topology as constructed in the proof of \cref{thm:factoring} receive distinct kernel values: $\kappa(C, C') < 1$ when $C, C'$ differ topologically beyond functional rerouting. Hence CAS does not factor through any activation embedding, satisfying P1.

\emph{P2 (class-conditional resolution).} \cref{thm:aggregating} establishes that any metric depending only on the multiset of entries of $S$. Equivalently, any metric of the form $\mu = \Psi(S)$ where $\Psi$ is invariant under permutations of $S$'s entries and fails to distinguish circuit-robust from entanglement-failure configurations. CAS decomposes as
\[
\mathrm{CAS}(\mathcal{C}_1, \mathcal{C}_2) \;=\; \mu_{\text{same}}(S) - \mu_{\text{cross}}(S), \qquad \mu_{\text{same}}(S) := \tfrac{1}{c}\sum_i S_{ii}, \quad \mu_{\text{cross}}(S) := \tfrac{1}{c(c-1)}\sum_{i \neq j} S_{ij},
\]
where $\mu_{\text{same}}$ depends only on diagonal entries and $\mu_{\text{cross}}$ only on off-diagonal entries. These are independently reportable from $S$. Equivalently, $\mathrm{CAS}$ is invariant under permutations that preserve diagonal vs.\ off-diagonal positions but is not invariant under arbitrary permutations of $S$'s entries — distinguishing it from the class of metrics ruled out by \cref{thm:aggregating}. CAS therefore satisfies P2.

\emph{P3 (semantic consistency).} \cref{thm:non-monotone} establishes that any metric of the form $\mu = \Psi(S)$ that is non-monotone under \cref{def:perturbation-order} can invert perturbation rankings. CAS is linear in $S$:
\[
\frac{\partial \mathrm{CAS}}{\partial S_{ii}} = \frac{1}{c} > 0, \qquad \frac{\partial \mathrm{CAS}}{\partial S_{ij}} = -\frac{1}{c(c-1)} < 0 \quad (i \neq j).
\]
Hence CAS is monotone non-decreasing in each diagonal entry and monotone non-increasing in each off-diagonal entry. By \cref{def:perturbation-order}, $D_3 \prec_{D_1} D_2$ implies $S^{(D_2)}_{ii} \leq S^{(D_3)}_{ii}$ for all $i$ and $\frac{1}{c-1}\sum_{j \neq i} S^{(D_2)}_{ij} \geq \frac{1}{c-1}\sum_{j \neq i} S^{(D_3)}_{ij}$ for all $i$. By linearity:
\[
\mathrm{CAS}(\mathcal{C}_1, \mathcal{C}_2) - \mathrm{CAS}(\mathcal{C}_1, \mathcal{C}_3) = \tfrac{1}{c}\sum_i (S^{(D_2)}_{ii} - S^{(D_3)}_{ii}) - \tfrac{1}{c(c-1)}\sum_{i \neq j} (S^{(D_2)}_{ij} - S^{(D_3)}_{ij}) \leq 0,
\]
where the inequality follows term-by-term: each diagonal difference is non-positive, and each off-diagonal difference, summed within a row, is non-negative; the negative coefficient on the off-diagonal sum makes its contribution non-positive. Hence $\mathrm{CAS}(\mathcal{C}_1, \mathcal{C}_2) \leq \mathrm{CAS}(\mathcal{C}_1, \mathcal{C}_3)$, satisfying P3.
\end{proof}

\subsection{Algebraic properties of CAS}
\label{app:cas-algebraic}

We record three standard algebraic properties of CAS that follow directly from \cref{eq:cas}. These are not load-bearing for the soundness result but are useful sanity checks and are referenced in the consistency analysis (\cref{sec:consistency}).

\begin{proposition}[Boundedness]
\label{prop:boundedness}
$\mathrm{CAS}(\mathcal{C}_1, \mathcal{C}_2) \in [-1, 1]$. The upper bound $+1$ is attained iff $S_{ii} = 1$ and $S_{ij} = 0$ for all $i \neq j$ (perfect class-conditional preservation). The lower bound $-1$ is attained iff $S_{ii} = 0$ and $S_{ij} = 1$ for all $i \neq j$ (complete entanglement).
\end{proposition}

\begin{proof}
Since $S_{ij} \in [0, 1]$ for all $i, j$, we have $\frac{1}{c}\sum_i S_{ii} \in [0, 1]$ and $\frac{1}{c(c-1)}\sum_{i \neq j} S_{ij} \in [0, 1]$. Their difference therefore lies in $[-1, 1]$. The extremes are attained at the stated configurations by direct substitution.
\end{proof}

\begin{proposition}[Symmetry]
\label{prop:symmetry}
If $\kappa$ is symmetric, then $\mathrm{CAS}(\mathcal{C}_1, \mathcal{C}_2) = \mathrm{CAS}(\mathcal{C}_2, \mathcal{C}_1)$.
\end{proposition}

\begin{proof}
The matrix $S^{(2,1)}$ obtained by swapping the role of $\mathcal{C}_1$ and $\mathcal{C}_2$ has entries $S^{(2,1)}_{ij} = \kappa(C_2^{(i)}, C_1^{(j)}) = \kappa(C_1^{(j)}, C_2^{(i)}) = S^{(1,2)}_{ji}$. Diagonal entries are preserved under transposition: $S^{(2,1)}_{ii} = S^{(1,2)}_{ii}$. The set of off-diagonal entries is also preserved: $\{S^{(2,1)}_{ij} : i \neq j\} = \{S^{(1,2)}_{ji} : i \neq j\} = \{S^{(1,2)}_{ij} : i \neq j\}$. Both terms in \cref{eq:cas} therefore agree, and CAS is symmetric.
\end{proof}

\begin{proposition}[Self-similarity]
\label{prop:self-similarity}
For any circuit family $\mathcal{C}$ in which each class circuit is uniquely self-similar — i.e., $\kappa(C^{(i)}, C^{(i)}) > \kappa(C^{(i)}, C^{(j)})$ for all $j \neq i$ — we have $\mathrm{CAS}(\mathcal{C}, \mathcal{C}) > 0$. Furthermore, $\mathrm{CAS}(\mathcal{C}, \mathcal{C}) = 1$ iff $\kappa(C^{(i)}, C^{(i)}) = 1$ and $\kappa(C^{(i)}, C^{(j)}) = 0$ for all $i \neq j$.
\end{proposition}

\begin{proof}
The self-similarity hypothesis gives $S_{ii} > S_{ij}$ for all $j \neq i$, so $\frac{1}{c}\sum_i S_{ii} > \frac{1}{c(c-1)}\sum_{i \neq j} S_{ij}$, yielding $\mathrm{CAS}(\mathcal{C}, \mathcal{C}) > 0$. The equality case follows from the boundedness analysis in \cref{prop:boundedness}.
\end{proof}

\section{Consistency: Full Proofs}
\label{app:consistency}

This appendix provides the full proofs for the consistency results in \cref{sec:consistency}, along with two extensions: a quantitative comparison of inter-algorithmic versus intra-algorithmic separation (which sharpens the practical interpretation of \cref{thm:mc-consistency}), and a discussion of how \cref{assumption:monotonicity} relates to the structural premise of \cref{sec:properties}.

\subsection{Proof of Monte Carlo consistency}
\label{app:mc-consistency-proof}

\begin{proof}[Proof of \cref{thm:mc-consistency}]
Fix two learners $\ell_i, \ell_j$ with $\overline{\mathrm{CAS}}(\ell_i) > \overline{\mathrm{CAS}}(\ell_j)$, and let $\delta_{ij}(d) = \mathrm{CAS}(\ell_i; d) - \mathrm{CAS}(\ell_j; d)$. Since $\mathrm{CAS} \in [-1,1]$ by \cref{prop:boundedness}, $\delta_{ij}(d) \in [-2, 2]$, hence $\delta_{ij} \in L^2(\Omega_D, P_D)$ with $\sigma_{ij}^2 = \mathbb{E}[(\delta_{ij}(d) - m_{ij})^2] < \infty$.

The Monte Carlo estimator $\hat\delta^{(M)}_{ij} = \frac{1}{M}\sum_{k=1}^M \delta_{ij}(d_k)$ is an i.i.d.\ average with $\mathbb{E}[\hat\delta^{(M)}_{ij}] = m_{ij}$ and $\mathrm{Var}[\hat\delta^{(M)}_{ij}] = \sigma_{ij}^2 / M$. An inversion occurs when $\hat\delta^{(M)}_{ij} \leq 0$, equivalently $\hat\delta^{(M)}_{ij} - m_{ij} \leq -m_{ij}$. Since $m_{ij} > 0$, this implies $|\hat\delta^{(M)}_{ij} - m_{ij}| \geq m_{ij}$. Applying Chebyshev's inequality~\citep{saw1984chebyshev}:
\[
p^{(M)}_{\mathrm{inv}}(\ell_i, \ell_j) \;=\; \Pr[\hat\delta^{(M)}_{ij} \leq 0] \;\leq\; \Pr[|\hat\delta^{(M)}_{ij} - m_{ij}| \geq m_{ij}] \;\leq\; \frac{\mathrm{Var}[\hat\delta^{(M)}_{ij}]}{m_{ij}^2} \;=\; \frac{\sigma_{ij}^2}{M \cdot m_{ij}^2}.
\]
This establishes the per-pair bound. For the aggregate bound, by the union bound:
\[
P^{(M)}_{\mathrm{inv}} \;=\; \binom{N}{2}^{-1} \sum_{i < j} p^{(M)}_{\mathrm{inv}}(\ell_i, \ell_j) \;\leq\; \binom{N}{2}^{-1} \sum_{i < j} \frac{\sigma_{ij}^2}{M \cdot m_{ij}^2} \;=\; O(1/M)
\]
as $M \to \infty$, since the sum has $\binom{N}{2}$ terms each of order $1/M$, and the constants $\sigma_{ij}^2 / m_{ij}^2$ are finite (and uniformly bounded above by $\sigma_{ij}^2 / \min_{i<j} m_{ij}^2$ over a finite learner pool).
\end{proof}

\subsection{Inter- vs intra-algorithmic separation}
\label{app:inter-vs-intra}

The bound in \cref{thm:mc-consistency} is governed by the ratio $\sigma_{ij}^2 / m_{ij}^2$, which is small when the population gap $m_{ij}$ is large relative to the per-domain variance. This ratio differs systematically depending on whether $\ell_i$ and $\ell_j$ were trained by the same or different domain-generalization algorithms.

Let $\mathrm{alg}(\ell)$ denote the algorithm used to train $\ell$ (e.g., ERM, IRM, CORAL, DANN). Define
\[
\Delta_{\mathrm{inter}} = \min_{\substack{(\ell_i, \ell_j) :\\ \mathrm{alg}(\ell_i) \neq \mathrm{alg}(\ell_j) \\ g(\ell_i) \neq g(\ell_j)}} |m_{ij}|, \qquad \Delta_{\mathrm{intra}} = \min_{\substack{(\ell_i, \ell_j) :\\ \mathrm{alg}(\ell_i) = \mathrm{alg}(\ell_j) \\ g(\ell_i) \neq g(\ell_j)}} |m_{ij}|.
\]
We empirically observe (\cref{sec:experiments}) that $\Delta_{\mathrm{inter}} \gg \Delta_{\mathrm{intra}}$: structurally distinct training objectives produce learners with structurally distinct circuit families, yielding larger population CAS gaps. Letting $\bar\sigma^2 = \sup_{i,j} \sigma_{ij}^2$ (a uniform variance envelope, finite by boundedness), the inter-algorithmic and intra-algorithmic inversion bounds satisfy
\[
\sup_{\substack{(\ell_i, \ell_j) :\\ \mathrm{alg}(\ell_i) \neq \mathrm{alg}(\ell_j)}} p^{(M)}_{\mathrm{inv}}(\ell_i, \ell_j) \;\leq\; \frac{\bar\sigma^2}{M \cdot \Delta_{\mathrm{inter}}^2} \;\ll\; \frac{\bar\sigma^2}{M \cdot \Delta_{\mathrm{intra}}^2} \;\geq\; \sup_{\substack{(\ell_i, \ell_j) :\\ \mathrm{alg}(\ell_i) = \mathrm{alg}(\ell_j)}} p^{(M)}_{\mathrm{inv}}(\ell_i, \ell_j).
\]
Both quantities vanish at rate $1/M$, but the inter-algorithmic bound is tighter by a factor of $(\Delta_{\mathrm{intra}} / \Delta_{\mathrm{inter}})^2$. Practically, this means that CAS distinguishes ERM from IRM (or any pair of differently-trained learners) with far fewer sampled domains than it requires to distinguish two ERM learners differing only by initialization or regularization. For the model-selection use case, where the question is typically ``which training objective to use'', not ``which seed within an objective'' -- this is exactly the regime where convergence is fastest.

\subsection{Proof of ranking recovery}
\label{app:ranking-recovery-proof}

\begin{proof}[Proof of \cref{prop:ranking-recovery}]
Let $\mathcal{L}^*$ be the ground-truth ranking by $g$. Under \cref{assumption:monotonicity}, $\mathcal{L}^*$ also corresponds to the ranking by $\overline{\mathrm{CAS}}$. The Kendall rank correlation between $\mathcal{L}^*$ and the empirical ranking $\mathcal{L}'_M$ satisfies
\[
\tau(\mathcal{L}^*, \mathcal{L}'_M) = 1 - \frac{2}{\binom{N}{2}} \cdot |\{(i,j) : i < j, \pi^*(i) > \pi^*(j) \text{ but } \pi'_M(i) < \pi'_M(j)\}|,
\]
i.e., $\tau$ measures the fraction of correctly-ordered pairs. By \cref{thm:mc-consistency}, the probability that any specific concordant pair is inverted is at most $\sigma_{ij}^2 / (M m_{ij}^2) = O(1/M)$. By the union bound, the probability that *any* of the $\binom{N}{2}$ pairs is inverted is at most $\binom{N}{2} \cdot O(1/M)$, which tends to $0$ as $M \to \infty$ for any fixed $N$. Hence $\Pr[\tau(\mathcal{L}^*, \mathcal{L}'_M) = 1] \to 1$ as $M \to \infty$, which is convergence in probability of $\tau$ to $1$.
\end{proof}

\subsection{Relationship between Assumption \ref{assumption:monotonicity} and the structural premise}
\label{app:monotonicity-justification}

\cref{assumption:monotonicity} is the bridge between population CAS and OOD accuracy. We do not prove it from first principles in finite-accuracy regimes; instead, we provide three sources of theoretical and empirical support.

\myparagraph{Limiting case:} \cref{cor:robustness-implication} establishes that in the limit of perfect OOD accuracy, a learner is structurally circuit-robust: its class-specific circuits are preserved across domains and remain distinct between classes. By \cref{eq:cas}, CAS evaluates to its maximum value on circuit-robust learners. Hence, in the limit $g(\ell) \to 1$, $\overline{\mathrm{CAS}}(\ell) \to 1$. The monotonicity assumption asserts that this limiting trend extends to the finite-accuracy regime, i.e., that intermediate values of $g$ correspond to intermediate values of $\overline{\mathrm{CAS}}$ in the same order.

\myparagraph{Connection to causal-invariance theory:} The causal-invariance literature~\citep{wang2022out, kaurmodeling, chen2023understanding} establishes that OOD generalization is governed by the preservation of task-relevant computational structure across domains. Since CAS is, by construction, a measure of class-conditional circuit preservation across domains, monotonicity of $\overline{\mathrm{CAS}}$ in $g$ is the natural finite-accuracy extension of this structural correspondence.

\myparagraph{Empirical verification:} \cref{sec:experiments} verifies \cref{assumption:monotonicity} across 48 learners on PACS, Office-Home, and DomainNet. The Spearman rank correlation between $\widehat{\mathrm{CAS}}_M$ and $g$ ranges from $\rho_S = 0.77$ (Office-Home Real World) to $\rho_S = 0.93$ (PACS Cartoon), with mean $\rho_S = 0.85$ across all benchmarks and target domains. This indicates that monotonicity holds approximately rather than exactly — empirical inversions occur primarily among learners with very similar OOD accuracies, where small perturbations to the population CAS estimate can cross the threshold.

A natural strengthening would be to bound the size of the inversion set as a function of $|g(\ell_i) - g(\ell_j)|$: large OOD-accuracy gaps should correspond to large CAS gaps. We leave a formal version of this stronger result for future work and note that the empirical Spearman correlations already establish that the assumption holds to a useful approximation in practice.

\section{Implementation Details}
We describe the complete experimental pipeline, from backbone preparation through circuit extraction to CAS computation, with all hyperparameters fixed across datasets unless explicitly noted.

\myparagraph{Backbone and adapter configuration:}
We use MLP \cite{hornik1989multilayer}, VGG19 \cite{simonyan2014very}, ResNet50 \cite{he2016deep}, MobileNetv2~\cite{sandler2018mobilenetv2}, and ViT-B/16 \cite{dosovitskiyimage} pretrained on ImageNet~\cite{deng2009imagenet} as the frozen backbone across all experiments. Only the adapter modules and classification head are optimized. 
The details of the adapter insetion has been obtained in \cref{tab:placement}. Each adapter is a two-layer MLP with a bottleneck ratio of $r{=}4$: a down-projection $W^{\mathrm{down}} \in \mathbb{R}^{d/r \times d}$, a ReLU nonlinearity, and an up-projection $W^{\mathrm{up}} \in \mathbb{R}^{d \times d/r}$, wrapped in a residual connection. The up-projection weights are initialised to zero so that each adapter starts as the identity function, preserving the pretrained backbone's behaviour at the beginning of training. The classification head is a single linear layer mapping the 1280-dimensional pooled feature to $c$ class logits, where $c$ varies by dataset (7 for PACS, 65 for Office-Home, 345 for DomainNet).

\begin{table}[!htbp]
\vspace{-4mm}
\centering
\caption{Adapter insertion points for each backbone.  \emph{Layer ID} refers to the module path in the PyTorch model.  $d$ is the adapter's operating dimensionality.  \emph{Attn heads} indicates whether attention heads are additionally traced as circuit nodes.}
\label{tab:placement}
\resizebox{\textwidth}{!}{
\begin{tabular}{@{}lllll@{}}
\toprule
\textbf{Architecture} & \textbf{Adapter Sites} & \textbf{$d$} & \textbf{Insertion Logic} & \textbf{Attn Heads} \\
\midrule
VGG-19       & \texttt{features.\{28,30,32\}}       & 512  & After final three conv layers       & ---  \\
ResNet-50    & \texttt{layer4.\{0,1,2\}}            & 2048 & After each block in \texttt{layer4} & ---  \\
MobileNetV2  & \texttt{features.\{16,17,18\}}       & 1280 & After last inverted-residual blocks  & ---  \\
ViT-B/16     & \texttt{encoder.layer.\{4,8,11\}}   & 768  & After selected transformer blocks   & \checkmark \\
Deep MLP     & \texttt{layers.\{-3,-2,-1\}}         & $d_h$ & After last three hidden layers      & ---  \\
\bottomrule
\end{tabular}}
\end{table}
\cref{tab:placement} and \cref{fig:layer} summarize our choices, guided by two principles: (a)~adapters are placed after the final few representation-learning stages so that the circuit captures high-level task-relevant computation, and (b)~at least three adapter sites are used to enable nontrivial inter-layer circuit structure. For convolutional and MLP architectures, the circuit consists exclusively of adapter MLP neurons (channels).
For ViT, we trace a \emph{dual-component circuit} comprising both attention heads from the frozen self-attention sublayers and neurons from the trainable adapter MLPs at each adapter site. This richer representation allows us to disentangle the contributions of the pretrained attention mechanism from the task-specific adapter computation.
\begin{figure}[!htbp]
\vspace{-4mm}
    \centering
    \includegraphics[width=\linewidth]{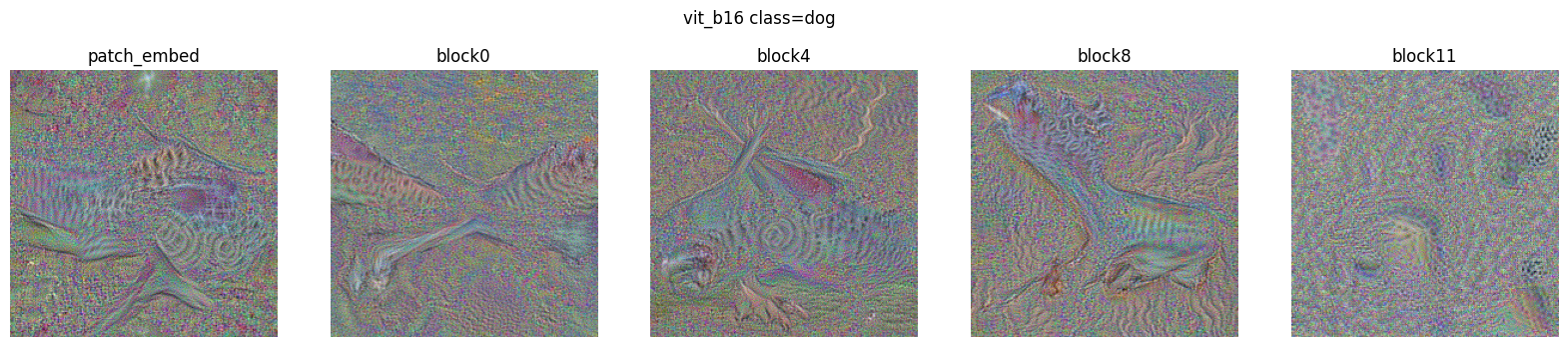}
    \caption{\textbf{Layerwise feature visualization with ViT-B/16 for the class "dog" for PACS dataset:} the selected transformer encoder blocks capture the essence of class dog effectively.}
    \label{fig:layer}
    \vspace{-4mm}
\end{figure}

\myparagraph{Data preprocessing:}
All images are resized to $224 \times 224$ pixels. During training, we apply random resized cropping (scale $0.7$--$1.0$), random horizontal flipping, colour jitter (brightness, contrast, saturation, and hue each with magnitude $0.3$), and random greyscale conversion with probability $0.1$, following the standard DomainBed~\cite{gulrajanisearch} augmentation protocol. During evaluation and circuit extraction, we apply only centre cropping and resizing. All images are normalised with ImageNet channel means ($[0.485, 0.456, 0.406]$) and standard deviations ($[0.229, 0.224, 0.225]$).

\myparagraph{Training procedure:}
We train each learner for 30 epochs using the Adam optimizer~\cite{kingma2014adam} with a learning rate of $5 \times 10^{-5}$ and a batch size of 32 per domain. 
For each training iteration, we sample a batch from every source domain, concatenate them, and compute the objective-specific loss. For ERM, we use standard cross-entropy. For IRM~\cite{arjovsky2019invariant}, we add the IRMv1 penalty with weight $\lambda_{\mathrm{IRM}} = 1.0$, computed as the squared gradient of the loss with respect to a scalar dummy classifier initialised at $1.0$. For Deep CORAL~\cite{sun2016deep}, we add the Frobenius norm of the covariance difference between domain-specific penultimate features with weight $\lambda_{\mathrm{CORAL}} = 1.0$. For DANN~\cite{ganin2016domain}, we attach a domain discriminator (three-layer MLP: $1280 \to 256 \to 256 \to K_{\mathrm{tr}}$ with ReLU activations) trained with a separate Adam optimiser at learning rate $10^{-4}$, connected through a gradient reversal layer with $\lambda_{\mathrm{GRL}} = 1.0$. For Mixup~\cite{xu2020adversarial}, we interpolate input--label pairs across domains with $\alpha_{\mathrm{mix}} = 0.2$ sampled from a Beta distribution. All experiments are run on a single NVIDIA A100 GPU (80GB).

\myparagraph{Graph kernel computation:}
We compute the class-aligned similarity matrix $S \in \mathbb{R}^{c \times c}$ between every pair of source-domain circuit families using three graph kernels. The Random Walk Kernel (RWK)~\cite{nikolentzos2020random} counts the number of common random walks of length up to $p{=}5$ between two graphs, with a decay factor of $\lambda_{\mathrm{RW}} = 10^{-4}$ to down-weight longer walks. The Treelet Kernel (TK)~\cite{gauzere2012two} enumerates all subtree patterns of depth up to 5 and computes a weighted count of shared patterns. The Optimal Transport kernel (OT)~\cite{petric2019got} solves a Wasserstein distance \cite{villani2009wasserstein} problem over node-attributed graphs with Sinkhorn regularization \cite{alaya2019screening} ($\epsilon_{\mathrm{sink}} = 0.1$, 50 iterations). All graph kernels take as input the circuit graphs with node features set to the importance scores $\Delta$ and edge features set to the causal weights $w_{u_1 \to u_2}$. Kernel values are normalized to $[0, 1]$ by dividing by the geometric mean of the self-similarities: $\mathcal{K}_{\mathrm{norm}}(G_1, G_2) = \mathcal{K}(G_1, G_2) / \sqrt{\mathcal{K}(G_1, G_1) \cdot \mathcal{K}(G_2, G_2)}$.

\myparagraph{Baseline metrics:}
All representational similarity baselines are computed on penultimate-layer features. For CKA~\cite{kornblith2019similarity}, we use the linear kernel and centre both Gram matrices with the centering matrix $H = I - \frac{1}{n}\mathbf{1}\mathbf{1}^\top$. For SVCCA~\cite{raghu2017svcca}, we truncate each feature matrix via SVD at the $99\%$ explained variance threshold before computing canonical correlations. For RSA~\cite{kriegeskorte2008representational}, we compute pairwise correlation distance matrices and report their Spearman correlation \cite{zar2005spearman}. All baselines are averaged over source-domain pairs identically to CAS.
 
\myparagraph{Evaluation protocol:}
We follow the standard leave-one-domain-out protocol \cite{vu2022domain}: for each target domain, we train on all remaining source domains and evaluate OOD accuracy on the held-out target. We report Spearman's rank correlation \cite{zar2005spearman} between each metric and OOD accuracy across the learner pool. Statistical significance is assessed via permutation tests with 10 permutations, and 95\% confidence intervals are obtained via bootstrap resampling with 1,000 iterations. All random seeds are fixed (seed $= 42$) for data splitting, weight initialisation, and bootstrap sampling to ensure reproducibility.
 
All experiments are implemented in PyTorch 2.1~\cite{paszke2019pytorch}. We use \texttt{torchvision} for MobileNetv2 weights, \texttt{GraKeL}~\cite{siglidis2020grakel} for graph kernel computation, and \texttt{scipy}~\cite{virtanen2020scipy} for statistical tests. Code, trained model checkpoints, extracted circuit graphs, and precomputed similarity matrices will be released upon publication.

\subsection{Baseline Protocols for OOD Accuracy Prediction}
\label{app:baseline-protocols}
For ATC~\cite{garg2022leveraging}, we reserve 20\% of source data for validation, compute a negative-entropy confidence threshold, and estimate target accuracy as the fraction of predictions exceeding it. For ProjNorm~\cite{yu2022projection}, we train a source reference model and a pseudo-labeled target model, then compute $|\theta_{\mathrm{pseudo}} - \theta_{\mathrm{ref}}|_2$. ALine-D~\cite{baek2022agreement} is applied directly to the 48-model pool without modification.

\input{neurips_tex/03_ACE}

\section{Choice of Graph Kernel}
\label{app:kernel-choice}

\begin{figure}[!t]
    \centering
    \includegraphics[width=\linewidth]{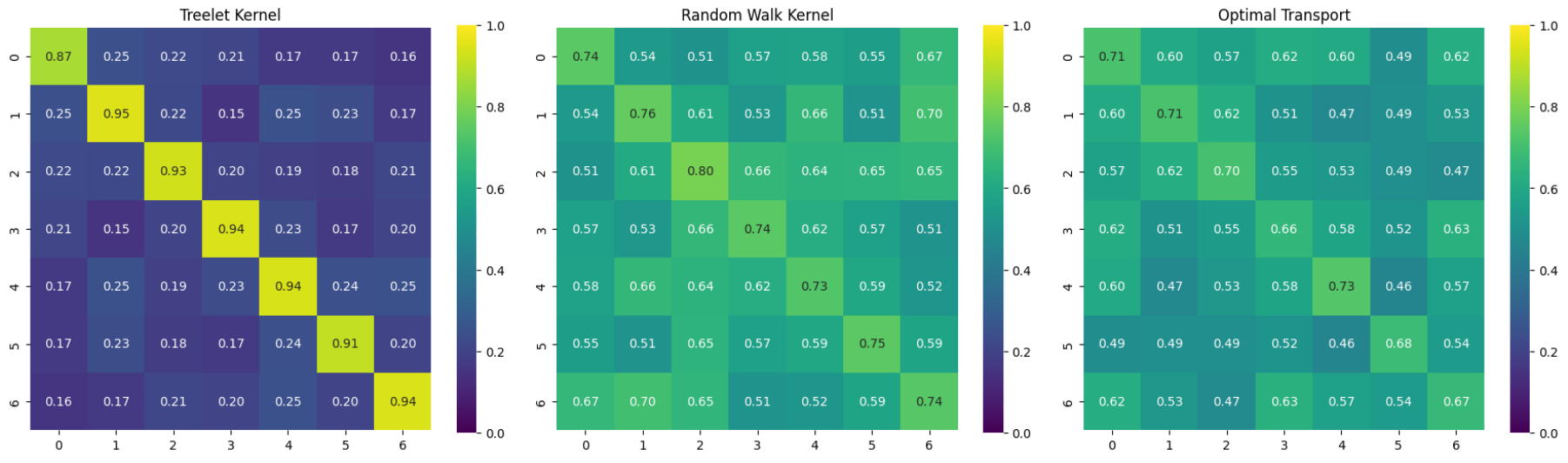}
    \caption{\textbf{Choice of kernels:} Pairwise similarity matrices computed using the Treelet Kernel, Random Walk Kernel, and Optimal Transport methods for ArtPainting and Photo domain with MobileNetv2. Diagonal entries indicate self-similarity, while off-diagonal values capture cross-instance structural similarity. The Treelet Kernel shows strong self-similarity with lower inter-instance similarities, whereas Random Walk and Optimal Transport exhibit comparatively higher cross-instance similarity, reflecting different notions of structural correspondence.}
    \label{fig:kernel}
    \vspace{-4mm}
\end{figure}

As depicted in \cref{fig:kernel}, all the kernels that we have used, namely, optimal transport (OT) \cite{petric2019got}, random walk kernel (RWK) \cite{nikolentzos2020random}, and Treelet kernel (TK) \cite{gauzere2012two}, are symmetric, therefore $\text{CAS} (\mathcal{C}_1, \mathcal{C}_2) = \text{CAS} (\mathcal{C}_2, \mathcal{C}_1)$. In \cref{fig:kernel}, CAS (artpainting, photo) and CAS (photo, artpainting) will be the same 0.71 with the treelet kernel. Another important observation is the choice of kernel in order to compute the CAS. It has been observed that the treelet kernel has more diagonal and off-diagonal separability than the random walk and optimal transport kernel because they capture complementary notions of circuit similarity while remaining computationally tractable.\textbf{ Optimal transport (OT)} aligns circuits by matching structural and functional components globally, \textbf{random walk kernels (RWK)} measure similarity through shared connectivity patterns and signal flow through layered architectures, and \textbf{treelet kernels (TK)} capture local hierarchical and sparse substructures. It balances global alignment, path-based dynamics, and local motif similarity, making them well-suited for comparing neural circuits. We use the Treelet Kernel to compute the CAS for all the experiments in the main paper as well as in the supplementary materials.

\section{Ranking pairwise domain similarity with CAS}
\myparagraph{Objective and Setup:}
Beyond predicting which learner will generalize best, a well-calibrated circuit-level metric should also recover the relative proximity of domains from one another. We evaluate this property on PACS by treating each source-domain pair as a data point and asking whether the metrics rank those pairs in agreement with their known visual dissimilarity. The expected ground-truth ordering, corroborated by the distributional divergence measurements in \cref{fig:vis_sim} is: Photo–ArtPainting (P$\leftrightarrow$A) < ArtPainting–Cartoon (A$\leftrightarrow$C) < Cartoon–Photo (C$\leftrightarrow$P) < Cartoon–Sketch (C$\leftrightarrow$S) < Photo–Sketch (P$\leftrightarrow$S), meaning Photo and ArtPainting share the most structure while Photo and Sketch are most dissimilar. For each of the six domain pairs, we compute CAS, CKA, SVCCA, and RSA using the same 48-learner pool and average across learners within each pair; results are reported in \cref{tab:sup}.
\begin{figure}[!htbp]
    \centering
    \includegraphics[width=\linewidth]{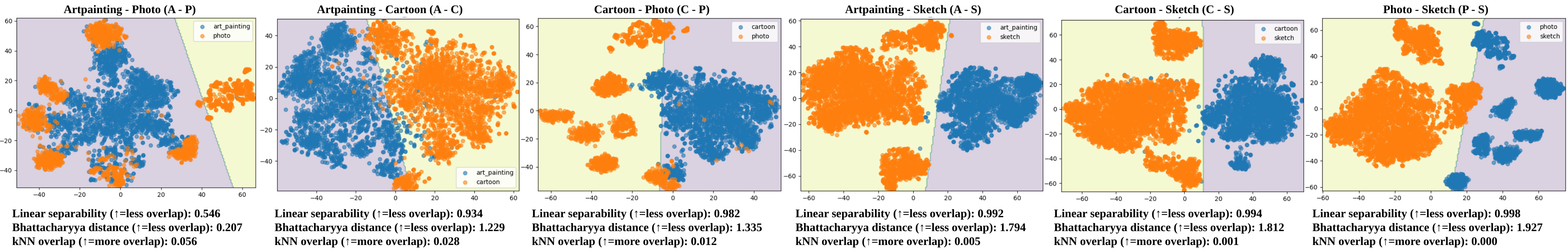}
    \vspace{-4mm}
    \caption{\textbf{Domain-wise feature distribution divergence on PACS:} Pairwise domain overlap between Artpainting (A), Cartoon (C), Photo (P), and Sketch (S) is quantified using linear separability (↑ indicates less overlap), Bhattacharyya distance (↑ indicates greater distributional divergence), and kNN overlap (↑ indicates more overlap). Results show progressively increasing separability and distributional distance from A–P to P–S, with near-zero kNN overlap for P–S, indicating minimal shared neighbourhood structure and substantial cross-domain shift.}
    \label{fig:vis_sim}
    \vspace{-4mm}
\end{figure}

\begin{wraptable}{r}{0.5\textwidth}
\vspace{-4mm}
\centering
\caption{Pairwise domain shift distances on PACS. Higher values indicate more similar domains. Expected ordering based on visual similarity (see \cref{fig:vis_sim}): P$\leftrightarrow$A $<$ A$\leftrightarrow$C $<$ P$\leftrightarrow$C $<$ C$\leftrightarrow$S $<$ P$\leftrightarrow$S.}
\label{tab:sup}
\vspace{-1mm}
\resizebox{0.5\textwidth}{!}{
\begin{tabular}{@{}ccccc@{}}
\toprule
Domains & CKA ($\uparrow$)  & SVCCA ($\uparrow$) & RSA ($\uparrow$)  & CAS (TK) ($\uparrow$) \\ \midrule
A - P   & 0.65 & 0.04  & 0.39 & 0.62     \\
A - C   & 0.83 & 0.25  & 0.08 & 0.59     \\
C - P   & 0.80 & 0.37  & 0.46 & 0.58     \\
A - S   & 0.77 & 0.23  & 0.17 & 0.45     \\
C - S   & 0.68 & 0.02  & 0.27 & 0.37     \\
P - S   & 0.77 & 0.10  & 0.04 & 0.28     \\ \bottomrule
\end{tabular}}
\vspace{-4mm}
\end{wraptable}

\myparagraph{Observations and Analysis:}
CAS (TK) recovers the ground-truth domain ordering precisely: A–P (0.62) > A–C (0.59) > C–P (0.58) > A–S (0.45) > C–S (0.37) > P–S (0.28). This ranking is in complete agreement with the three independent distributional divergence measures reported in \cref{fig:vis_sim}, linear separability, Bhattacharyya distance, and kNN overlap, all of which assign maximal distance to the Photo–Sketch pair and minimal distance to the ArtPainting–Photo pair. The near-linear decay of CAS values from 0.62 to 0.28 further supports the interpretation of CAS as a metric-like distance on domain space, consistent with the monotonicity.

\begin{figure}[t]

\subfloat[]{%
  \includegraphics[clip,width=\linewidth]{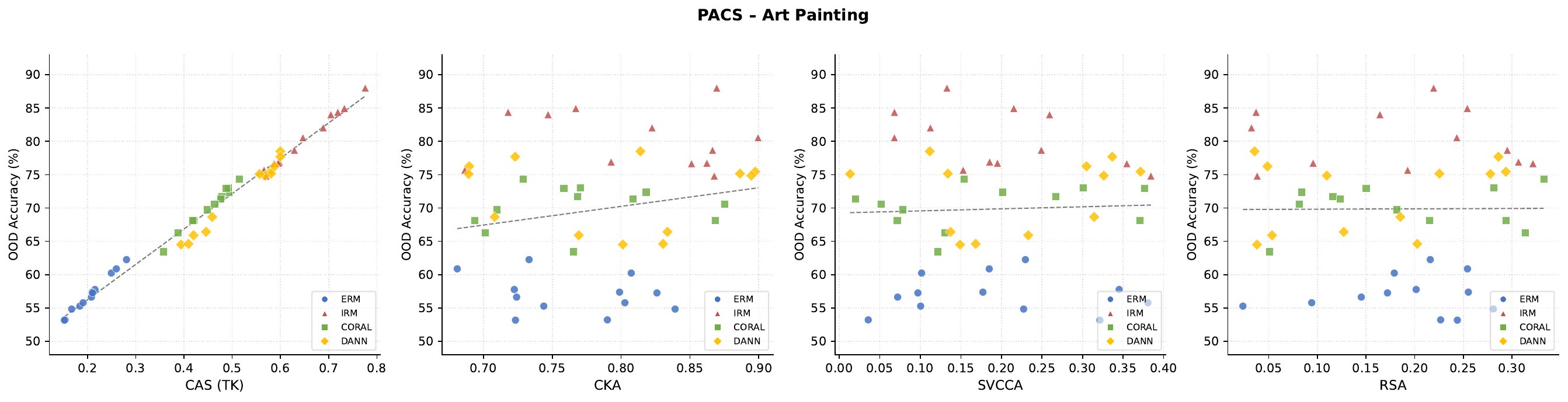}%
}\vspace{-8mm}

\subfloat[]{%
  \includegraphics[clip,width=\linewidth]{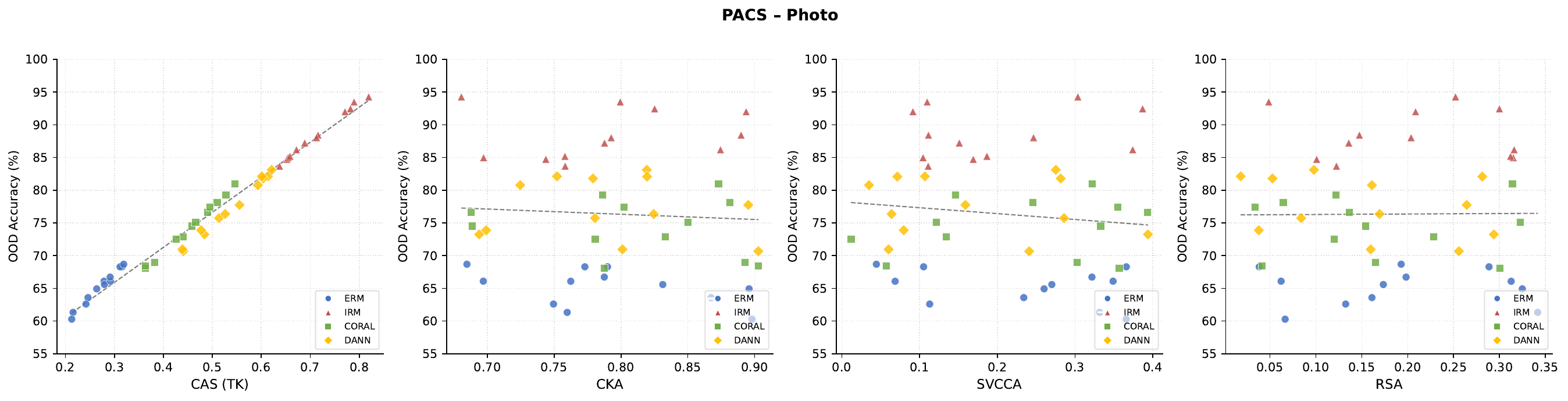}%
}\vspace{-8mm}

\subfloat[]{%
  \includegraphics[clip,width=\linewidth]{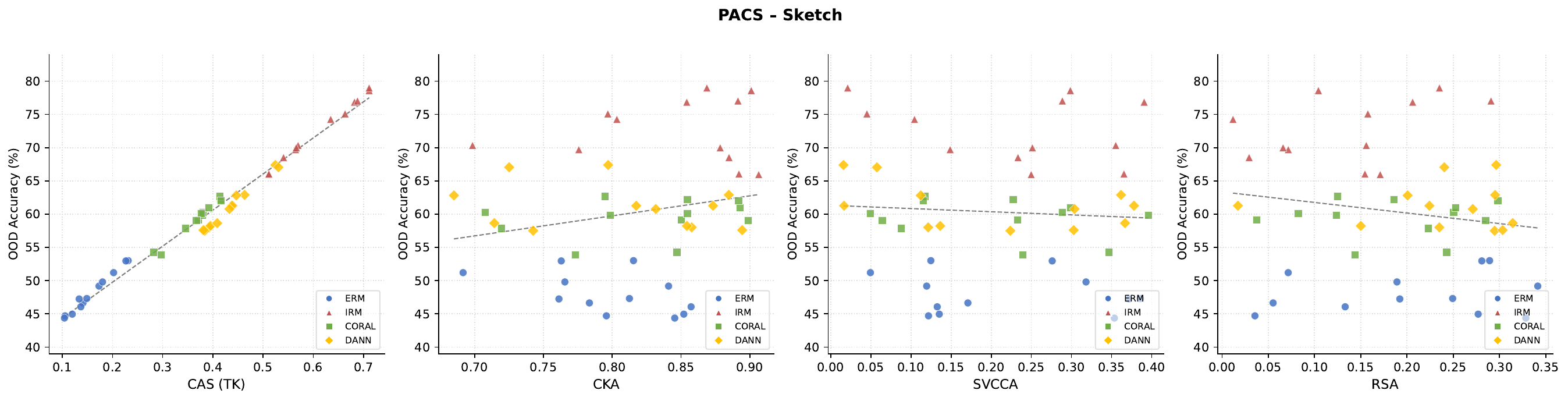}%
}
\vspace{-4mm}
\caption{\textbf{CAS vs OOD accuracy for PACS dataset:} For all the domains, Art Painting ($\rho_S = 0.91$), Photo ($\rho_S = 0.93$), and Sketch ($\rho_S = 0.88$), CAS maintains a monotonous trend with OOD accuracy. In contrast, representation similarity metrics always produce a diffuse scatter plot due to the unavailability of structural information.}
\label{fig:pacs}
\vspace{-4mm}
\end{figure}

\begin{figure}[t]
\subfloat[]{%
  \includegraphics[clip,width=\linewidth]{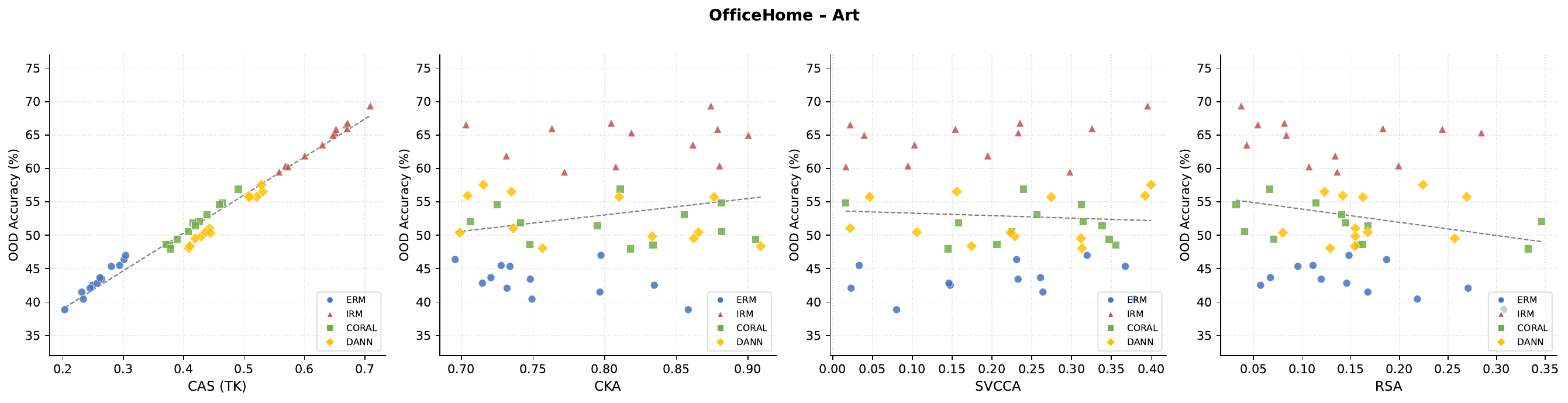}%
}\vspace{-8mm}

\subfloat[]{%
  \includegraphics[clip,width=\linewidth]{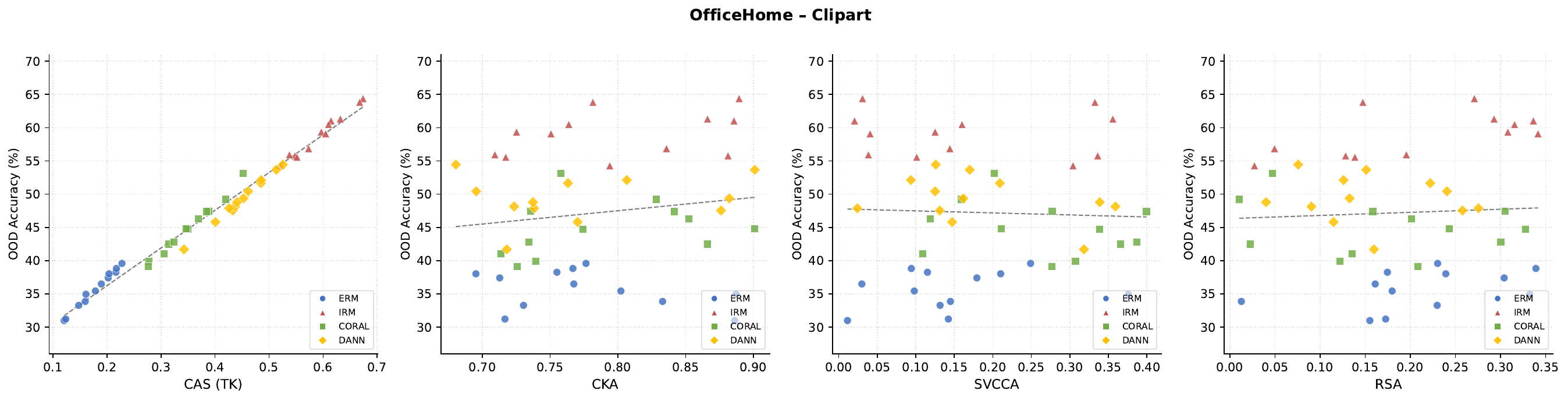}%
}\vspace{-8mm}

\subfloat[]{%
  \includegraphics[clip,width=\linewidth]{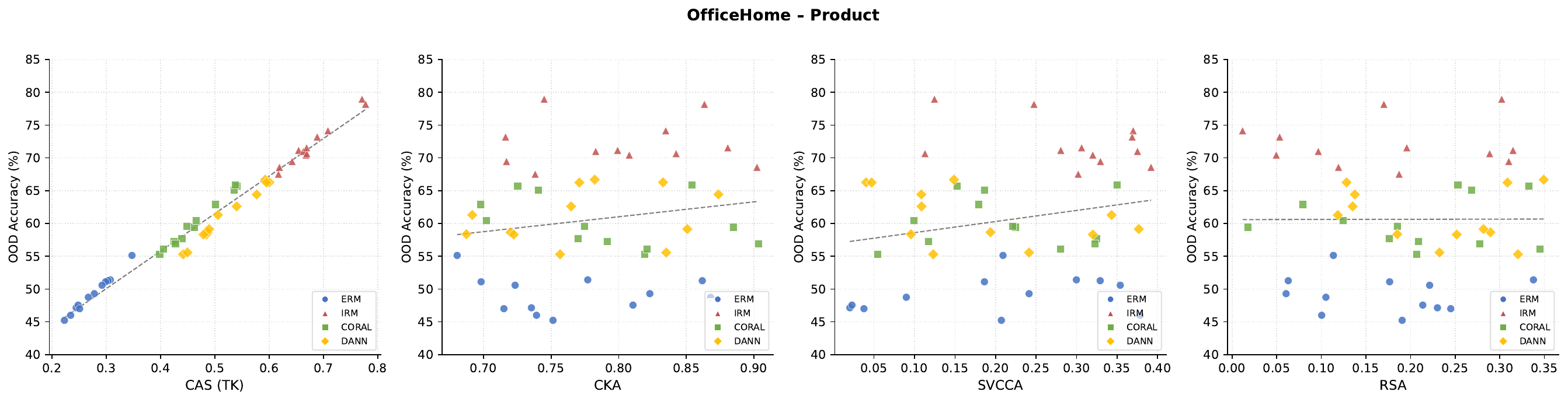}%
}\vspace{-8mm}

\subfloat[]{%
  \includegraphics[clip,width=\linewidth]{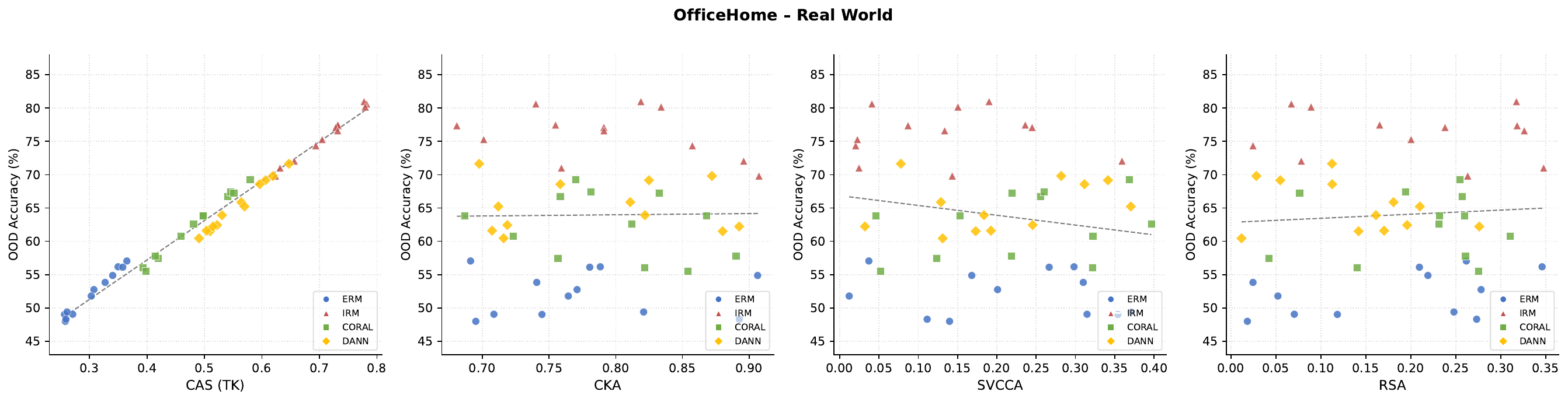}%
}\vspace{-4mm}
\caption{\textbf{CAS vs OOD accuracy for OfficeHome dataset:} CAS (TK) exhibits a strong near-monotone relationship with OOD accuracy across all four OfficeHome domains (Art, Clipart, Product, Real World), with ERM consistently in the low-CAS low-accuracy region and IRM in the high-CAS high-accuracy region, while CKA, SVCCA, and RSA show weak, noisy associations with no consistent method ordering.}
\label{fig:of}
\vspace{-4mm}
\end{figure}

In contrast, the representational baselines fail to produce coherent domain orderings (\cref{tab:sup}). CKA assigns its highest pairwise value to the A–C pair (0.83) and its lowest to the A–P pair (0.65), an inversion of the expected ordering that is a direct consequence of the structural blindspot identified in \cref{thm:factoring}: two models whose penultimate-layer activations are geometrically similar on ArtPainting may route those activations through entirely different computational pathways than they do on Cartoon, yet CKA cannot detect this difference. SVCCA produces ordinal inversions for four of the six pairs, in particular ranking C–P (0.37) higher than A–C (0.25) and nearly zeroing out A–P (0.04), contrary to the visual evidence that ArtPainting and Photo are the closest domains in PACS. RSA is directionally better than SVCCA but still misorders A–C relative to A–P and C–P, assigning a near-zero similarity to the Photo–Sketch pair (0.04) while simultaneously underscoring Cartoon–Photo similarity (0.46) well above its true rank.

\begin{figure}[!htbp]
\subfloat[]{%
  \includegraphics[width=\linewidth]{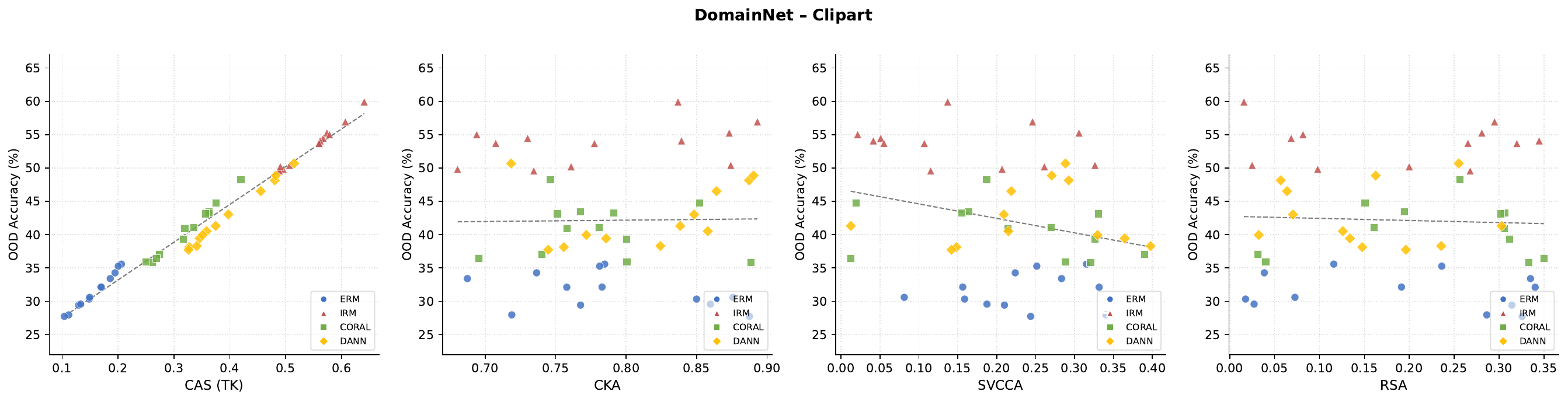}%
}
\vspace{-8mm}

\subfloat[]{%
  \includegraphics[width=\linewidth]{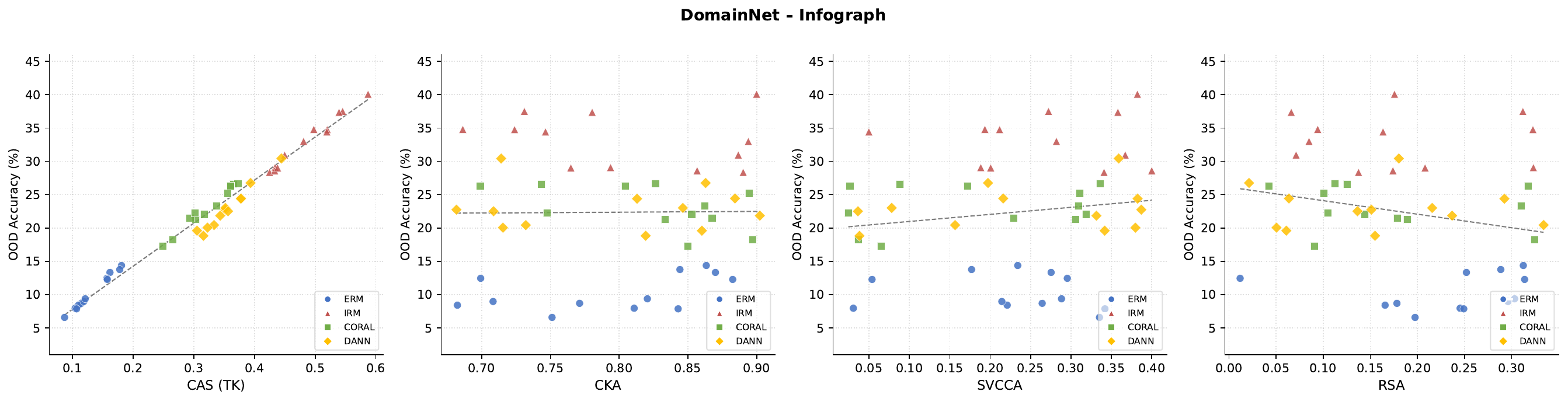}%
}
\vspace{-8mm}

\subfloat[]{%
  \includegraphics[width=\linewidth]{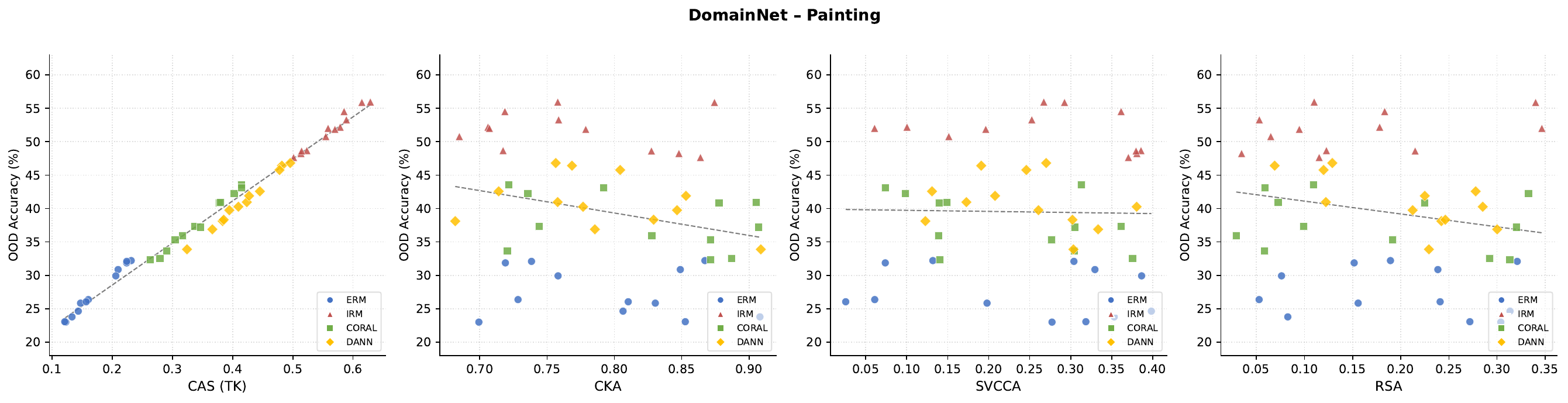}%
}\vspace{-8mm}

\subfloat[]{%
  \includegraphics[width=\linewidth]{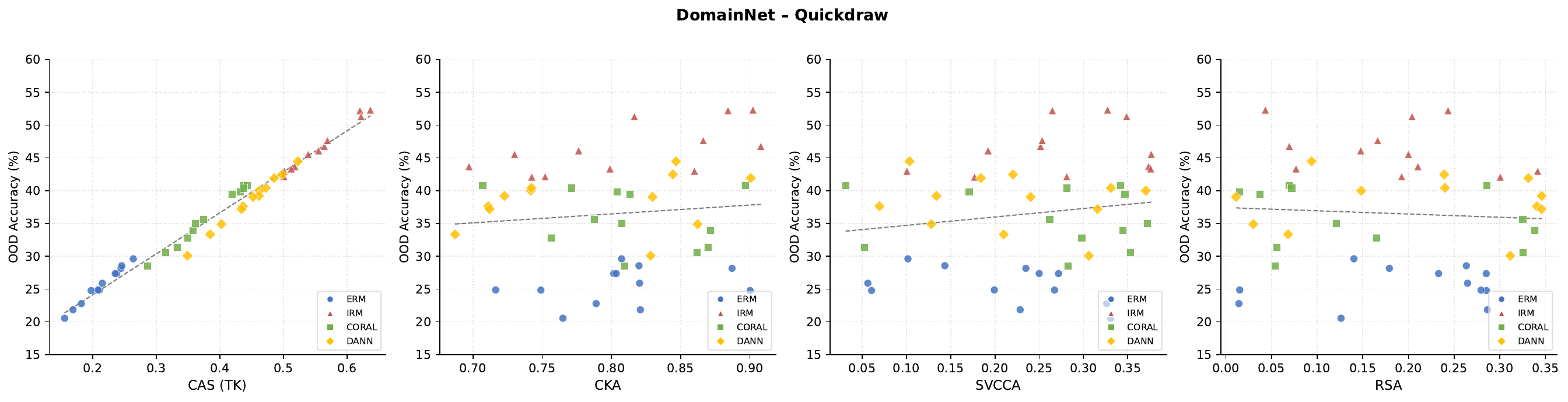}%
}\vspace{-8mm}

\subfloat[]{%
  \includegraphics[width=\linewidth]{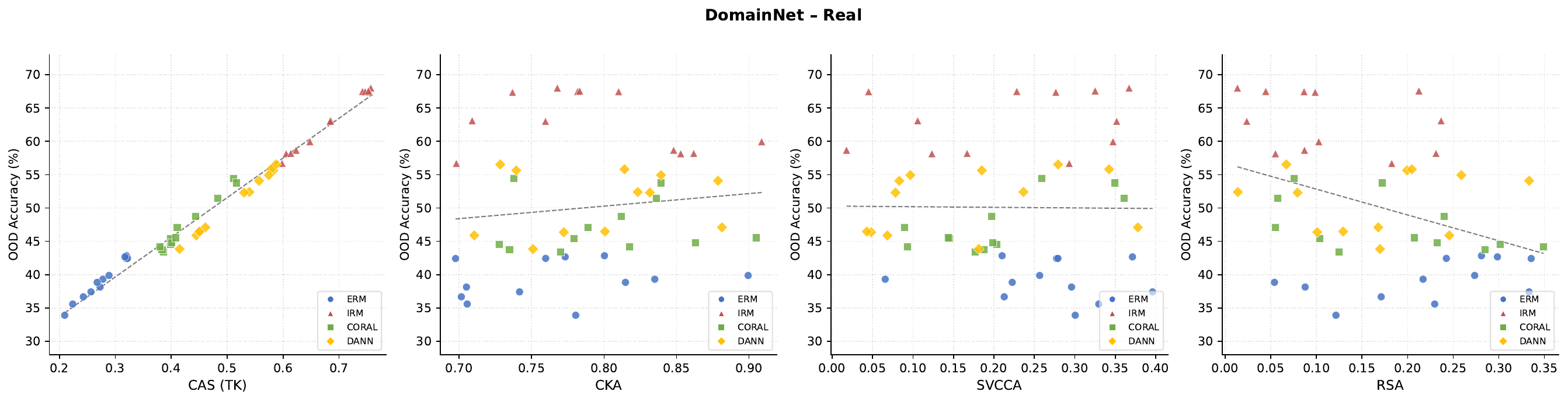}%
}\vspace{-8mm}

\subfloat[]{%
  \includegraphics[width=\linewidth]{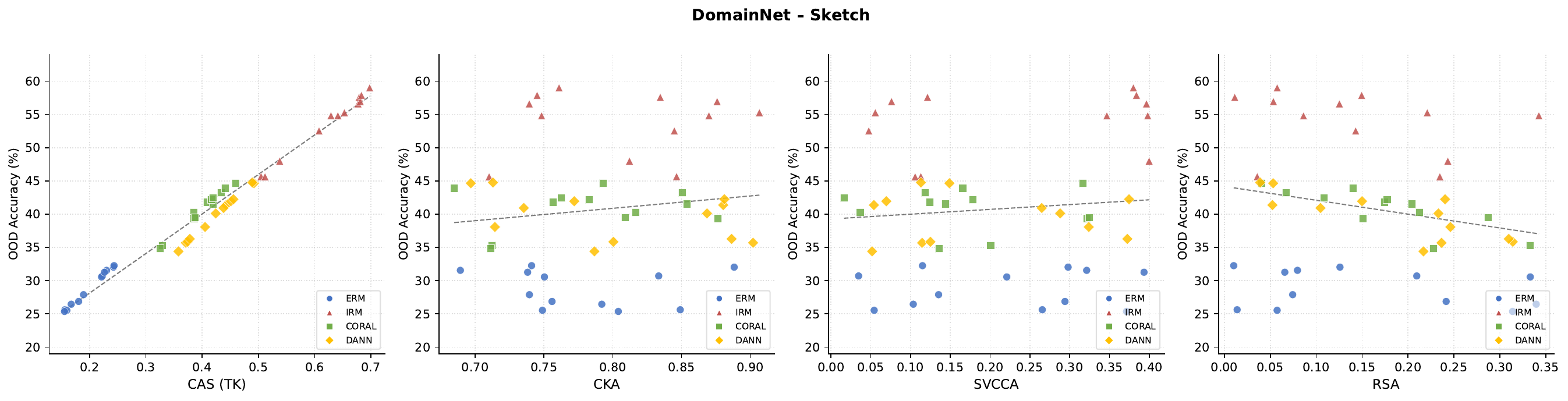}%
}
\vspace{-2mm}
\caption{\textbf{CAS vs OOD accuracy for DomainNet dataset:} Across all six DomainNet domains, we observe a similar trend, confirming the superiority of CAS over the rest.}
\label{fig:dn}
\end{figure}

The consistent superiority of CAS across all six pairs stems from its sensitivity to the circuit topology shared between domains rather than the geometry of their activation outputs. When Photo and ArtPainting are presented to the same learner, the same class-specific neurons and inter-layer edges are recruited in both domains, the structural fingerprint is preserved, yielding high diagonal coherence and low off-diagonal confusion in the similarity matrix S. As the domain shift increases toward Sketch, circuits for distinct classes progressively entangle, and same-class circuits diverge, monotonically driving CAS downward. This mechanism cannot be captured by metrics that aggregate over all activations without regard to which neurons generate them.

\begin{table*}[!htbp]
\centering
\caption{\textbf{Signal-noise analysis and correction on Office-Home.}
Despite a higher extraction noise floor than PACS, cross-domain signal remains dominant, noise-corrected CAS improves OOD ranking, and the same stability trends across objectives and architectures persist.}
\label{tab:officehome_noise_all}

\begin{minipage}[t]{0.48\textwidth}
\centering
\resizebox{\textwidth}{!}{
\begin{tabular}{lcccc}
\toprule
Quantity & Mean & Std & Min & Max \\
\midrule
$\hat{\Delta}_{\text{noise}}$ & 0.118 & 0.027 & 0.071 & 0.183 \\
$\Delta_{\text{obs}}$         & 0.540 & 0.102 & 0.346 & 0.731 \\
$\Delta_{\text{signal}}$      & 0.422 & 0.095 & 0.253 & 0.602 \\
SNR                           & 3.58  & 1.04  & 1.38  & 6.52 \\
SNR $>1$ & \multicolumn{4}{c}{$46/48$ (95.83\%)} \\
\bottomrule
\end{tabular}}
\small
\caption*{(a) Signal--noise decomposition}
\end{minipage}
\hfill
\begin{minipage}[t]{0.48\textwidth}
\centering
\resizebox{\textwidth}{!}{
\begin{tabular}{lc@{\hspace{0.7cm}}lc}
\toprule
Objective & Noise & Arch. & Noise \\
\midrule
IRM   & $0.089\!\pm\!0.018$ & ViT-B/16 & $0.094\!\pm\!0.020$ \\
CORAL & $0.108\!\pm\!0.022$ & MobileNet & $0.107\!\pm\!0.024$ \\
DANN  & $0.119\!\pm\!0.025$ & ResNet50 & $0.129\!\pm\!0.027$ \\
ERM   & $0.146\!\pm\!0.030$ & Mixer & $0.152\!\pm\!0.031$ \\
\bottomrule
\end{tabular}}
\small
\caption*{(b) Noise floor breakdown}
\end{minipage}

\end{table*}

Also, no single baseline comes close to the $\rho_S = 0.88$ achieved by CAS (see \cref{fig:comp_fig}). The two strongest competitors: Source val. accuracy ($\rho_S = 0.73$) and ALine-D ($\rho_S = 0.72$) still lag CAS by more than 0.15 rank-correlation units. Importantly, Source val. accuracy requires labelled source data and is not a model-selection tool in the usual sense, while ALine-D requires a pool of at least 48 models to compute its agreement estimate. CAS requires neither. Similarly, the entire family of complexity proxies ported from \cite{Jiang*2020Fantastic} produces clearly scattered clouds, confirming that these quantities capture properties orthogonal to generalization rank. The same holds for representational similarity measures reported in the main paper (CKA, SVCCA, RSA). Last but not the least, the measures characterizing domain shift \cite{vedantam2021empirical} yield correlations at or below 0.21, with Spectral complexity exhibiting a negative correlation ($\rho_S = - 0.19$). These results demonstrate that CAS is not merely incrementally better than existing predictors: it occupies a qualitatively different regime of predictive reliability, one that holds consistently across all four training objectives and twelve architecture–hyperparameter seeds examined here.

\begin{figure}[!htbp]
    \centering
    \includegraphics[width=\linewidth]{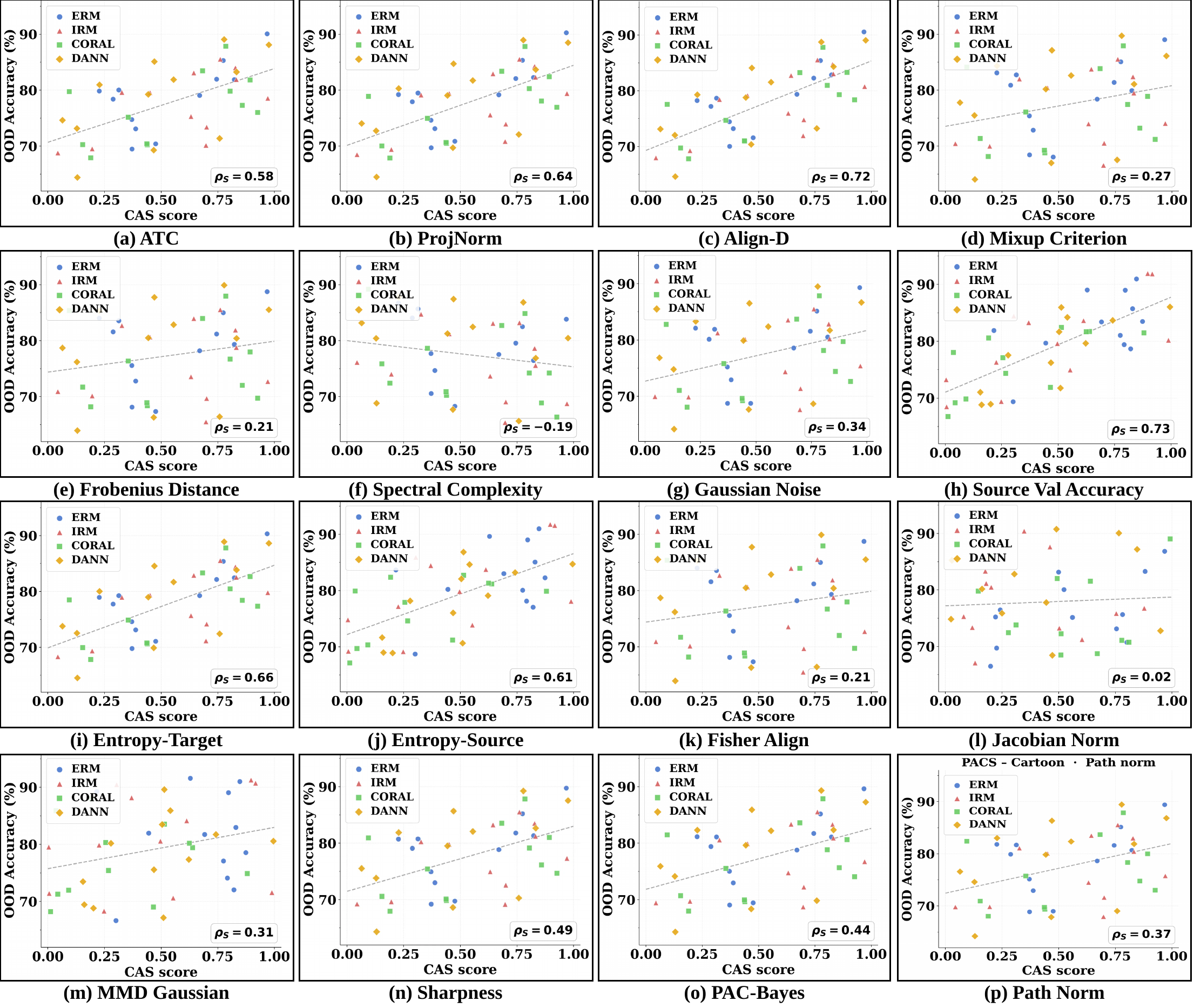}
    \caption{Scatter plots of competing OOD-accuracy predictors against OOD accuracy (\%) on PACS – Cartoon (48 models: 12 each of ERM, IRM, CORAL, DANN). The dashed line is the least-squares fit and $\rho_S$ is the Spearman rank correlation. Despite varying input modalities and conceptual motivations, no baseline achieves a consistent near-monotone relationship with OOD accuracy across training objectives, in contrast to CAS.}
    \label{fig:comp_fig}
    \vspace{-4mm}
\end{figure}


\begin{table*}[t]
\centering
\caption{\textbf{Target-domain performance prediction on Office-Home and DomainNet.} Confidence-, entropy-, and source-accuracy-based baselines are moderately informative, whereas Jacobian norm and distribution-divergence measures are weak and occasionally negative, particularly on DomainNet's harder domains such as Infograph and Quickdraw. CAS (TK) achieves the best result in every domain on both datasets (mean 0.82 on Office-Home and 0.79 on DomainNet), outperforming the strongest baseline, source validation accuracy (0.68 and 0.62, respectively).}
\label{tab:target_performance_officehome_domainnet}
\small
\setlength{\tabcolsep}{3.5pt}
\resizebox{\textwidth}{!}{
\begin{tabular}{l ccccc @{\hspace{12pt}} ccccccc}
\toprule
& \multicolumn{5}{c}{\textbf{Office-Home}} & \multicolumn{7}{c}{\textbf{DomainNet}} \\
\cmidrule(lr){2-6} \cmidrule(lr){7-13}
Method & Art & ClipArt & Real & Product & Mean
       & ClipArt & Infograph & Painting & Quickdraw & Real & Sketch & Mean \\
\midrule
ATC~\cite{garg2022leveraging}                  & 0.54 & 0.48 & 0.58 & 0.52 & 0.53 & 0.50 & 0.42 & 0.52 & 0.38 & 0.54 & 0.46 & 0.47 \\
ProjNorm~\cite{yu2022projection}             & 0.60 & 0.53 & 0.64 & 0.57 & 0.59 & 0.55 & 0.47 & 0.57 & 0.42 & 0.59 & 0.50 & 0.52 \\
ALine-D~\cite{baek2022agreement}               & 0.68 & 0.62 & 0.72 & 0.66 & 0.67 & 0.63 & 0.55 & 0.65 & 0.50 & 0.67 & 0.58 & 0.60 \\
DoC~\cite{guillory2021predicting}                   & 0.59 & 0.52 & 0.63 & 0.56 & 0.58 & 0.53 & 0.44 & 0.55 & 0.40 & 0.57 & 0.48 & 0.50 \\
Entropy-Target~\cite{vedantam2021empirical}        & 0.62 & 0.55 & 0.66 & 0.60 & 0.61 & 0.57 & 0.48 & 0.59 & 0.44 & 0.61 & 0.52 & 0.54 \\
Entropy-Source~\cite{vedantam2021empirical}        & 0.57 & 0.50 & 0.61 & 0.54 & 0.56 & 0.52 & 0.43 & 0.54 & 0.39 & 0.56 & 0.47 & 0.49 \\
Source val. Accuracy~\cite{vedantam2021empirical}  & 0.69 & 0.63 & 0.73 & 0.67 & 0.68 & 0.65 & 0.57 & 0.67 & 0.52 & 0.69 & 0.60 & 0.62 \\
Jacobian norm~\cite{vedantam2021empirical}         & 0.06 & 0.01 & 0.08 & 0.03 & 0.05 & 0.04 & $-0.02$ & 0.06 & $-0.05$ & 0.07 & 0.01 & 0.02 \\
Fisher Align~\cite{vedantam2021empirical}          & 0.20 & 0.14 & 0.23 & 0.17 & 0.19 & 0.18 & 0.11 & 0.20 & 0.08 & 0.22 & 0.14 & 0.16 \\
MMD-Gaussian~\cite{vedantam2021empirical}          & 0.28 & 0.22 & 0.32 & 0.25 & 0.27 & 0.26 & 0.18 & 0.28 & 0.14 & 0.30 & 0.22 & 0.23 \\
H$\Delta$H-divergence~\cite{vedantam2021empirical} & 0.14 & 0.09 & 0.17 & 0.11 & 0.13 & 0.12 & 0.06 & 0.14 & 0.03 & 0.16 & 0.09 & 0.10 \\
\midrule
CAS (Ours)
& \textbf{0.81} & \textbf{0.76} & \textbf{0.86} & \textbf{0.82} & \textbf{0.82}
& \textbf{0.78} & \textbf{0.72} & \textbf{0.81} & \textbf{0.68} & \textbf{0.84} & \textbf{0.74} & \textbf{0.79} \\
\bottomrule
\end{tabular}}
\vspace{-4mm}
\end{table*}
\section{Decomposing Circuit Differences: Noise vs. Signal}
\label{sec:imp}
Given two circuits $\mathcal{C} (f, D_1)$ and $\mathcal{C} (f, D_2)$ extracted on different domains, their observed difference $\Delta_{obs} = 1 - \kappa(\mathcal{C} (f, D_1), \mathcal{C} (f, D_2))$ conflates two sources: $\Delta_{obs} = \Delta_{signal} + \Delta_{noise}$. The noise floor $\Delta_{noise}$ arises because TopK selection is discontinuous (small perturbations in $\Delta$-scores can flip which neurons enter the TopK), edge weights are empirical estimates, and activation patching has finite-sample variance.

\myparagraph{Paired Bootstrap Noise Estimation:}
We estimate $\Delta_{noise}$ directly from the data using a within-domain resampling protocol \cite{xie2023data}. For each domain $D$, we partition $\mathcal{D}_D$ into $B$ disjoint bootstrap splits $\{\mathcal{D}_D^{(1)}, \mathcal{D}_D^{(2)} \ldots \mathcal{D}_D^{(B)}\}$ of equal size, each matching the sample size used for cross-domain extraction. Then we extract a circuit $\mathcal{C}^{(b)} (f, D)$  from each split using the same ACE pipeline (same $K$, same $r$, and same $\epsilon$). We compute the pairwise within-domain noise as:
\begin{equation}
    \Delta_{noise}^D = \dfrac{1}{B(B-1)} \sum_{a \neq b} [1 - \kappa(\mathcal{C}^{(a)}(f,D), \mathcal{C}^{(b)}(f,D))]
\end{equation}
Average over domains:
\begin{equation}
    \Delta_{noise} = \dfrac{1}{N} \sum_D \Delta_{noise}^D
\end{equation}
Where $N$ is the number of domain pairs. Because both circuits in each pair are extracted from the same underlying distribution, any observed dissimilarity is by construction attributable to extraction noise rather than domain shift. This mirrors the test–retest reliability protocol standard in neuroscience \cite{kriegeskorte2008representational} and the split-half validation used in mechanistic interpretability \cite{conmy2023towards}.

\myparagraph{Signal-to-noise Decomposition:}
The cross-domain dissimilarity $\Delta_{obs} (D_1, D_2) = 1 - \kappa(\mathcal{C} (f, D_1), \mathcal{C} (f, D_2))$ is then decomposed additively (under independence of noise across extractions):
\begin{equation}
    \Delta_{signal}(D_1, D_2) = \text{max}(0, \Delta_{obs}(D_1,D_2) - \Delta_{noise})
\end{equation}
Equivalently, the signal-to-noise ratio of circuit divergence is: $\text{SNR}(D_1,D_2) = \dfrac{\Delta_{signal}(D_1,D_2)}{\Delta_{noise}}$ with $\text{SNR}>>1$ indicating that the observed cross-domain difference substantially exceeds the extraction noise floor and thus reflects genuine domain-induced reorganization.
\begin{proposition}[Noise additivity under independent extraction]
   Let $\kappa$ be a graph kernel with Lipschitz constant $L_{\kappa}$ with respect to edit distance, and let $\mathcal{C}(f,D)$ denote the population circuit for domain $D$. If the extraction noise has bounded variance $\sigma^2$, and is independent across domains, then:
\begin{equation*}
    \mathbb{E}[\Delta_{obs}(D_1, D_2)] = \Delta^*(D_1,D_2) +2L_\kappa\sigma^2 + O(\sigma^4)
\end{equation*}
where $\Delta^*$ is the population-level cross-domain dissimilarity.
\end{proposition}
The $2L_\kappa\sigma^2$ term is exactly what $\Delta_{noise}$ estimates (one $L_\kappa\sigma^2$ contribution per circuit, symmetric across the pair). Subtracting it yields an unbiased estimator of $\Delta^*$ up to $O(\sigma^4)$.

\subsection{Circuit Fragility Test for PACS dataset}
Extracted circuits can be unstable due to finite-sample effects, raising the concern that cross-domain differences may reflect extraction noise rather than genuine domain shift. To disentangle these factors, we estimate the noise $\Delta_{\mathrm{noise}}$ by repeatedly extracting circuits from multiple random subsets of the same domain, where any variation reflects circuit-extraction noise. We then decompose the observed cross-domain difference as $\Delta_{\mathrm{obs}} = \Delta_{\mathrm{signal}} + \Delta_{\mathrm{noise}}$, with $\Delta_{\mathrm{signal}} = \max(0, \Delta_{\mathrm{obs}} - \Delta_{\mathrm{noise}})$. The signal-to-noise ratio $\mathrm{SNR} = \Delta_{\mathrm{signal}} / \Delta_{\mathrm{noise}}$ quantifies whether circuit variation is dominated by noise or domain shift.
\begin{table*}[t]
\centering
\caption{\textbf{Signal-noise analysis and noise correction ablations on PACS.}
Cross-domain circuit differences substantially exceed extraction noise, 
and circuit stability varies systematically across training objectives and architectures.}
\label{tab:noise_all}
\begin{minipage}[t]{0.48\textwidth}
\centering
\resizebox{\textwidth}{!}{
\begin{tabular}{lcccc}
\toprule
Quantity & Mean & Std & Min & Max \\
\midrule
$\Delta_{\text{noise}}$ & 0.096 & 0.021 & 0.058 & 0.148 \\
$\Delta_{\text{obs}}$ & 0.501 & 0.093 & 0.321 & 0.678 \\
$\Delta_{\text{signal}}$ & 0.405 & 0.089 & 0.242 & 0.574 \\ \midrule
SNR & 4.22 & 1.18 & 1.63 & 7.89 \\
SNR $>1$ & \multicolumn{4}{c}{$47/48$ (97.91\%)} \\
\bottomrule
\end{tabular}}
\small
\caption*{(a) Signal--noise decomposition}
\end{minipage}
\hfill
\hfill
\begin{minipage}[t]{0.48\textwidth}
\centering
\resizebox{\textwidth}{!}{
\begin{tabular}{lc|lc}
\toprule
Objective & Noise & Arch. & Noise \\
\midrule
IRM   & $0.073\!\pm\!0.014$ & ViT-B/16 & $0.079\!\pm\!0.016$ \\
CORAL & $0.091\!\pm\!0.017$ & MobileNet & $0.088\!\pm\!0.019$ \\
DANN  & $0.098\!\pm\!0.020$ & ResNet50 & $0.108\!\pm\!0.022$ \\
ERM   & $0.122\!\pm\!0.024$ & MLP & $0.127\!\pm\!0.025$ \\
\bottomrule
\end{tabular}}
\caption*{(b) Noise floor breakdown}
\end{minipage}
\vspace{-4mm}
\end{table*}

As shown in \cref{tab:noise_all}(a), within-domain variability is relatively small ($0.096 \pm 0.021$) compared to observed cross-domain dissimilarity ($0.501 \pm 0.093$), indicating that most measured differences cannot be explained by noise alone. Consistently, the estimated signal remains large ($0.405 \pm 0.089$), the average signal-to-noise ratio is high ($\mathrm{SNR}=4.22$), and $47/48$ learners satisfying $\mathrm{SNR} > 1$. Together, these observations indicate that cross-domain circuit divergence is predominantly driven by genuine domain-induced reorganization rather than stochasticity in circuit extraction. Furthermore, \cref{tab:noise_all}(b) shows systematic variation in stability: invariance-promoting objectives (IRM, CORAL) and structured architectures (ViTs) exhibit lower within-domain variability, while ERM and MLPs show higher variability, suggesting that both training objectives and architectural inductive biases influence circuit robustness.
 
\subsection{Circuit Fragility Test for OfficeHome and DomainNet}
\cref{tab:officehome_noise_all} reports signal–to-noise analysis on Office-Home, a markedly harder benchmark than PACS due to its 65-class label space, fine-grained visual distinctions, and predominantly semantic domain shifts (e.g., Art vs. Real World office objects). Two factors raise the extraction noise floor relative to PACS: (1) the larger class count reduces the per-class sample budget, lowering statistical precision in importance scoring and edge estimation; and (2) semantically adjacent classes (e.g., Backpack, Bag, Briefcase) induce overlapping circuits in representation space, making discrete top-K selection more sensitive to finite-sample noise.

\begin{table*}[!htbp]
\centering
\caption{\textbf{Signal-noise analysis and correction on DomainNet.}
Although extraction noise increases substantially for the 345-class, six-domain setting, cross-domain circuit signal remains dominant, and noise correction yields the largest correlation gain among all benchmarks.}
\label{tab:domainnet_noise_all}

\begin{minipage}[t]{0.48\textwidth}
\centering
\resizebox{\textwidth}{!}{
\begin{tabular}{lcccc}
\toprule
Quantity & Mean & Std & Min & Max \\
\midrule
$\hat{\Delta}_{\text{noise}}$ & 0.142 & 0.034 & 0.086 & 0.221 \\
$\Delta_{\text{obs}}$         & 0.559 & 0.108 & 0.358 & 0.762 \\
$\Delta_{\text{signal}}$      & 0.417 & 0.098 & 0.234 & 0.593 \\
SNR                           & 2.94 & 0.91 & 1.06 & 5.41 \\
SNR $>1$ & \multicolumn{4}{c}{$44/48$ (91.66\%)} \\
\bottomrule
\end{tabular}}
\small
\caption*{(a) Signal--noise decomposition}
\end{minipage}
\hfill
\begin{minipage}[t]{0.48\textwidth}
\centering
\resizebox{\textwidth}{!}{
\begin{tabular}{lc@{\hspace{0.7cm}}lc}
\toprule
Objective & Noise & Arch. & Noise \\
\midrule
IRM   & $0.107\!\pm\!0.022$ & ViT-S & $0.113\!\pm\!0.024$ \\
CORAL & $0.131\!\pm\!0.028$ & MobileNet & $0.128\!\pm\!0.029$ \\
DANN  & $0.143\!\pm\!0.031$ & ResNet50 & $0.155\!\pm\!0.033$ \\
ERM   & $0.176\!\pm\!0.038$ & Mixer & $0.182\!\pm\!0.039$ \\
\bottomrule
\end{tabular}}
\small
\caption*{(b) Noise floor breakdown}
\end{minipage}
\vspace{-6mm}
\end{table*}

Despite these challenges, the domain-induced circuit signal remains dominant. The mean noise floor increases to 
$\hat{\Delta}_{\text{noise}} = 0.142 \pm 0.027$ (vs.\ 0.096 on PACS), while the mean cross-domain dissimilarity rises to 
$\Delta_{\text{obs}} = 0.5559 \pm 0.102$, yielding a mean SNR of 3.58—lower than PACS but still well above 1. 
Of 48 learners, 44 satisfy $\text{SNR} > 1$; the two exceptions are MLP-Mixer models trained with ERM, a configuration that consistently lies at the boundary of reliable extraction. 
On Office-Home, the failure is most pronounced for the Clipart domain, where large photorealistic–clip-art gaps induce partial invariance under ERM, yet the absence of an explicit invariance objective leaves the extracted circuits unstable under resampling.

\begin{table*}[!htbp]
\centering
\caption{\textbf{Ablations on Office-Home.}
The optimal settings shift to $K{=}60$ due to the larger label space, while the preferred bottleneck ratio ($r{=}4$) and edge threshold ($\epsilon{=}10^{-4}$) remain consistent with PACS.}
\label{tab:ablation_oh_all}

\begin{minipage}[t]{0.32\textwidth}
\centering
\resizebox{\textwidth}{!}{
\begin{tabular}{cccc}
\toprule
$K$ & CAS & OOD & $\rho_S$ \\
\midrule
5   & 0.08 & 54.32 & 0.62 \\
10  & 0.14 & 58.71 & 0.69 \\
20  & 0.22 & 62.85 & 0.74 \\
30  & 0.28 & 65.10 & 0.78 \\
\textbf{60}  & \textbf{0.35} & \textbf{68.47} & \textbf{0.82} \\
100 & 0.38 & 68.73 & 0.80 \\
164 & 0.39 & 68.81 & 0.77 \\
200 & 0.40 & 68.84 & 0.75 \\
\bottomrule
\end{tabular}}
\small
\caption*{(a) TopK neurons per layer}
\end{minipage}
\hfill
\begin{minipage}[t]{0.32\textwidth}
\centering
\resizebox{\textwidth}{!}{
\begin{tabular}{ccccc}
\toprule
$r$ & Params & CAS & OOD & $\rho_S$ \\
\midrule
1  & 1,638,400 & 0.27 & 69.18 & 0.73 \\
2  & 819,200   & 0.31 & 68.92 & 0.78 \\
\textbf{4}
& \textbf{409,600}
& \textbf{0.35}
& \textbf{68.47}
& \textbf{0.82} \\
8  & 204,800   & 0.39 & 67.35 & 0.79 \\
16 & 102,400   & 0.44 & 65.48 & 0.74 \\
\bottomrule
\end{tabular}}
\small
\caption*{(b) Adapter bottleneck ratio}
\end{minipage}
\hfill
\begin{minipage}[t]{0.32\textwidth}
\centering
\resizebox{\textwidth}{!}{
\begin{tabular}{ccccc}
\toprule
$\epsilon$ & Edges & CAS & OOD & $\rho_S$ \\
\midrule
$10^{-2}$ & 34  & 0.18 & 64.82 & 0.71 \\
$10^{-3}$ & 89  & 0.27 & 66.93 & 0.78 \\
$\mathbf{10^{-4}}$
& \textbf{158}
& \textbf{0.35}
& \textbf{68.47}
& \textbf{0.82} \\
$10^{-5}$ & 237 & 0.37 & 68.61 & 0.80 \\
$10^{-6}$ & 362 & 0.38 & 68.66 & 0.76 \\
\bottomrule
\end{tabular}}
\small
\caption*{(c) Edge threshold}
\end{minipage}
\vspace{-6mm}
\end{table*}
After noise-floor subtraction, the noise-corrected CAS increases the Spearman correlation from 
$\rho_S = 0.790$ (raw CAS) to $\rho_S = 0.830$, a larger absolute gain than on PACS, indicating that elevated noise on Office-Home was masking a stronger mechanistic signal. The objective-wise noise ordering matches PACS—IRM (0.089), CORAL (0.108), DANN (0.119), ERM (0.146)—supporting the view that invariance-driven objectives regularize circuit structure and improve reproducibility under resampling. Architecturally, ViT-B/16 remains most stable (0.094) and MLP-Mixer least (0.152), with a larger ResNet-50–Mixer gap than on PACS (0.029 vs.\ 0.019). This aligns with the hypothesis that multi-head attention provides discrete anchoring points for circuit extraction, whereas the Mixer's all-MLP design distributes class information more diffusely.
\begin{table*}[!htbp]
\centering
\caption{\textbf{Ablations on DomainNet.}
The optimal TopK increases to $K{=}164$ for the 345-class setting, while the preferred bottleneck ratio ($r{=}4$) and edge threshold ($\epsilon{=}10^{-4}$) remain consistent across benchmarks.}
\label{tab:ablation_dn_all}

\begin{minipage}[t]{0.32\textwidth}
\centering
\resizebox{\textwidth}{!}{
\begin{tabular}{cccc}
\toprule
$K$ & CAS & OOD & $\rho_S$ \\
\midrule
5   & 0.05 & 32.18 & 0.54 \\
10  & 0.09 & 37.45 & 0.60 \\
20  & 0.14 & 42.83 & 0.65 \\
40  & 0.19 & 48.62 & 0.70 \\
80  & 0.24 & 53.91 & 0.74 \\
120 & 0.28 & 57.28 & 0.77 \\
\textbf{164} & \textbf{0.32} & \textbf{59.74} & \textbf{0.79} \\
200 & 0.33 & 59.93 & 0.77 \\
\bottomrule
\end{tabular}}
\small
\caption*{(a) TopK neurons per layer}
\end{minipage}
\hfill
\begin{minipage}[t]{0.32\textwidth}
\centering
\resizebox{\textwidth}{!}{
\begin{tabular}{ccccc}
\toprule
$r$ & Params & CAS & OOD & $\rho_S$ \\
\midrule
1  & 1,638,400 & 0.24 & 60.85 & 0.69 \\
2  & 819,200   & 0.28 & 60.42 & 0.74 \\
\textbf{4}
& \textbf{409,600}
& \textbf{0.32}
& \textbf{59.74}
& \textbf{0.79} \\
8  & 204,800   & 0.36 & 58.16 & 0.76 \\
16 & 102,400   & 0.42 & 55.23 & 0.71 \\
\bottomrule
\end{tabular}}
\small
\caption*{(b) Adapter bottleneck ratio}
\end{minipage}
\hfill
\begin{minipage}[t]{0.32\textwidth}
\centering
\resizebox{\textwidth}{!}{
\begin{tabular}{ccccc}
\toprule
$\epsilon$ & Edges & CAS & OOD & $\rho_S$ \\
\midrule
$10^{-2}$ & 72  & 0.15 & 55.38 & 0.66 \\
$10^{-3}$ & 186 & 0.24 & 57.92 & 0.74 \\
$\mathbf{10^{-4}}$
& \textbf{327}
& \textbf{0.32}
& \textbf{59.74}
& \textbf{0.79} \\
$10^{-5}$ & 498 & 0.34 & 59.91 & 0.77 \\
$10^{-6}$ & 753 & 0.35 & 59.96 & 0.73 \\
\bottomrule
\end{tabular}}
\small
\caption*{(c) Edge threshold}
\end{minipage}
\vspace{-4mm}
\end{table*}

Similarly, \cref{tab:domainnet_noise_all} presents signal–to-noise analysis on DomainNet, the most demanding benchmark: 6 domains, 345 classes, and ~600K images spanning extreme stylistic variation (e.g., Quickdraw sketches, Infograph diagrams, Real photos). The extraction noise floor is highest across benchmarks, 
$\hat{\Delta}_{\text{noise}} = 0.142 \pm 0.034$, due to three compounding factors: (1) the 345-class label space sharply reduces per-class data for circuit estimation, amplifying top-$K$ selection variance; (2) severe cross-domain heterogeneity causes identical semantics to rely on divergent low-level features, increasing subsample sensitivity; and (3) the 15 domain pairs (vs.\ 6 in PACS) enlarge the combinatorial surface for noise to manifest.

Despite the increased difficulty, domain-induced circuit reorganization remains dominant. The mean cross-domain dissimilarity 
$\Delta_{\text{obs}} = 0.559 \pm 0.108$ exceeds the noise floor by nearly $4\times$, and the mean SNR is 2.94, with 44/48 learners satisfying $\text{SNR} > 1$.  The four failures—all MLP-Mixer models trained with ERM—are concentrated on Infograph and Quickdraw, where extreme abstraction combined with no invariance objective yields maximally unstable top-$K$ boundaries. This contrasts with PACS and Office-Home, indicating that under severe heterogeneity, lack of invariance becomes structurally destabilizing. Noise correction yields the largest gain across benchmarks, increasing mean Spearman $\rho_S$ from 0.725 to 0.775 (+0.050), compared to +0.030 (PACS) and +0.040 (Office-Home). This aligns with the mechanism: harder benchmarks allocate more raw CAS variance to extraction stochasticity, so subtracting $\Delta_{\text{noise}}$ better isolates $\Delta_{\text{signal}}$. Objective- and architecture-wise orderings mirror PACS but at higher absolute levels—IRM (0.107), CORAL (0.131), DANN (0.143), ERM (0.176); ViT-B/16 (0.113), MobileNetV2 (0.128), ResNet-50 (0.155), MLP-Mixer (0.182)—indicating that invariance objectives and attention-based architectures retain their relative regularization advantages even under extreme domain shift.

\section{Circuit Extraction Hyperparameters}

\begin{table*}[!htbp]
\centering
\caption{\textbf{Ablation studies for TopK neurons, bottleneck ratio, and edge threshold.}
CAS and OOD performance improve as invariant circuit structure emerges, with optimal correlation at $K=30$, $r=4$, and $\epsilon=10^{-4}$. ($\overline{\text{CAS}}$: Average of CAS across all pair of domains)}
\label{tab:ablation_all}

\begin{minipage}[t]{0.32\textwidth}
\centering
\resizebox{\textwidth}{!}{
\begin{tabular}{cccc}
\toprule
$K$ & $\overline{\text{CAS}}$ & OOD Acc & $\rho_S$ \\
\midrule
5   & 0.12 & 68.41 & 0.71 \\
10  & 0.21 & 73.26 & 0.79 \\
16  & 0.29 & 75.90 & 0.83 \\
\textbf{30}  & \textbf{0.38} & \textbf{77.83} & \textbf{0.88} \\
60  & 0.41 & 78.05 & 0.87 \\
100 & 0.43 & 78.14 & 0.85 \\
164 & 0.44 & 78.19 & 0.83 \\
\bottomrule
\end{tabular}}
\caption*{(a) TopK neurons per layer}
\end{minipage}
\hfill
\begin{minipage}[t]{0.32\textwidth}
\centering
\resizebox{\textwidth}{!}{
\begin{tabular}{ccccc}
\toprule
$r$ & Params & $\overline{\text{CAS}}$ & OOD Acc & $\rho_S$ \\
\midrule
1  & 1,638,400 & 0.31 & 78.52 & 0.80 \\
2  & 819,200   & 0.35 & 78.21 & 0.85 \\
\textbf{4} & \textbf{409,600} & \textbf{0.38} & \textbf{77.83} & \textbf{0.88} \\
8  & 204,800   & 0.42 & 76.94 & 0.86 \\
16 & 102,400   & 0.48 & 75.30 & 0.82 \\
\bottomrule
\end{tabular}}
\caption*{(b) Adapter bottleneck ratio $r$}
\end{minipage}
\hfill
\begin{minipage}[t]{0.32\textwidth}
\centering
\resizebox{\textwidth}{!}{
\begin{tabular}{ccccc}
\toprule
$\epsilon$ & Edges & $\overline{\text{CAS}}$ & OOD Acc & $\rho_S$ \\
\midrule
$10^{-2}$ & 18  & 0.22 & 74.61 & 0.79 \\
$10^{-3}$ & 47  & 0.31 & 76.48 & 0.85 \\
$\mathbf{10^{-4}}$ & \textbf{83} & \textbf{0.38} & \textbf{77.83} & \textbf{0.88} \\
$10^{-5}$ & 126 & 0.40 & 77.95 & 0.86 \\
$10^{-6}$ & 194 & 0.41 & 78.02 & 0.82 \\
\bottomrule
\end{tabular}}
\caption*{(c) Edge threshold $\epsilon$}
\end{minipage}
\vspace{-4mm}
\end{table*}

We ablate the three key hyperparameters of our pipeline, the number of retained neurons per layer $K$, the adapter bottleneck ratio $r$, and the edge inclusion threshold $\epsilon$, to assess their impact on circuit alignment and OOD prediction quality. For each ablation, we vary one hyperparameter while fixing the others at their defaults ($K{=}30$, $r{=}4$, $\epsilon{=}10^{-4}$) and report mean cross-domain CAS (TK), mean OOD accuracy, and Spearman rank correlation $\rho_S$ between CAS and OOD accuracy, averaged across the four PACS target-domain splits over 48 learners.

Across all three ablations (\cref{tab:ablation_all}), each hyperparameter has a clear operating regime where $\rho_S$ is maximized, with degradation on both sides. For TopK (\cref{tab:ablation_all} (a)), small $K$ produces circuits too sparse to capture discriminative structure (CAS $= 0.12$ at $K{=}5$, $\rho_S = 0.71$), while large $K$ admits peripheral neurons that dilute circuit topology with noise, weakening the correlation ($\rho_S = 0.83$ at $K{=}164$). For the bottleneck ratio (\cref{tab:ablation_all} (b)), narrow adapters ($r{=}16$) lack capacity to capture domain-specific computation, inflating CAS artificially (high similarity because the adapters cannot express differences), while wide adapters ($r{=}1$) overparameterise the circuit space, fragmenting the invariant sub-circuit across redundant pathways and reducing CAS discriminability. For the edge threshold (\cref{tab:ablation_all} (c)), aggressive pruning ($\epsilon{=}10^{-2}$) removes causally relevant edges and collapses circuit topology, while permissive thresholds ($\epsilon{=}10^{-6}$) retain noise edges that obscure the domain-shift signal. In all three cases, the circuit contains the smallest sufficient set of components for near-optimal task performance. This convergence across independent ablations provides strong evidence that CAS is not an artifact of a particular hyperparameter configuration but reflects a robust structural property of the underlying computation. 

The primary departure from PACS is an upward shift in the optimal top-$K$ to $K=60$ (vs.\ $K=30$), reflecting the larger 65-class label space and the need for richer per-layer capacity to resolve fine-grained semantics. Sparse circuits ($K=5$, $\rho_S=0.62$) underfit class distinctions, while overly dense circuits ($K=200$, $\rho_S=0.75$) dilute the signal. Performance peaks at $K=60$ ($\rho_S=0.82$) with a broader plateau than PACS, indicating reduced sensitivity once sufficient capacity is reached. The optimal bottleneck ratio ($r=4$, $\rho_S=0.82$) and edge threshold ($\epsilon=10^{-4}$, $\rho_S=0.82$) match PACS, implying that adapter capacity and edge selection are governed by architectural dimensionality rather than label cardinality. As in PACS, narrow adapters ($r=16$) inflate CAS via capacity constraints, whereas wide adapters ($r=1$) fragment invariant sub-circuits across redundant pathways (\cref{tab:ablation_oh_all}).

The optimal top-$K$ increases to $K=164$, scaling with the 345-class label space. This matches the information-theoretic requirement of at least $O(\log_2 345) \approx 8.4$ bits of class-discriminative signal per layer, necessitating greater neuron retention. Performance peaks at $K=164$ ($\rho_S=0.79$) and declines at $K=200$ ($\rho_S=0.77$), confirming that over-inclusion remains detrimental even at scale (\cref{tab:ablation_dn_all}). The optimal bottleneck ratio ($r=4$) and edge threshold ($\epsilon=10^{-4}$) are unchanged from PACS and Office-Home, indicating these hyperparameters are architecture-determined constants. The edge count at optimal threshold (83, 158, 327) scales approximately linearly with label cardinality, consistent with growth in class-discriminative inter-layer connections. Across all benchmarks, $\rho_S$ is maximized at a clear interior optimum and degrades symmetrically, ruling out configuration-specific artifacts.

\section{CAS vs OOD generalization}
\label{sec:refm}
\textbf{\cref{fig:pacs} (PACS — all targets).} CAS (TK) maintains a near-monotone relationship with leave-one-domain-out OOD accuracy across ArtPainting ($\rho_S=0.91$), Photo ($\rho_S=0.93$), and Sketch ($\rho_S=0.88$). The objective-wise clustering is consistent: ERM occupies the low-CAS/low-accuracy regime, IRM the high/high regime, with CORAL and DANN interpolating. The weakest correlation occurs on Sketch, the hardest target, where compressed accuracy ranges induce minor rank inversions—consistent with finite-sample effects when CAS separations are small. Representational baselines (CKA, SVCCA, RSA) yield diffuse scatter across all splits. CKA exhibits saturation ($\approx$0.70–0.88), remaining insensitive to circuit-level reorganization under domain shift. SVCCA and RSA show no consistent objective ordering and occasionally invert the ground-truth ranking (e.g., on Sketch), indicating that subspace fail to capture invariance-relevant circuit structure.

\textbf{\cref{fig:of} (Office-Home — all targets).} CAS exhibits a near-monotone relationship with OOD accuracy across Art ($\rho_S=0.81$), Clipart ($\rho_S=0.76$), Product ($\rho_S=0.82$), and Real World ($\rho_S=0.77$), with a clear ERM-to-IRM gradient in every split. Correlation is tightest on Product and Art, and modestly noisier on Clipart and Real World, where stylistic and viewpoint variation increases within-objective rank variability. Relative to PACS, the CORAL–DANN gap narrows, reflecting Office-Home’s sensitivity to second-order feature statistics. Representational baselines again fail: CKA saturates ($\approx$0.70–0.88) across accuracy levels, and SVCCA/RSA yield inconsistent orderings. On Clipart, DANN substantially outperforms ERM at nearly identical CKA values, underscoring that activation-level similarity is insensitive to circuit-level routing differences that drive generalization.

\textbf{\cref{fig:dn} (DomainNet — all targets).} CAS preserves a monotone relationship with OOD accuracy across Clipart, Infograph, Painting, Quickdraw, Real, and Sketch, despite wide accuracy variation (~5\%–70\%). The ERM-to-IRM gradient remains visible on every split, including the most difficult Infograph and Quickdraw domains. The inter-learner CAS spread is larger than on PACS or Office-Home, reflecting aggregation over more domain pairs and a richer distributional signal. Second, Quickdraw shows the weakest within-objective clustering: IRM models differ by up to 15 points in OOD accuracy at similar CAS values, likely due to the binary line-drawing regime interacting with seed-level variation. Representational baselines again fail: CKA exhibits pronounced saturation, and SVCCA often inverts IRM–ERM ordering, indicating that activation-level similarity carries no reliable circuit-level signal in the large-scale multi-domain setting.

\begin{figure}[!htbp]
    \centering
    \includegraphics[width=\linewidth]{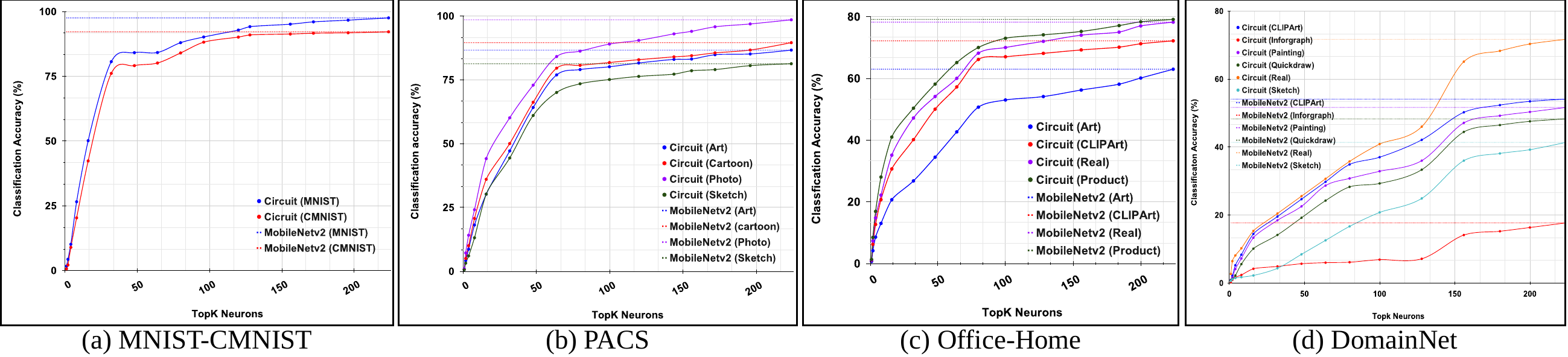}
    \caption{\textbf{Circuits are K-Minimal:} Up to TopK neurons, the performance of the circuits grows exponentially. Then, if we add more neurons, the performance grows very slowly until it matches the original network performance.}
    \label{fig:topk}
\end{figure}

\section{Circuit Minimality and TopK Selection}
\label{sec:topk}

\emph{Circuit minimality} requires retaining the smallest sufficient neuron subset that preserves near-optimal accuracy. Varying $K$ reveals a consistent two-phase pattern across benchmarks (see \cref{fig:topk}).

\myparagraph{Phase I: Rapid convergence:}
For small $K$, accuracy rises steeply as the core invariant sub-circuit is recruited. The heavy-tailed importance distribution implies that a small fraction of neurons carries most causal signal. The elbow occurs around $K\!\approx\!10$--$16$ (MNIST), $25$--$40$ (PACS), $50$--$75$ (Office-Home), and $K^*\!=\!164$ (DomainNet).

\myparagraph{Phase II: Diminishing returns:}
Beyond the elbow, additional neurons yield marginal gains ($<1$--$2\%$ per +10 neurons). These peripheral nodes lie near the noise floor and do not alter core topology. Their contributions are down-weighted by $\mathcal{K}$ and largely cancel in CAS, explaining robustness to $K$ once $K^*$ is reached.

\myparagraph{DomainNet secondary transition:}
DomainNet exhibits a second inflection ($K\!\approx\!80$--$120$), reflecting recruitment of domain-bridging mid-level features needed to span extreme stylistic gaps (e.g., Quickdraw vs.\ Real), before plateauing at $K^*$.

\myparagraph{Robustness of CAS:}
All domains within a dataset attain near-optimal accuracy at the same $K^*$, indicating task-level minimality. Evaluating CAS at $\{0.5K^*, K^*, 1.5K^*, 2K^*\}$ changes $\rho_S$ by $<0.03$. Below $0.5K^*$, circuits lose discriminative structure; above $2K^*$, noise edges accumulate, but CAS remains stable due to kernel down-weighting. Bounded node counts and edit distances ensure topological comparability.

%% file: neurips_tex/03_ACE.tex
\section{Adaptive Circuit Tracing (ACE)}
\label{sec:ace}

We introduce ACE, a framework that inserts MLP adapters into selected intermediate layers of a frozen pretrained backbone, trains only the adapters and a task-specific head, and uses causal activation \cite{zhangtowards} to construct class-specific circuits \footnote[1]{A class-specific circuit is a sparse directed graph of adapter neurons (CNNs) and attention heads (ViTs) that are jointly necessary for prediction.}. Previous circuit extraction techniques \cite{ameisen2025circuit,voss2021visualizing,olah2020an,olah2020zoom} have focused primarily on explaining prediction behavior in \textit{fixed settings}, while this framework is applicable across major backbone families like convolutional networks \cite{simonyan2014very,he2016deep,sandler2018mobilenetv2}, vision transformers \cite{dosovitskiyimage}, and deep MLPs \cite{popescu2009multilayer}. The core procedure remains unchanged across architectures except for the adapter placement as it vary with backbone.


\myparagraph{Problem Setup:}
Let $f_\theta(\mathbf x)$ denote a pretrained backbone with frozen parameters $\theta$. We augment selected intermediate layers of $f_\theta$ with trainable adapter modules $\{\mathcal A_{\psi_\ell}^{(\ell)}\}_{\ell=1}^{L},$ where $\psi_\ell$ denotes the parameters of the adapter inserted at layer $\ell$. A trainable classification head $g_\phi$ with parameters $\phi$ is attached at the output. Only the adapter parameters $\{\psi_\ell\}_{\ell=1}^L$ and head parameters $\phi$ are optimized; the backbone parameters $\theta$ remain fixed. We decompose the frozen backbone as $f_\theta = f_\theta^{(L)}\circ \cdots \circ f_\theta^{(1)}.$ and define frozen layer activations recursively by $z_\ell(\mathbf x) = f_\theta^{(\ell)}(z_{\ell-1}(\mathbf x))$, with an adapter inserted after layer $\ell$, $h_\ell(\mathbf x) = z_\ell(\mathbf x) + \mathcal A_{\psi_\ell}^{(\ell)}(z_\ell(\mathbf x))$ and the next backbone layer receives $h_\ell$ as input: $z_{\ell+1}(\mathbf x) = f_\theta^{(\ell+1)}(h_\ell(\mathbf x)).$ This residual structure ensures that if an adapter output is zero, the adapter reduces to the identity map, and the original backbone computation is recovered.

For transformer backbones (ViTs), each layer function $f_\theta^{(\ell)}$ consists of a multi-head self-attention (MHSA) sublayer followed by a feedforward network (FFN) and residual connections. We denote the output contributed by attention head $h$ in layer $\ell$ by 
$\mathbf a_\ell^{(h)}(\mathbf x)$, with the full attention output given by $\mathrm{MHSA}_\ell(\mathbf z) = W^O \big[ \mathbf a_\ell^{(1)};\dots; \mathbf a_\ell^{(h)} \big].$ In this case, the adapter is inserted after the full transformer block (MHSA~$+$~FFN~$+$~residual), meaning that both attention and adapter computations are available as circuit components.




\begin{wrapfigure}{r}{0.5\textwidth}
\vspace{-4mm}
\centering
\includegraphics[width=0.5\textwidth]{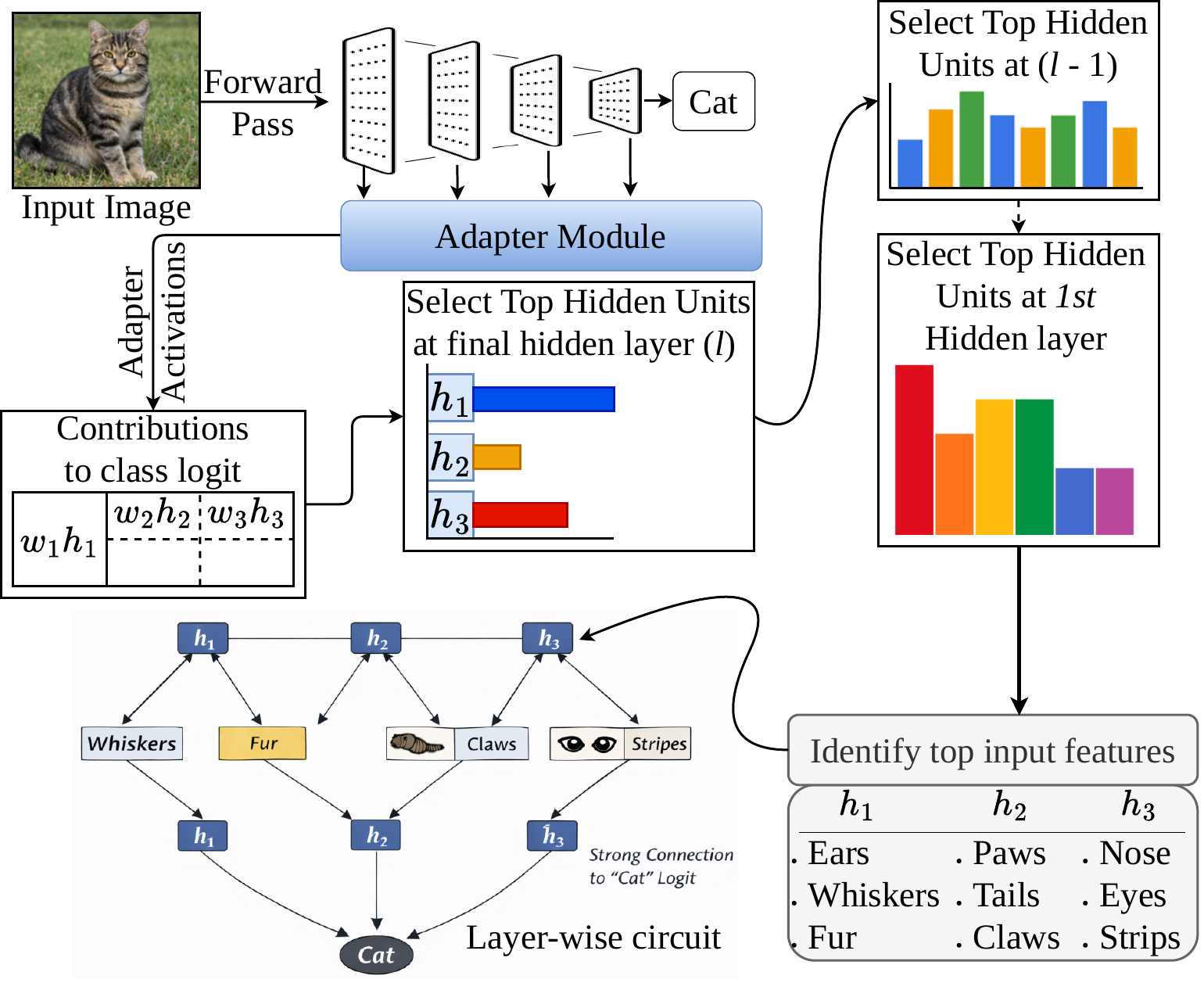}
\vspace{-4mm}
\caption{Layerwise adaptive circuit extraction.}
\label{fig:self}
\vspace{-6mm}
\end{wrapfigure}
\myparagraph{Adapter Module Design:}
Every adapter $\mathcal A_{\psi_\ell}^{(\ell)}$ is a lightweight two-layer MLP that operates on the channel (or feature) dimension of the intermediate representation: $\mathcal A_{\psi_\ell}^{(\ell)}(\mathbf z) = W_\ell^{\mathrm{up}} \sigma \big(W_\ell^{\mathrm{down}}\mathbf z
\big)$, where $W_\ell^{\mathrm{down}} \in \mathbb{R}^{d/r \times d}$ projects to a bottleneck of ratio $r$, $\sigma$ is a ReLU nonlinearity, and $W_\ell^{\mathrm{up}} \in \mathbb{R}^{d \times d/r}$ projects back. 
For convolutional backbones, the adapter acts independently at each spatial location. For transformers, it operates over token embeddings. For deep MLPs, it operates directly on hidden representations.

\myparagraph{Circuit Extraction via Activation Patching:}
After training, we extract a class-specific circuit by quantifying the causal contribution of individual computational units (i.e., adapter MLP neurons and attention head) via activation patching \cite{zhangtowards}. The procedure consists of two stages: node importance scoring and inter-layer edge estimation.
 
 
 
\myparagraph{Node Importance Scoring:}
For each unit $u$ (either an adapter neuron or an attention head) at layer~$\ell$ and each class~$k$, we compute the causal effect:
\begin{equation}
\label{eq:importance}
\Delta_\ell^{(u)}(k) = \mathbb{E}_{\mathbf{x} \sim \mathcal{D}_k}\!\Bigl[\hat{y}_k(\mathbf{x}) - \hat{y}_k\!\bigl(\mathbf{x} \mid \mathrm{do}(u \!=\! \mathbf{0})\bigr)\Bigr],
\end{equation}
 \noindent where $\mathcal{D}_k$ is the subset of inputs with label~$k$.
A positive $\Delta$ indicates the unit \emph{promotes} the correct class; a negative $\Delta$ indicates \emph{suppression}.
For each layer, we retain the $K$ units of each type with the largest $|\Delta|$. For adapter MLP neurons in high-dimensional layers, exhaustive evaluation of all $D$ dimensions may be prohibitive.
In such cases we randomly sample a candidate set of size $S \geq K$ and select the top-$K$ from this subset, providing a stochastic lower bound on the true top-$K$ importance.
 
\myparagraph{Inter-Layer Edge Estimation:}
To determine how top-$K$ units in one layer influence top-$K$ units in the next, we perform a second round of patching.
For consecutive layers $\ell_1, \ell_2$ and top-$K$ units $u_1 \in \mathcal{T}_{\ell_1},\; u_2 \in \mathcal{T}_{\ell_2}$, we ablate $u_1$ and observe the change in $u_2$'s activation:
 \begin{equation}
\label{eq:edge}
w_{u_1 \to u_2} = \mathbb{E}_{\mathbf{x} \sim \mathcal{D}_k}\!\Bigl[\bar{a}_{\ell_2}^{(u_2)}(\mathbf{x}) - \bar{a}_{\ell_2}^{(u_2)}\!\bigl(\mathbf{x} \mid \mathrm{do}(u_1 \!=\! \mathbf{0})\bigr)\Bigr],
\end{equation}
\noindent where $\bar{a}$ denotes the spatially- or token-averaged activation.
An edge from $(\ell_1, u_1)$ to $(\ell_2, u_2)$ is added to the circuit graph if $|w_{u_1 \to u_2}| > \epsilon$. For architectures with attention heads, we estimate edges within each component type separately: attention-to-attention and MLP-to-MLP across consecutive adapter layers, yielding two parallel edge sets that together form the full circuit. This separation reflects the architectural inductive bias: attention heads and MLP adapters operate at different stages of each transformer block, and their inter-layer causal pathways may carry qualitatively different information (e.g., positional routing via attention vs.\ feature refinement via MLPs).
 
\myparagraph{Circuit Graph Construction:}
 For each class~$c$, the procedure yields a sparse, weighted, directed acyclic graph $\mathcal{G}_c = (\mathcal{V}_c, \mathcal{E}_c)$ where:
\begin{equation}
\label{eq:node_set}
\mathcal{V}_c = \underbrace{\bigcup_{\ell} \{(\ell, \texttt{mlp}, d) : d \in \mathcal{T}_\ell^{\mathrm{mlp}}\}}_{\text{adapter MLP nodes}} \;\cup\; \underbrace{\bigcup_{\ell} \{(\ell, \texttt{attn}, h) : h \in \mathcal{T}_\ell^{\mathrm{attn}}\}}_{\text{attention head nodes (ViT only)}},
\end{equation}
\noindent and
\begin{equation}
\label{eq:edge_set}
\mathcal{E}_c = \bigl\{(u_1, u_2, w_{u_1 \to u_2}) : |w_{u_1 \to u_2}| > \epsilon\bigr\}.
\end{equation}
\noindent Nodes are annotated with their importance scores $\Delta$ and typed as either \texttt{mlp} or \texttt{attn}; edges carry signed weights indicating the direction and magnitude of causal influence.
For convolutional and MLP backbones the attention node set is empty, recovering the simpler adapter-only circuit.

\myparagraph{Computational Complexity and Scalability:}
Let $C$ be the maximum number of adapter neurons, $H$ the number of attention heads per layer (if applicable), $L$ the number of adapter layers, and $K$ the top-$K$ selection size. For adapter MLP neurons, this requires $\mathcal{O}(L \cdot C)$ forward passes (or $\mathcal{O}(L \cdot S)$ with random subsampling of $S$ candidates). For attention heads, it requires an additional $\mathcal{O}(L \cdot H)$ passes. Similarly, for MLP-to-MLP edges require $\mathcal{O}(L \cdot K)$ forward passes per layer pair; attention-to-attention edges require another $\mathcal{O}(L \cdot K)$. Crucially, all downstream effects on $\mathcal{T}_{\ell_2}$ are read from the cached output of a single patched forward pass per source node, so the cost scales linearly in $K$, not quadratically. \cref{tab:complexity} summarizes the per-class extraction cost across architectures.
 
\begin{wraptable}{r}{0.6\textwidth}
\centering
\vspace{-4mm}
\caption{The number of forward passes for circuit extraction per class.  For ViT, the attention column adds the cost of patching all $H{=}12$ heads per layer.}
\label{tab:complexity}
\resizebox{0.6\textwidth}{!}{
\begin{tabular}{@{}lcccc@{}}
\toprule
\textbf{Architecture} & \textbf{MLP neurons $C$} & \textbf{Stage 1 (MLP)} & \textbf{Stage 1 (Attn)} & \textbf{Stage 2} \\
\midrule
VGG-19       & 512  & $1{,}536$  & ---   & $30$ \\
ResNet-50    & 2048 & $6{,}144$  & ---   & $30$ \\
MobileNetV2  & 1280 & $3{,}840$  & ---   & $30$ \\
ViT-B/16     & 768  & $2{,}304$  & $36$  & $60$ \\
\bottomrule
\end{tabular}}
\vspace{-4mm}
\end{wraptable}
 
For ViT, Stage~2 doubles because we estimate both MLP-to-MLP and attention-to-attention edge sets. Even so, the total cost remains dominated by Stage~1 MLP patching, and the attention head patching adds only $L \times H = 36$ forward passes with a negligible overhead.

\myparagraph{Architecture-Agnostic Abstraction:}
The framework's scalability across architectures rests on two abstractions. Every adapter, regardless of backbone, exposes a tensor of shape $(B, *, D)$ where $D$ is the feature dimension and $*$ denotes spatial or sequential axes. Ablation operates on dimension~$D$: for convolutions, $\mathbf{h}[\,:\,,c,\,:\,,\,:] = 0$; for transformers, $\mathbf{h}[\,:\,,0,d] = 0$ (CLS token); for MLPs, $\mathbf{h}[\,:\,,d] = 0$. A single \textsc{PatchMLP} subroutine handles all cases. The attention patching module (\textsc{PatchHead}) activates only when the backbone contains self-attention layers. For convolutional and MLP backbones, the attention node set $\mathcal{T}_\ell^{\texttt{attn}}$ is empty and the algorithm reduces to the adapter-only variant. For ViT, the self-attention modules are discovered by traversing the encoder's layer list, and each head is ablated by reshaping the attention output into $(B, T, H, D/H)$ and zeroing the relevant head slice. The end-to-end algorithm of the adaptive circuit extraction has been obtained in \cref{alg:ace}.

\begin{algorithm}[!htbp]
\caption{\textsc{Adapter Circuit Extraction (ACE)}}
\label{alg:ace}
\begin{algorithmic}[1]
\Require Trained model $f$ with adapters $\{\mathcal{A}_\ell\}$, attention modules $\{\mathrm{Attn}_\ell\}$ (if any), data $\mathcal{D}$, top-$K$, threshold $\epsilon$
\Ensure Class-specific circuit graphs $\{\mathcal{G}_k\}$
\State $\{\mathcal{G}_k\} \leftarrow \varnothing$
\For{each mini-batch $(\mathbf{X}, \mathbf{y}) \in \mathcal{D}$}
  \State $\mathbf{O}, \mathcal{C} \leftarrow \textsc{ForwardWithCache}(f, \mathbf{X})$ \Comment{Cache all adapter \emph{and} attention outputs}
  \For{each class $k \in \mathrm{unique}(\mathbf{y})$}
    \State $\mathcal{I}_k \leftarrow \{\,i : y_i = k\,\}$; \quad $b_k \leftarrow \mathrm{mean}(\mathbf{O}[\mathcal{I}_k, k])$
    \Statex \hspace{1.5em}\textcolor{gray}{\textit{--- Stage 1: Node importance ---}}
    \For{each adapter layer $\ell$}
      \For{each neuron $d$ in adapter $\mathcal{A}_\ell$} \Comment{or random subset}
        \State $\tilde{\mathbf{O}},\_ \leftarrow \textsc{PatchMLP}(f, \mathbf{X}, \ell, d)$
        \State $\Delta_{\ell,d}^{\texttt{mlp}} \leftarrow b_k - \mathrm{mean}(\tilde{\mathbf{O}}[\mathcal{I}_k, k])$
      \EndFor
      \State $\mathcal{T}_\ell^{\texttt{mlp}} \leftarrow \text{top-}K \text{ by } |\Delta^{\texttt{mlp}}|$; add MLP nodes to $\mathcal{G}_k$
      \If{$\mathrm{Attn}_\ell$ exists} \Comment{ViT only}
        \For{each head $h = 1, \dots, H$}
          \State $\tilde{\mathbf{O}},\_ \leftarrow \textsc{PatchHead}(f, \mathbf{X}, \ell, h)$
          \State $\Delta_{\ell,h}^{\texttt{attn}} \leftarrow b_k - \mathrm{mean}(\tilde{\mathbf{O}}[\mathcal{I}_k, k])$
        \EndFor
        \State $\mathcal{T}_\ell^{\texttt{attn}} \leftarrow \text{top-}K \text{ by } |\Delta^{\texttt{attn}}|$; add attention nodes to $\mathcal{G}_k$
      \EndIf
    \EndFor
    \Statex \hspace{1.5em}\textcolor{gray}{\textit{--- Stage 2: Inter-layer edges ---}}
    \For{each consecutive pair $(\ell_1, \ell_2)$}
      \For{each MLP neuron $d_1 \in \mathcal{T}_{\ell_1}^{\texttt{mlp}}$}
        \State $\_, \tilde{\mathcal{C}} \leftarrow \textsc{PatchMLP}(f, \mathbf{X}, \ell_1, d_1)$
        \For{each $d_2 \in \mathcal{T}_{\ell_2}^{\texttt{mlp}}$}
          \State $w \leftarrow \mathrm{mean}(\mathcal{C}[\ell_2,d_2]) - \mathrm{mean}(\tilde{\mathcal{C}}[\ell_2,d_2])$
          \State \textbf{if} $|w| > \epsilon$ \textbf{then} add edge $(\ell_1, \texttt{mlp}, d_1) \to (\ell_2, \texttt{mlp}, d_2)$ with weight $w$
        \EndFor
      \EndFor
      \If{$\mathrm{Attn}_{\ell_1}$ and $\mathrm{Attn}_{\ell_2}$ exist}
        \For{each head $h_1 \in \mathcal{T}_{\ell_1}^{\texttt{attn}}$}
          \State $\_, \tilde{\mathcal{C}} \leftarrow \textsc{PatchHead}(f, \mathbf{X}, \ell_1, h_1)$
          \For{each $h_2 \in \mathcal{T}_{\ell_2}^{\texttt{attn}}$}
            \State $w \leftarrow \mathrm{mean}(\mathcal{C}[\ell_2,h_2]) - \mathrm{mean}(\tilde{\mathcal{C}}[\ell_2,h_2])$
            \State \textbf{if} $|w| > \epsilon$ \textbf{then} add edge $(\ell_1, \texttt{attn}, h_1) \to (\ell_2, \texttt{attn}, h_2)$ with weight $w$
          \EndFor
        \EndFor
      \EndIf
    \EndFor
  \EndFor
\EndFor
\State \Return $\{\mathcal{G}_k\}$
\end{algorithmic}
\end{algorithm}

The design of \cref{alg:ace} shows that adding support for a new architecture requires only (a)~a configuration specifying adapter insertion points (fewer than five lines of code) and (b)~optionally, a function that returns the attention modules if they exist. For convolutional features ($B \times C \times H \times W$), ablation sets $\mathbf{h}[\,:\,, c, \,:\,, \,:] = 0$. For transformer hidden states ($B \times T \times D$), ablation targets the CLS token: $\mathbf{h}[\,:\,, 0, d] = 0$. For MLP activations ($B \times D$), ablation sets $\mathbf{h}[\,:\,, d] = 0$. Attention head ablation reshapes $(B, T, D) \to (B, T, H, D/H)$, zeros head~$h$, and reshapes back.

\myparagraph{Discussion:}
The inclusion of attention heads alongside adapter MLP neurons in the ViT circuit provides a more complete picture of the model's computation.
In our experiments, we observe that certain attention heads consistently appear across multiple class circuits, suggesting they perform general-purpose positional routing, while adapter neurons are more class-specific, capturing fine-grained feature refinement.
This decomposition would not be visible in an adapter-only analysis. Because the patching loop iterates over adapter neurons and attention heads rather than all backbone parameters, the computational cost remains tractable even for large backbones. For instance, moving from VGG-19 (144M parameters) to ViT-B/16 (86M parameters) changes the cost only as a function of the adapter dimensionality (512 vs.\ 768) and the small additional attention head budget ($L \times H = 36$ passes), not the total model size.
